\documentclass[11pt]{article}

\usepackage[T1]{fontenc}
\usepackage{microtype}

\usepackage{amsmath}
\usepackage{amssymb}
\usepackage{amsthm}
\usepackage{mathtools}
\usepackage{thmtools}
\usepackage{mathrsfs}

\usepackage[style=trad-alpha,natbib=true,maxnames=3,maxbibnames=99]{biblatex}

\usepackage{libertine}             %
\usepackage[varl,varqu]{zi4}      %

\usepackage{array}     %
\usepackage{booktabs}  %
\usepackage[inline,shortlabels]{enumitem}
\usepackage{graphicx}
\usepackage[dvipsnames]{xcolor}

\usepackage{algorithm}
\usepackage{algorithmicx}
\usepackage[noend]{algpseudocode}

\usepackage{svg}
\usepackage{tikz}

\usepackage{xurl}
\usepackage{hyperref}
\usepackage{cleveref}

\definecolor{alizarin}{RGB}{227,38,54}
\definecolor{ultramarine}{RGB}{24,13,191}

\hypersetup{
  linkcolor  = alizarin,
  citecolor  = ultramarine,
  urlcolor   = ultramarine,
  colorlinks = true
}

\usepackage[margin=1in]{geometry}

\theoremstyle{plain}
\newtheorem{theorem}{Theorem}[section]

\newtheorem{proposition}[theorem]{Proposition}
\newtheorem{lemma}[theorem]{Lemma}

\theoremstyle{definition}
\newtheorem{definition}[theorem]{Definition}
\newtheorem{example}[theorem]{Example}

\theoremstyle{definition}
\newtheorem*{remark*}{Remark}
\newtheorem*{openquestion*}{Open Question}

\definecolor{brightpink}{rgb}{1.0, 0.0, 0.5}

\usepackage[outline]{contour}
\contourlength{1.2pt}

\usepackage{thmtools}

\makeatletter
\thmt@define@thmuse@key{mycontinues}{%
  \thmt@suspendcounter{\thmt@envname}{\thmt@trivialref{#1}{??}}%
}
\makeatother

\newcommand*{\R}{\mathbb{R}}   %

\newcommand*{\Men}{\mathcal{I}}    %
\newcommand*{\Women}{\mathcal{J}}  %
\newcommand*{\MenAll}{\mathcal{I}^*}    %
\newcommand*{\WomenAll}{\mathcal{J}^*}  %
\newcommand*{\Agents}{\mathcal{A}} %
\newcommand*{\AgentsAll}{\mathcal{A}^*} %
\newcommand*{\man}{i}              %
\newcommand*{\woman}{j}            %
\newcommand*{\agent}{a}            %

\newcommand*{\MSet}{X}
\newcommand*{\MSetUCB}{X^\UCB}
\newcommand*{\MSetii}{X'}
\newcommand*{\MSetopt}{X^*}
\newcommand*{\MMap}{\mu_{\MSet}}
\newcommand*{\MMapUCB}{\mu_{\MSet^\UCB}}
\newcommand*{\MMapopt}{\mu_{\MSetopt}}
\newcommand*{\MMapt}{\mu_{\MSet^\round}}

\newcommand*{\MMapii}{\mu_{\MSetii}}
\newcommand*{\MMatrix}{Z}
\newcommand*{\CMatrix}{S}

\newcommand*{\Confidence}{C}         %
\newcommand*{\Instability}[4][]{I(#3, #4; #2, \mathcal{A}#1)}  %
\newcommand*{\InstabilityNTU}[3][]{I(#3; #2, \Agents#1)}  %

\newcommand*{\Gap}{\Delta}  %

\newcommand*{\IR}{Z}  %

\newcommand*{\subsidies}{s}         %
\newcommand*{\transfers}{\tau}      %
\newcommand*{\transfersii}{\tau'}   %
\newcommand*{\prices}{p}            %
\newcommand*{\pricesopt}{p^*}       %
\newcommand*{\utility}{u}           %
\newcommand*{\utilityii}{\tilde u}           %
\newcommand*{\uUCB}{u^{\mathrm{UCB}}} %
\newcommand*{\round}{t}             %
\newcommand*{\Rounds}{T}            %
\newcommand*{\Regret}{R}            %
\newcommand*{\noise}{\epsilon}   %

\newcommand*{\Coalition}{\mathcal{S}}
\newcommand*{\blame}{b}
\newcommand*{\pulls}{n}

\newcommand*{\mean}{\hat u}

\newcommand*{\UUnstructured}{\mkern2mu\mathcal{U}_{\mathrm{unstructured}}}
\newcommand*{\UAll}{\mkern2mu\mathcal{U}}

\newcommand*{\UTyped}{\mkern2mu\mathcal{U}_{\mathrm{typed}}}
\newcommand*{\ULinear}{\mkern2mu\mathcal{U}_{\mathrm{linear}}}

\newcommand*{\context}{c}

\newcommand*{\NTUinstabmeasure}{{NTU Subset Instability}}
\newcommand*{\instabmeasure}{{Subset Instability}}
\newcommand*{\UCB}{\mathrm{UCB}}
\newcommand*{\LS}{\mathrm{LS}}
\newcommand*{\eluderdim}{d_{\textrm{$\epsilon$-eluder}}}

\let\temp\epsilon         %
\let\epsilon\varepsilon
\let\varepsilon\temp

\newcommand*{\Hard}{\mathcal{I}}
\newcommand*{\arm}{\alpha}

\DeclarePairedDelimiter{\norm}{\lVert}{\rVert}
\DeclarePairedDelimiter{\paren}{\lparen}{\rparen}

\newcommand*{\Event}{E}
\newcommand*{\bonus}{\rho}

\newcommand*{\Contexts}{\mathcal{C}}
\newcommand*{\Ball}{\mathcal{B}}

\DeclareMathOperator*{\argmin}{arg\,min}  %

\newcommand*{\todoconst}{0.1}
\newcommand*{\todoconstmod}{0.05}

\usepackage{thmtools,thm-restate}

\bibliography{new-bib.bib}

\definecolor{seagreen}{RGB}{0, 139, 114}

\makeatletter
\newcommand{\printfnsymbol}[1]{%
  \textsuperscript{\@fnsymbol{#1}}%
}
\makeatother

\title{Learning Equilibria in Matching Markets from Bandit Feedback} 

\author{
  Meena Jagadeesan\stepcounter{footnote}\thanks{Equal contribution.} \\
  {mjagadeesan@berkeley.edu} \\
  UC Berkeley, EECS
  \and
  Alexander Wei\printfnsymbol{2} \\
  {awei@berkeley.edu} \\
  UC Berkeley, EECS
  \and
  Yixin Wang \\
  {ywang@eecs.berkeley.edu} \\
  UC Berkeley, EECS
  \and 
  Michael I. Jordan \\
  {jordan@cs.berkeley.edu} \\
  UC Berkeley, EECS and Statistics
  \and
  Jacob Steinhardt \\
  {jsteinhardt@berkeley.edu} \\
  UC Berkeley, Statistics
}

\date{\vspace{-5ex}}

\begin{document}

\maketitle

\begin{abstract}
Large-scale, two-sided matching platforms must find market outcomes that align with user preferences while simultaneously learning these preferences from data. Classical notions of stability (Gale and Shapley, 1962; Shapley and Shubik, 1971) are unfortunately of limited value in the learning setting, given that preferences are inherently uncertain and destabilizing while they are being learned. To bridge this gap, we develop a framework and algorithms for learning stable market outcomes under uncertainty. Our primary setting is matching with transferable utilities, where the platform both matches agents and sets monetary transfers between them. We design an incentive-aware learning objective that captures the distance of a market outcome from equilibrium. Using this objective, we analyze the complexity of learning as a function of preference structure, casting learning as a stochastic multi-armed bandit problem. Algorithmically, we show that ``optimism in the face of uncertainty,'' the principle underlying many bandit algorithms, applies to a primal-dual formulation of matching with transfers and leads to near-optimal regret bounds. Our work takes a first step toward elucidating when and how stable matchings arise in large, data-driven marketplaces.
\end{abstract}

\maketitle

\section{Introduction}\label{sec:intro}

Data-driven marketplaces face the simultaneous challenges of learning agent preferences and aligning market outcomes with the incentives induced by these preferences. Consider, for instance, online platforms that {match} two sides of a market to each other (e.g., Lyft, TaskRabbit, and Airbnb). On these platforms, customers are {matched} to service providers and {pay} for the service they receive. If agents on either side are not offered desirable matches at fair prices, they would have an incentive to leave the platform and switch to a competing platform.
Agent preferences, however, are often unknown to the platform and must be learned.
When faced with uncertainty about agent preferences (and thus incentives), \emph{when can a marketplace efficiently explore and learn market outcomes that align with agent incentives?}

We center our investigation around a model called \emph{matching with transferable utilities}, proposed by Shapley and Shubik~\citep{SS71}. In this model, there is a two-sided market of customers and service providers. 
Each customer has a utility that they derive from being matched to a given provider and vice versa. The platform selects a {matching} between the two sides and assigns a monetary transfer between each pair of matched agents. {Transfers are a salient feature of most real-world matching markets: riders pay drivers on Lyft, clients pay freelancers on TaskRabbit, and guests pay hosts on Airbnb.} An agent's net utility is their value for being matched to their partner plus the value of their transfer (either of which can be negative in the cases of costs and payments).
In matching markets, the notion of \textit{stability} captures alignment of a market outcome with agent incentives. Informally, a market outcome is \textit{stable} if no pair of agents would rather match with each other than abide by the market outcome, and stable matchings can be computed when preferences are fully known. 

In the context of large-scale matching platforms, however, the assumption that preferences are known breaks down. Platforms usually cannot have users report their complete preference profiles. Moreover, users may not even be aware of what their own preferences are. For example, a freelancer may not exactly know what types of projects they prefer until actually trying out specific ones. In reality, a data-driven platform is more likely to learn information about preferences from repeated feedback\footnote{Feedback might arise from explicit sources (e.g., riders rating drivers after a Lyft ride) or implicit sources (e.g., engagement metrics on an app); in either case, feedback is likely to be sparse and noisy.} over time. Two questions now emerge: In such marketplaces, how can stable matchings be learned? And what underlying structural assumptions are necessary for efficient learning to be possible? 

To address these questions, we propose and investigate a model for learning stable matchings from noisy feedback. We model the platform's learning problem using stochastic multi-armed bandits, which lets us leverage the extensive body of work in the bandit literature to analyze the data efficiency of learning (see \citet{LS2020} for a textbook treatment). More specifically, our three main contributions are:
(i)~We develop an incentive-aware learning objective---\instabmeasure{}---that captures the distance of a
market outcome from equilibrium. (ii)~Using \instabmeasure{} as a measure of regret, we show that any ``UCB-based'' algorithm from the classical bandit literature can be adapted to this incentive-aware setting. (iii)~We instantiate this idea for several families of preference structures to design efficient algorithms for incentive-aware learning. This helps elucidate how preference structure affects the complexity of learning stable matchings.

\paragraph{Designing the learning objective.}

Since mistakes are inevitable while exploring and learning, achieving exact stability at every time step is an unattainable goal.
To address this issue, we lean on approximation, focusing on learning market outcomes that are \emph{approximately} stable.
Thus, we need a metric that captures the distance of a market outcome from equilibrium.\footnote{Previous work \cite{DK05, LMJ20} has investigated utility difference (i.e. the difference between the total utility achieved by the selected matching and the utility achieved by a stable matching) as a measure of regret. However, this does not capture distance from equilibrium in matching markets with monetary transfers (see \Cref{sec:instab}) or without monetary transfers (see \Cref{subsubsec:utility}).} 

We introduce a notion for approximate stability that we call \emph{\instabmeasure{}}. Specifically, we define the \instabmeasure{} of a market outcome to be the maximum difference, over all subsets $\Coalition$ of agents, between the total utility of the maximum weight matching on $\Coalition$ and the total utility of $\Coalition$ under the market outcome.\footnote{This formulation is inspired by the \emph{strong $\epsilon$-core} of \citet{SS66}.} We show that \instabmeasure{} can be interpreted as the amount the platform would have to \emph{subsidize} participants to keep them on the platform and make the resulting matching stable. We can also interpret \instabmeasure{} as the platform's cost of learning when facing competing platforms with greater knowledge of user preferences. Finally, we show that \instabmeasure{} is the maximum gain in utility that a coalition of agents could have derived from an alternate matching such that no agent in the coalition is worse off. 

\instabmeasure{} also satisfies the following properties, which make it suitable for learning: (i) \instabmeasure{} is equal to zero if and only if the market outcome is (exactly) stable; (ii) \instabmeasure{} is robust to small perturbations to the utility functions of individual agents, which is essential for learning with noisy feedback; and (iii) \instabmeasure{} upper bounds the utility difference of a market outcome from the socially optimal market outcome.

\paragraph{Designing algorithms for learning a stable matching.} Using \instabmeasure{}, we investigate the problem of learning a stable market outcome from noisy user feedback using the stochastic contextual bandit model (see, e.g., \citep{LS2020}). In each round, the platform selects a market outcome (i.e., a matching along with transfers), with the goal of minimizing cumulative instability. 

{We develop a general approach for designing bandit algorithms within our framework. Our approach is based on a primal-dual formulation of matching with transfers \cite{SS71}, in which the primal variables correspond to the matching and the dual variables can be used to set the transfers.} We find that ``optimism in the face of uncertainty,'' the principle underlying many UCB-style bandit algorithms \cite{DBLP-journals/ml/AuerCF02, LS2020}, can be adapted to this primal-dual setting. The resulting algorithm is simple: maintain upper confidence bounds on the agent utilities and compute, in each round, an optimal primal-dual pair in terms of these upper confidence bounds. The crux of the analysis is the following lemma, which bounds instability by the gap between the upper confidence bound and true utilities:
\begin{lemma}[Informal, see \Cref{lemma:confset} for a formal statement]
\label{lemma:confsetinf}
Given confidence sets for each utility value such that each confidence set contains the true utility, let $(\MSet, \transfers)$ be a stable matching with transfers with respect to the utility functions given by the upper confidence bounds. The instability of $(\MSet, \transfers)$ is upper bounded by the sum of the sizes of the confidence sets of pairs in $\MSet$.
\end{lemma}
We can thus analyze our algorithms by combining Lemma \ref{lemma:confsetinf} with the analyses of existing UCB-style algorithms. In particular, we can essentially inherit the bounds on the size of the confidence bounds from traditional analyses of multi-arm bandits. 

\paragraph{Complexity of learning a stable matching.}
{Our main technical result is a collection of regret bounds for different structural assumptions on agent preferences. These bounds resemble the classical stochastic multi-armed bandits bounds when rewards have related structural assumptions.} We summarize these regret bounds in Table \ref{table:regret} and elaborate on them in more detail below.

\begin{table}[]
    \centering
    \begin{tabular}{l@{\hskip 0.3in}c}
    \toprule & Regret bound %
    \\ \midrule
     Unstructured preferences & $\widetilde O\bigl(N \sqrt{n\Rounds}\bigr)$ 
     \\ \addlinespace[0.08cm] 
     Typed preferences & $\widetilde O\bigl(|\Contexts|  \sqrt{n\Rounds} \bigr)$ 
     \\ \addlinespace[0.08cm]
     Separable linear preferences & $\widetilde O\bigl(d\sqrt{N} \sqrt{n\Rounds}\bigr)$
     \\ \bottomrule
    \end{tabular}
    \caption{Regret bounds for different preference structures when there are $N$ agents on the platform and no more than $n$ agents arriving in each round.}
    \label{table:regret}
\end{table}

\begin{theorem}[Unstructured Preferences, Informal]
\label{thm:unstructinf}
For unstructured preferences, there exists a UCB-style algorithm that incurs \smash{$\widetilde O(N \sqrt{n\Rounds})$} regret according to \instabmeasure{} after $\Rounds$ rounds, where $N$ is the number of agents on the platform and $n$ is the number of agents that arrive in any round. (This bound is optimal up to logarithmic factors.) 
\end{theorem}
\begin{theorem}[Typed Preferences, Informal]
\label{thm:typedinf}
Consider preferences such that each agent $a$ has a type $\context_a \in \Contexts$ and the utility of $a$ when matched to another agent $a'$ is given by a function of the types $\context_a$ and $\context_{a'}$. There exists a UCB-style algorithm that incurs \smash{$\widetilde O(|\Contexts|  \sqrt{n \Rounds})$} regret according to \instabmeasure{} after $\Rounds$ rounds, where $n$ is the maximum number of agents that arrive to the platform in any round. 
\end{theorem}
\begin{theorem}[Separable Linear Preferences, Informal]
\label{thm:sepinf}
Consider preferences such that the utility of an agent $a$ when matched to another agent $a'$ is $\langle \phi(a), \context_{a'}\rangle$, where $\phi(a) \in \mathbb{R}^d$ is unknown and $\context_{a'} \in \mathbb{R}^d$ is known. There exists a UCB-style algorithm that incurs \smash{$\widetilde O(d \sqrt{N} \sqrt{n \Rounds})$} regret according to \instabmeasure{} after $\Rounds$ rounds, where $N$ is the number of agents on the platform and $n$ is the maximum number of agents that arrive in any round.
\end{theorem}

These results elucidate the role of preference structure on the complexity of learning a stable matching. Our regret bounds scale with $N \sqrt{n T}$ for unstructured preferences (\Cref{thm:unstructinf}), $|\Contexts| \sqrt{n T}$ for typed preferences (\Cref{thm:typedinf}), and $d \sqrt{N} \sqrt{n T}$ for linear preferences (\Cref{thm:sepinf}). To illustrate these differences in a simple setting, let's consider the case where all of the agents show up every round, so $n = N$. In this case, our regret bound for unstructured preferences is superlinear in $N$; in fact, this dependence on $N$ is \textit{necessary} as we demonstrate via a lower bound (see \Cref{lemma:lowerbound}).  On the other hand, the complexity of learning a stable matching changes substantially with preference structure assumptions. In particular, our regret bounds are sublinear / linear in $N$ for typed preferences and separable linear preferences. This means that in large markets, a centralized platform can efficiently learn a stable matching with these preference structure assumptions.

\paragraph{Connections and extensions.} 

Key to our results and extensions is the primal-dual characterization of equilibria in matching markets with transfers. Specifically, equilibria are described by a linear program whose primal form maximizes total utility over matchings and whose dual variables correspond to transfers. This linear program inspires our definition of Subset Instability, connects Subset Instability to platform profit (see \Cref{appendix:alternatedef}), and relates learning with Subset Instability to regret minimization in combinatorial bandits (see \Cref{sec:lowerbound}). We adapt ideas from combinatorial bandits to additionally obtain $O(\log T)$ instance-dependent regret bounds (see \Cref{appendix:instancespecific}). 

Our approach also offers a new perspective on learning stable matchings in {markets with \emph{non-transferable} utilities} \citep{DK05,LMJ20}. Although this setting does not admit a linear program formulation, we show Subset Instability can be extended to what we call NTU Subset Instability (see \Cref{appendix:matchntu}), which turns out to have several advantages over the instability measures studied in previous work. Our algorithmic principles extend to NTU Subset Instability: we prove regret bounds commensurate with those for markets with transferable utilities.

\subsection{Related work}\label{sec:related}

In the machine learning literature, starting with \citet{DK05} and \citet{LMJ20}, several works \cite{DK05,LMJ20,SBS20,LRMJ20,CS21,BSS21} study learning stable matchings from bandit feedback in the Gale-Shapley stable marriage model \cite{GS62}. A major difference between this setting and ours is the absence of monetary transfers between agents. These works focus on the \textit{utility difference} rather than the instability measure that we consider. \citet{CS21} extend this bandits model to incorporate fixed, predetermined cost/transfer rules. However, they do not allow the platform to set arbitrary transfers between agents. Moreover, they also consider a weaker notion of stability that does not consider agents negotiating arbitrary transfers: defecting agents must set their transfers according to a fixed, predetermined structure. In contrast, we follow the classical definition of stability \cite{SS71}. 

Outside of the machine learning literature,  several papers also consider the complexity of finding stable matchings in other feedback and cost models, e.g., communication complexity \cite{GNOR19, ashlagi2020clearing, DBLP-conf/sigecom/Shi20} and query complexity \cite{EGK20, ashlagi2020clearing}. Of these works, \citet{DBLP-conf/sigecom/Shi20}, which studies the communication complexity of finding approximately stable matchings with transferable utilities, is perhaps most similar to ours. This work assumes agents know their preferences and focuses on the communication bottleneck, whereas we study the costs associated with learning preferences. Moreover, the approximate stability notion in \citet{DBLP-conf/sigecom/Shi20} is the maximum unhappiness of any \emph{pair} of agents, whereas \instabmeasure{} is equivalent to the maximum unhappiness over any \emph{subset} of agents. For learning stable matchings, \instabmeasure{} has the advantages of being more fine-grained and having a primal view that motivates a clean UCB-based algorithm.

Our notion of instability connects to historical works in coalitional game theory: related are the concepts of the strong-$\epsilon$ core of \citet{SS66} and the indirect function of \citet{indirect-function}, although each was introduced in a very different context than ours. Nonetheless, they reinforce the fact that our instability notion is a very natural one to consider.

A complementary line of work in economics  \cite{LMPS14, B17, A20, L20} considers stable matchings under incomplete information. These works focus on defining stability when the agents have incomplete information about their own preferences, whereas we focus on the platform's problem of learning stable matchings from noisy feedback. As a result, these works relax the definition of stability to account for uncertainty in the preferences of agents, rather than the uncertainty experienced by the platform from noisy feedback.

Multi-armed bandits have also been applied to learning in other economic contexts. For example, learning a socially optimal matching (without learning transfers) is a standard application of combinatorial bandits \cite{CesaBianchiLugosi12, GKJ12, DBLP-conf/icml/ChenWY13, combinatorialbanditsrevisited, DBLP-conf/aistats/KvetonWAS15}. Other applications at the interface of  bandit methodology and economics include dynamic pricing \cite{R74, DBLP-conf/focs/KleinbergL03, DBLP-journals/jacm/BadanidiyuruKS18}, incentivizing exploration \cite{FKKK14, MSS15}, learning under competition \cite{AMSW20}, and learning in matching markets without incentives \cite{JKK21}. 

Finally, primal-dual methods have also been applied to other problems in the bandits literature (e.g., \cite{ISSS19, TPRL20, LSY21}).

\section{Preliminaries}\label{sec:preliminaries}

The foundation of our framework is the \emph{matching with transfers} model of \citet{SS71}.
In this section, we introduce this model along with the concept of stable matching. %

\subsection{Matching with transferable utilities}

Consider a two-sided market that consists of a finite set $\Men$ of customers on one side and a finite set $\Women$ of providers on the other. Let $\Agents\coloneqq\Men\cup\Women$ be the set of all agents. A \emph{matching} $\MSet\subseteq\Men\times\Women$ is a set of pairs $(\man, \woman)$ that are pairwise disjoint, representing the pairs of agents that are matched. Let $\mathscr{X}_\Agents$ denote the set of all matchings on $\Agents$. For notational convenience, we define for each matching $\MSet \in\mathscr{X}_\Agents$ an equivalent functional representation $\MMap\colon\Agents\to\Agents$, where $\MMap(\man) = \woman$ and $\MMap(\woman) = \man$ for all matched pairs $(\man, \woman)\in\MSet$, and $\MMap (\agent) = \agent$ if $\agent\in\Agents$ is unmatched. 

When a pair of agents $(\man, \woman)\in\Men\times\Women$ matches, each experiences a utility gain. We denote these utilities by a global utility function $\utility: \Agents \times \Agents \rightarrow \mathbb{R}$, where $\utility(\agent, \agent')$ denotes the utility that agent $\agent$ gains from being matched to agent $\agent'$. (If $\agent$ and $\agent'$ are on the same side of the market, we take $\utility(\agent, \agent')$ to be zero by default.) We allow these utilities to be negative, if matching results in a net cost (e.g., if an agent is providing a service). We assume each agent $\agent\in\Agents$ receives zero utility if unmatched, i.e., $\utility(\agent, \agent) = 0$. When we wish to emphasize the role of an individual agent's utility function, we will use the equivalent notation $\utility_\agent(\agent')\coloneqq \utility(\agent, \agent')$.

A \emph{market outcome} consists of a matching $\MSet\in\mathscr{X}_\Agents$ along with a vector \smash{$\transfers \in \R^{\Agents}$} of transfers, where $\transfers_\agent$ is the amount of money transferred from the platform to agent $\agent$ for each $\agent\in\Agents$.
These monetary transfers are a salient feature of most real-world matching markets: riders pay drivers on Lyft, clients pay freelancers on TaskRabbit, and guests pay hosts on Airbnb. \citet{SS71} capture this aspect of matching markets by augmenting the classical two-sided matching model with transfers of utility between agents. Transfers are typically required to be \emph{zero-sum}, meaning that $\transfers_\man + \transfers_\woman = 0$ for all matched pairs $(\man, \woman)\in\MSet$ and $\transfers_\agent = 0$ if $\agent$ is unmatched. Here, $\MSet$ represents how agents are matched and $\transfers_\agent$ represents the transfer that agent $\agent$ receives (or pays). The net utility that an agent $\agent$ derives from a matching with transfers $(\MSet, \transfers)$ is therefore $\utility(\agent, \MMap(\agent)) + \transfers_\agent$.

\paragraph{Stable matchings.}
In matching theory, stability captures when a market outcome aligns with individual agents' preferences. Roughly speaking, a market outcome $(\MSet, \transfers)$ is stable if: (i) no individual agent $\agent$ would rather be unmatched, and (ii) no pair of agents $(\man, \woman)$ can agree on a transfer such that both would rather match with each other than abide by $(\MSet, \transfers)$. Formally:
\begin{definition}%
\label{def:stability}
A market outcome $(\MSet, \transfers)$ is \emph{stable} if:
(i) it is \emph{individually rational}, i.e., 
\begin{equation}
\label{eq:IR}
\utility_\agent(\MMap(\agent)) + \transfers_\agent\ge 0
\end{equation}
for all agents $\agent\in\Agents$, and (ii) it \emph{has no blocking pairs}, i.e.,
\begin{equation}\label{eq:block}
\bigl(\utility_\man(\MMap(\man)) + \transfers_\man\bigr) + \bigl(\utility_\woman(\MMap(\woman)) + \transfers_\woman\bigr)\ge \utility_\man(\woman) + \utility_\woman(\man)
\end{equation}
for all pairs of agents $(\man,\woman)\in\Men\times\Women$.\footnote{We observe that \eqref{eq:block} corresponds to no pair of agents $(\man, \woman)$ being able to agree on a transfer such that both would rather match with each other than abide by $(\MSet, \transfers)$. Notice that a pair $(\man, \woman)$ violates \eqref{eq:block} if and only if they can find a transfer $\transfers_\man' = -\transfers_\woman'$ such that $\utility_\man(\woman) + \transfers_\man' > \utility_\man(\MMap(\man)) + \transfers_\man$ and $\utility_\woman(\man) + \transfers_\woman' > \utility_\woman(\MMap(\woman)) + \transfers_\woman$.}
\end{definition}

A fundamental property of the matching with transfers model is that if  $(\MSet, \transfers)$ is stable, then $\MSet$ is a maximum weight matching, i.e., $\MSet$ maximizes $\sum_{\agent\in\Agents} \utility_\agent(\MMap(\agent))$ over all matchings $\MSet\in\mathscr{X}_{\Agents}$ \citep{SS71}. The same work shows that stable market outcomes coincide with Walrasian equilibria. (For completeness, we recapitulate the basic properties of this model in \Cref{appendix:recap}.)

To make the matching with transfers model concrete, we use the simple market depicted in the center panel of \Cref{fig:my_label} as a running example throughout the paper. This market consists of a customer Charlene and two providers Percy and Quinn, which we denote by $\Men = \{C\}$ and $\Women = \{P, Q\}$. If the agents' utilities are as given in \Cref{fig:my_label}, then Charlene would prefer Quinn, but Quinn's cost of providing the service is much higher. Thus, matching Charlene and Percy is necessary for a stable outcome. This matching is stable for any transfer from Charlene to Percy in the interval $[5, 7]$.

\begin{figure}[t]
    \centering\footnotesize
    \begin{tikzpicture}[scale=0.2, roundnode/.style={circle, draw=gray!60, fill=gray!5, very thick, minimum size=3mm}]
    \draw[thin, dashed, gray] (-15.75, 8) -> (-15.75, -6);
    \node at (-20, 8) {Customers};
    \node at (-12, 8) {Providers};
    \node[roundnode] at (-20, -1) (charlene0) {};
    \node[roundnode] at (-20, 2) (dana0) {};
    \node[roundnode] at (-20, 5) (ellen0) {};
    \node[roundnode] at (-20, -4) (frances0) {};
    \node[roundnode] at (-12, 4) (percy0) {};
    \node[roundnode] at (-12, 0.5) (quinn0) {};
    \node[roundnode] at (-12, -3) (ryan0) {};
    \draw[ultra thick, blue] (charlene0) -> (percy0);
    \draw[ultra thick, blue] (ellen0) -> (quinn0);
    \draw[ultra thick, blue] (frances0) -> (ryan0);
    \draw[ultra thick, seagreen, ->] (-19.3, 0.55) -> (-13.3, 4.3) node [midway, sloped, above right, xshift=1] {\contour{white}{\includegraphics[scale=0.07]{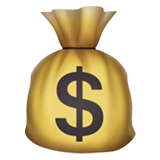}}};
    \draw[ultra thick, seagreen, ->] (-18.95, 3.525) -> (-13.35, 0.150) node [midway, sloped, below right, xshift=1] {\contour{white}{\includegraphics[scale=0.07]{money-bag_1f4b0.png}}};
    \draw[ultra thick, seagreen, ->] (-18.68, -4.81) -> (-12.88, -4.11) node [midway, sloped, below right, xshift=-3] {\contour{white}{\includegraphics[scale=0.07]{money-bag_1f4b0.png}}};
    
    \draw[thin, dashed, gray] (5, 8) -> (5, -6);
    \node at (0, 8) {Customers};
    \node at (10, 8) {Providers};
    \node[roundnode] at (0, 0) (charlene) {$C$};
    \node[roundnode] at (10, 4) (percy) {$P$};
    \node[roundnode] at (10, -4) (quinn) {$Q$};
    \draw[ultra thick, blue] (charlene) -> (percy);
    \draw[ultra thick, seagreen, ->] (1.9, 1.75) -> (7.6, 4.15) node [midway, sloped, above] {\contour{white}{pay 6}};
    \node[align=left] at (-2, 0) {$\utility_C(P) = 9$ \\[0.7em] \\ \\ $\utility_C(Q) = 12$};
    \node[align=left] at (11.5, 0) {$\utility_P(C) = -5$ \\[0.5em] $\utility_Q(C) = -10$};

    \draw[thin, dashed, gray] (32, 8) -> (32, -6);
    \node at (27, 8) {Customers};
    \node at (37, 8) {Providers};
    \node[roundnode] at (27, 0) (charlene2) {$C$};
    \node[roundnode] at (37, 4) (percy2) {$P$};
    \node[roundnode] at (37, -4) (quinn2) {$Q$};
    \draw[ultra thick, dashed, blue] (charlene2) -> (percy2); 
    \draw[ultra thick, dashed, blue] (charlene2) -> (quinn2); 
    \draw[ultra thick, dashed, seagreen, ->] (28.9, 1.75) -> (34.6, 4.15) node [midway, sloped, above] {\textbf{?}};
    \draw[ultra thick, dashed, seagreen, ->] (28.9, -1.75) -> (34.6, -4.15) node [midway, sloped, below] {\textbf{?}};
    \node[align=right] at (24, 0) {$\utility_C(P) = 9\pm 1$ \\[0.7em] \\ \\ $\utility_C(Q) = 12\pm 4$};
    \node[align=left] at (39, 0) {$\utility_P(C) = -7\pm 3$ \\[0.5em] $\utility_Q(C) = -10\pm 1$};
    
    \draw[thick] (17, 11) -> (17, -7);
    \draw[thick] (-7.75, 11) -> (-7.75, -7);
    
    \node[label={left:Matching market:}] at (-11.25, 10) {};
    
    \node[label={right:True utilities + stable outcome:}] at (-7.75, 10) {};
    
    \node[label={right:Platform's uncertainty sets:}] at (17, 10) {};
    
    \end{tikzpicture}
    \caption{The left panel depicts a schematic of a matching (blue) with transfers (green). The center panel depicts a matching market with three agents and a stable matching with transfers for that market. (If the transfer $6$ is replaced with any value between $5$ and $7$, the outcome remains stable.) The right panel depicts the same market, but with utilities replaced by uncertainty sets; note that no matching with transfers is stable for all realizations of utilities.}
    \label{fig:my_label}
\end{figure}

\section{Learning Problem and Feedback Model}\label{subsec:bandits}

We instantiate the platform's learning problem in a {stochastic contextual bandits} framework. Matching takes place over the course of $\Rounds$ rounds. We denote the set of all customers by $\MenAll$, the set of all providers by $\WomenAll$, and the set of all agents on the platform by $\AgentsAll = \MenAll \cup \WomenAll$. Each agent $\agent \in \AgentsAll$ has an associated context $\context_\agent\in\Contexts$, where $\Contexts$ is the set of all possible contexts. This context represents the {side information} available to the platform about the agent, e.g., demographic, location, or platform usage information. Each round, a set of agents arrives to each side of the market. The platform then selects a market outcome and incurs a regret equal to the \emph{instability} of the market outcome (which we introduce formally in \Cref{sec:instab}). Finally, the platform receives noisy feedback about the utilities of each matched pair $(\man, \woman)$.

To interpret the noisy feedback, note that platforms in practice often receive feedback both explicitly (e.g., riders rating drivers after a Lyft ride) and implicitly (e.g., engagement metrics on an app). In either instance, feedback is likely to be sparse and noisy. For simplicity, we do not account for agents strategically manipulating their feedback to the platform and focus on the problem of learning preferences from unbiased reports.

We now describe this model more formally. In the $\round$-th round: 
\begin{enumerate}
\itemsep=0.0em
    \item A set $\Men^\round\subseteq\MenAll$ of customers and a set $\Women^\round\subseteq\WomenAll$ of providers arrive to the market. Write $\Men^\round \cup \Women^\round \eqqcolon \Agents^\round$. The platform observes the identity $\agent$ and the \emph{context} $\context_\agent\in\Contexts$ of each agent $\agent \in \Agents^\round$. 
    \item The platform selects a matching with \emph{zero-sum} transfers $(\MSet^\round, \transfers^\round)$ between $\Men^\round$ and $\Women^\round$. 
    \item The platform observes noisy utilities $\utility_\agent(\mu_{X^t}(\agent)) + \noise_{\agent,\round}$ for each agent $\agent\in\Men^\round\cup\Women^\round$, where the $\noise_{\agent,\round}$ are independent, $1$-subgaussian random variables.\footnote{Our feedback model corresponds to \emph{semi-bandit} feedback, since the platform has (noisy) access to each agent's utility within the matching rather than the overall utility of the matching.} 
    \item The platform incurs regret equal to the \textit{instability} of the selected market outcome $(\MSet^\round, \transfers^\round)$. (We define instability formally in Section \ref{sec:instab}.)
\end{enumerate}
The platform's total regret $\Regret_\Rounds$ is thus the cumulative instability incurred up through round $\Rounds$.

\subsection{Preference structure} 
In this bandits framework, we can impose varying degrees of structure on agent preferences. We encode these preference structures via the functional form of agents' utility functions and their relation to agent contexts. More formally, let $\UAll$ be the set of functions $\utility\colon \AgentsAll \times \AgentsAll \rightarrow \mathbb{R}$, i.e., $\UAll$ is the set of all possible (global) utility functions. We now introduce several classes of preference structures as subsets of $\UAll$.

\paragraph{Unstructured preferences.} The simplest setting we consider is one where the preferences are unstructured. Specifically, we consider the class of utility functions
\[\UUnstructured = \bigl\{u \in \UAll \mid u(a, a') \in [-1, 1] \bigr\}.\]  
(Here, one can think of the context as being uninformative, i.e., $\Contexts$ is the singleton set.) In this setup, the platform must learn each agent's utility function $\utility_\agent(\cdot) = \utility(\agent, \cdot)$. 

\paragraph{Typed preferences.} We next consider a market where each agent comes in one of finitely many \emph{types}, with agents of the same type having identical preferences. Assuming typed preference structures is standard in theoretical models of markets (see, e.g., \citet{debreu-scarf, echenique-revealed, azevedo-hatfield}). We can embed types into our framework by having each agent's context represent their type, with $|\Contexts| < \infty$. The global utility function is then fully specified by agents' contexts:
\[\UTyped = \left\{u \in \UAll \mid u(\agent, \agent') = f(\context_\agent, \context_{\agent'}) \text{ for some $f\colon\Contexts\times\Contexts\to [-1,1]$}\right\}. \]

\paragraph{Separable linear preferences.} We next consider markets where each agent is associated with \emph{known} information given by their context as well as \emph{hidden} information that must be learned by the platform. (This differs from unstructured preferences, where all information was hidden, and typed preferences, where each agent's context encapsulated their full preferences.) We explore this setting under the assumption that agents' contexts and hidden information interact linearly. 

We assume that all contexts belong to $\Ball^d$ (i.e., $\Contexts=\Ball^d$) where $\Ball^d$ is the $\ell_2$ unit ball in $\R^d$. We also assume that there exists a function $\phi\colon\AgentsAll\to\Ball^d$ mapping each agent to the hidden information associated to that agent. The preference class $\ULinear^{d}$ can then be defined as
\[\ULinear^d = \Bigl\{\utility\in\mathcal{U} \bigm\vert \utility(\agent, \agent') = \langle c_{a'}, \phi(a) \rangle\text{ for some  $\phi\colon \AgentsAll \to \Ball^d$} \Bigr\}. \]

\section{Measuring Approximate Stability}\label{sec:instab}

When learning stable matchings, we must settle for guarantees of approximate stability, since exact stability---a binary notion---is unattainable when preferences are uncertain. To see this, we return to the example from \Cref{fig:my_label}. Suppose that the platform has uncertainty sets given by the right panel. Recall that for the true utilities, all stable outcomes match Charlene with Percy. If the true utilities were instead the upper bounds of each uncertainty set, then all stable outcomes would match Charlene and Quinn. Given only the uncertainty sets, it is impossible for the platform to find an (exactly) stable matching, so it is necessary to introduce a measure of approximate stability as a relaxed benchmark for the platform; we turn to this now.

Given the insights of \citet{SS71}---that all stable outcomes maximize the sum of agents' utilities---it might seem natural to measure distance from stability simply in terms of the
\emph{utility difference}. To define this formally, let $\Agents$ be the set of agents participating in the market. (This corresponds to $\Agents^\round$ at time step $t$ in the bandits model.) The utility difference\footnote{Utility difference is standard as a measure of regret for learning a maximum weight matching in the combinatorial bandits literature (see, e.g., \cite{GKJ12}). However, we show that for learning stable matchings, a fundamentally different measure of regret is needed.} of a market outcome $(\MSet, \transfers)$ is given by: 
\begin{equation}\label{eq:utilitydifference}
\left(\max_{\MSetii \in \mathscr{X}_\Agents} \sum_{\agent \in \Agents} \utility_\agent(\MMapii(\agent)) \right) - \left(\sum_{a \in \Agents} \utility_\agent(\MMap(\agent)) + \transfers_\agent)\right).
\end{equation} The first term $\max_{\MSetii \in \mathscr{X}_\Agents} \sum_{\agent \in \Agents} \utility_\agent(\MMapii(\agent))$ is the maximum total utility of any matching, and the second term $ \sum_{a \in \Agents} (\utility_\agent(\MMap(\agent)) + \transfers_\agent)$ is the total utility of market outcome $(\MSet, \transfers)$. Since transfers are zero-sum, \eqref{eq:utilitydifference} can be equivalently written as
\[\left(\max_{\MSetii \in \mathscr{X}_\Agents} \sum_{\agent \in \Agents} \utility_\agent(\MMapii(\agent)) \right) - \sum_{a \in \Agents} \utility_\agent(\MMap(\agent)).\]
But this shows that utility difference actually ignores the transfers $\transfers$ entirely! In fact, the utility difference can be zero even when the transfers lead to a market outcome that is far from stable (see \Cref{appendix:limitations}). Utility difference is therefore \emph{not} incentive-aware, making it unsuitable as an objective for learning stable matchings with transfers.

In the remainder of this section, we propose a measure of instability---\instabmeasure{}---which we will show serves as a suitable objective for learning stable matchings with transfers. Specifically, we show that \instabmeasure{} captures the distance of a market outcome from equilibrium while reflecting both the platform's objective and the users' incentives. We additionally show that \instabmeasure{} satisfies several structural properties that make it useful for learning.

\subsection{\instabmeasure{}}\label{sec:interpretation}

\instabmeasure{} is based on utility difference, but rather than only looking at the market in aggregate, it takes a maximum ranging over all subsets of agents.
\begin{definition}
\label{def:instability}
Given utilities $\utility$, the \emph{\instabmeasure{}}  $\Instability{\utility}{\MSet}{\transfers}$ of a matching with transfers $(\MSet, \transfers)$ is
\begin{equation*}\label{eq:subsetinstab}
\max_{\Coalition \subseteq \Agents} \left[ \left( \max_{\MSetii \in \mathscr{X}_\Coalition} \sum_{\agent \in \Coalition} \utility_\agent(\MMapii(\agent)) \right)  - \left(\sum_{\agent \in \Coalition} \utility_\agent(\MMap(\agent)) + \transfers_\agent\right)  \right]. \tag{$\ast$}
\end{equation*}
(The first term $\max_{X' \in \mathscr{X}_\Coalition} \sum_{\agent \in \Coalition} \utility_\agent(\MMapii(\agent))$ is the maximum total utility of any matching over $\Coalition$, and the second term $\sum_{a \in \Agents} (\utility_\agent(\MMap(\agent)) + \transfers_\agent)$ is the total utility of the agents in $\Coalition$ under market outcome $(\MSet, \transfers)$.)
\end{definition}
Intuitively, \instabmeasure{} captures stability because it checks whether any subset of agents would prefer an alternate outcome. We provide a more extensive economic interpretation below; but before doing so, we first illustrate Definition \ref{def:instability} in the context of the example in \Cref{fig:my_label}.

Consider the matching $\MSet = \{(C, Q)\}$ with transfers $\transfers_C = -11$ and $\transfers_Q = 11$. (This market outcome is stable for the upper bounds of the uncertainty sets of the platform in \Cref{fig:my_label}, but not stable for the true utilities.) It is not hard to see that the subset $\Coalition$ that maximizes \instabmeasure{} is $\Coalition = \left\{C, P \right\}$, in which case $\max_{\MSetii \in \mathscr{X}_\Coalition} \sum_{\agent \in \Coalition} \utility_\agent(\MMapii(\agent))= 4$ and $\sum_{\agent \in \Coalition} \left(\utility_\agent(\MMap(\agent)) + \transfers_\agent\right) = 1$. Thus, 
the \instabmeasure{} of $(\MSet, \transfers)$ is $\Instability{\utility}{\MSet}{\transfers} = 4 - 1 = 3$. In contrast, the utility difference of $(\MSet, \transfers)$ is $2$.

We now discuss several interpretations of \instabmeasure{}, which provide further insight into why \instabmeasure{} serves as a meaningful notion of approximate stability in online marketplaces. In particular, \instabmeasure{} can be interpreted as \textit{the minimum stabilizing subsidy}, as \textit{the platform's cost of learning}, as \textit{a measure of user unhappiness}, and as a \textit{distance from equilibrium}.

\paragraph{\instabmeasure{} as the platform's minimum stabilizing subsidy.} 
\instabmeasure{} can be interpreted in terms of monetary subsidies from the platform to the agents. Specifically, the \instabmeasure{} of a market outcome equals the {minimum amount the platform could subsidize agents so that the subsidized market outcome is individually rational and has no blocking pairs}.

More formally, let $\subsidies\in\R^\Agents_{\ge 0}$ denote subsidies made by the platform, where the variable $s_\agent \ge 0$ represents the subsidy provided to agent $\agent$.\footnote{The requirement that $s_a \geq 0$ enforces that all subsidies are nonnegative; without it, \eqref{eq:instability} would reduce to the utility difference, which is not incentive-aware.} For a market outcome $(\MSet, \transfers)$, the \textit{minimum stabilizing subsidy} is
\begin{equation}
\label{eq:instabilityrecast}
\min_{\subsidies \in \R_{\ge 0}^{\Agents}} \Biggl\{ \sum_{\agent \in \Agents} \subsidies_{\agent} \biggm\vert  \text{$(\MSet, \transfers + \subsidies)$ is stable}\Biggr\},
\end{equation}
where we define stability in analogy to \Cref{def:stability}. Specifically, we say that a market outcome $(\MSet, \transfers)$ with subsidies $\subsidies$ is \emph{stable} if it is {individually rational}, i.e., $\utility_\agent(\MMap(\agent)) + \transfers_\agent + \subsidies_\agent\ge 0$ for all agents $\agent\in\Agents$, and {has no blocking pairs}, i.e., $(\utility_\man(\MMap(\man)) + \transfers_\man + \subsidies_\man) + (\utility_\woman(\MMap(\woman)) + \transfers_\woman + \subsidies_\woman)\ge \utility_\man(\woman) + \utility_\woman(\man)$ for all pairs of agents $(\man, \woman)\in\Men\times\Women$. 

Given this setup, we show the following equivalence:
\begin{restatable}{proposition}{subsidy}
\label{prop:subsidy}
Minimum stabilizing subsidy equals \instabmeasure{} for any market outcome.
\end{restatable}

The proof boils down to showing that the two definitions are ``dual'' to each other. To formalize this, we rewrite the minimum stabilizing subsidy as the solution to the following linear program:\footnote{In this linear program, the first set of constraints ensures there are no blocking pairs, while the second set of constraints ensures individual rationality.}:
\begin{alignat}{5}\label{eq:instability}
	\min_{\subsidies \in\R^{|\Agents|}} & \sum_{\agent \in \Agents} \subsidies_{\agent} & \\ %
	\text{s.t.}\ \ & \bigl(\utility_\man(\MMap(\man)) + \transfers_\man + \subsidies_\man\bigr) + \bigl(\utility_\woman(\MMap(\woman)) + \transfers_\woman + \subsidies_\woman\bigr)\ge \utility_\man(\woman) + \utility_\woman(\man)\qquad&&\forall (\man,\woman)\in\Men\times\Women \nonumber \\
	& \utility_\agent(\MMap(\agent)) + \transfers_\agent + \subsidies_\agent\ge 0\quad&&\forall\agent\in\Agents \nonumber \\
	&\subsidies_\agent\ge 0 \quad&&\forall\agent\in\Agents. \nonumber
\end{alignat}
The crux of our argument is that the dual linear program to \eqref{eq:instability} maximizes the combinatorial objective \eqref{eq:subsetinstab}. The equivalence of \eqref{eq:subsetinstab} and \eqref{eq:instability} then follows from strong duality.

With this alternate formulation of \instabmeasure{} in mind, we revisit the example in \Cref{fig:my_label}. Again, consider the matching $\MSet = \{(C, Q)\}$ with transfers $\transfers_C = -11$ and $\transfers_Q = 11$. (This is stable for the upper bounds of the uncertainty sets of the platform in \Cref{fig:my_label}, but not stable for the true utilities.) We have already shown above that the \instabmeasure{} of this market outcome is 3. To see this via the subsidy formulation, note that the optimal subsidy $\subsidies$ gives $C$ and $P$ a total of $3$. (E.g., we give $C$ a subsidy of $\subsidies_C = 2$ and $P$ a subsidy of $\subsidies_P = 1$.) Indeed, if $\subsidies_C + \subsidies_P = 3$, then
\[ \bigl(\utility_C(\MMap(C)) + \transfers_C + \subsidies_C\bigr) + \bigl(\utility_P(\MMap(P)) + \transfers_P + \subsidies_P\bigr) \ge \utility_C(P) + \utility_P(C) \]
holds (with equality), so the pair $(C, P)$ could no longer gain by matching with each other.

The subsidy perspective turns out to be useful when designing learning algorithms. In particular, while the formulation in \Cref{def:instability} involves a maximization over the $2^{|\Agents|}$ subsets of $\Agents$, the linear programming formulation \eqref{eq:instability} only involves $O(|\Agents|)$ variables and $O(|\Agents|^2)$ constraints.

\paragraph{\instabmeasure{} as the platform's cost of learning.} We next
connect minimum stabilizing subsidies to the platform's \emph{cost of learning}---how much the platform would have to pay to keep users on the platform in the presence of a worst-case (but budget-balanced) competitor with perfect knowledge of agent utilities.

Observe that \eqref{eq:instabilityrecast} is the minimum amount the platform could subsidize agents so that no budget-balanced competitor could convince agents to leave.
The way that we formalize ``convincing agents to leave'' is that: (a) an agent will leave the original platform if they prefer to be unmatched over being on the platform, or (b) a pair of agents who are matched on the competitor's platform will leave the original platform if they both prefer the new market outcome over their original market outcomes. Thus, if we imagine the platform as actually paying the subsidies, then the cumulative instability (i.e., our regret) can be realized as a ``cost of learning'': it is how much the platform pays the agents to learn a stable outcome while ensuring that no agent has the incentive to leave during the learning process. Later on, we will see that our algorithmic approach can be extended to efficiently compute feasible subsidies for \eqref{eq:instability} that are within a constant factor of our regret bound, meaning that subsidies can be implemented using only the information that the platform has. Moreover, in \Cref{appendix:alternatedef}, we show that cost of learning can also be explicitly connected to the platform's revenue. 

\paragraph{\instabmeasure{} as a measure of user unhappiness.} While the above interpretations focus on \instabmeasure{} from the platform's perspective, we show that \instabmeasure{} can also be interpreted as a measure of {user unhappiness}. Given a subset $\Coalition\subseteq\Agents$ of agents, which we call a coalition, we define the \textit{unhappiness} of $\Coalition$ with respect to a market outcome $(\MSet, \transfers)$ to be the maximum gain (relative to $(\MSet, \transfers)$) in total utility that the members of coalition $\Coalition$ could achieve by matching only among themselves, such that no member is worse off than they were in $(\MSet, \transfers)$. (See \Cref{app:unhappiness} for a formal definition.) The condition that no member is worse off ensures that all agents would actually want to participate in the coalition (i.e.~they prefer it to the original market outcome).

User unhappiness differs from the original definition of \instabmeasure{} in \eqref{eq:subsetinstab}, because \eqref{eq:subsetinstab} does not require individuals to be better off in any alternative matching. However, we show that this difference is inconsequential: %

\begin{restatable}{proposition}{unhappiness}
\label{prop:unhappiness}
The maximum unhappiness of any coalition $\Coalition\subseteq\Agents$ with respect to $(\MSet, \transfers)$ equals the \instabmeasure{} $\Instability{\utility}{\MSet}{\transfers}$.
\end{restatable}
See \Cref{app:unhappiness} for a full proof. In the proof, we relate the maximum unhappiness of any coalition to the dual linear program to \eqref{eq:instability}. To show this relation, we leverage the fact that optimal solutions to the dual program correspond to blocking pairs of agents as well as individual rationality violations.

The main takeaway from \Cref{prop:unhappiness} is that \instabmeasure{} not only measures costs to 
the platform, but also costs to users, in terms of the maximum amount they ``leave on the table''
by not negotiating an alternate arrangement amongst themselves.

\paragraph{\instabmeasure{} as a distance from equilibrium.} Finally, we connect \instabmeasure{} to solution concepts for coalitional games, a general concept in game theory that includes matching with transfers as a special case. Coalitional games (also known as cooperative games) capture competition and cooperation amongst a group of agents. The \textit{core} is the set of outcomes in a cooperative game such that no subset $\Coalition$ of agents can achieve higher total utility among themselves than according to the given outcome. In games where the core is empty, a natural relaxation is the \textit{strong $\epsilon$-core} \cite{SS66}, which is the set of outcomes in a cooperative game such that no subset $\Coalition$ of agents can achieve total utility among themselves that is at least $\epsilon$ greater than according to the given outcome.

\instabmeasure{} can be seen as transporting the strong $\epsilon$-core notion to a slightly different context. In particular, in the context of matching with transferable utilities, the core is exactly the set of stable matchings; since a stable matching always exists, the core is always nonempty. Even though the core is nonempty, we can nonetheless use the strong $\epsilon$-core to measure \textit{distance from the core}. More specifically, it is natural to consider the smallest $\epsilon$ such that $(\MSet, \transfers)$ is in the strong $\epsilon$-core. This definition exactly aligns with {\instabmeasure{}}, thus providing an alternate interpretation of \instabmeasure{} within the context of coalitional game theory.

\subsection{Properties of \instabmeasure{}}\label{sec:properties}

We now describe additional properties of our instability measure that are important for learning. We show that \instabmeasure{} is: (i) zero if and only if the matching with transfers is stable, (ii) Lipschitz in the true utility functions, and (iii) lower bounded by the utility difference.

\begin{restatable}{proposition}{desiderata}
\label{prop:desiderata}
\instabmeasure{} satisfies the following properties:
\begin{enumerate}
    \item \instabmeasure{} is always nonnegative and is zero if and only if $(\MSet, \transfers)$ is stable.
    \item \instabmeasure{} is Lipschitz continuous with respect to agent utilities. That is, for any possible market outcome $(\MSet, \transfers)$, and any pair of utility functions $\utility$ and $\utilityii$ it holds that:
    \[|\Instability{\utility}{\MSet}{\transfers} - \Instability{\utilityii}{\MSet}{\transfers}| \le 2\sum_{\agent\in\Agents} \norm{\utility_\agent - \utilityii_\agent}_\infty.\]
    \item \instabmeasure{} is always at least the utility difference. 
\end{enumerate}
\end{restatable}
We defer the proof to \Cref{appendix:properties}.

These three properties show that \instabmeasure{} is useful as a regret measure for learning stable matchings. The first property establishes that \instabmeasure{} satisfies the basic desideratum of having zero instability coincide with exact stability. The second property shows that \instabmeasure{} is robust to small perturbations to the utility functions of individual agents. The third property ensures that, when learning using \instabmeasure{} as a loss function, the platform learns a socially optimal matching.

Note that the second property already implies the existence of an explore-then-commit algorithm that achieves \smash{$\widetilde{O}(N^{4/3} T^{2/3})$} regret in the simple setting where $\Agents^\round = \Agents$ for some $\Agents$ of size $N$ for all $\round$.\footnote{This bound can be achieved by adapting the explore-then-commit (ETC) approach where the platform explores by choosing each pair of agents \smash{$\widetilde{O}((T/N)^{2/3})$} times \cite{LS2020}. Thus, \smash{$\widetilde{O}(N^{1/3} T^{2/3})$} rounds are spent exploring, and the \instabmeasure{} of the matching selected in the commit phase is \smash{$\widetilde{O}(N^{4/3} T^{2/3})$} with high probability. We omit further details since this analysis is a straightforward adaptation of the typical ETC analysis.}
In the next section, we will explore algorithms that improve the dependence on the number of rounds $\Rounds$ to \smash{$\sqrt{\Rounds}$} and also work in more general settings.

\section{Regret Bounds}\label{sec:regret}

In this section, we develop a general approach for designing algorithms that achieve near-optimal regret within our framework. To be precise, the platform's regret is defined to be
\[\Regret_\Rounds = \sum_{\round = 1}^\Rounds \Instability[^\round]{\utility}{\MSet^\round}{\transfers^\round}.\]
While our framework bears some resemblance to the (incentive-free) combinatorial bandit problem of learning a maximum weight matching, two crucial differences differentiate our setting: (i) in each round, the platform must choose \emph{transfers} in addition to a matching, and (ii) loss is measured with respect to \emph{instability} rather than the utility difference. Nonetheless, we show that a suitable interpretation of ``optimism in the face of uncertainty'' can still apply. 

\paragraph{Regret bounds for different preference structures.}  By instantiating this optimism-based approach, we derive regret bounds for the preference structures introduced in \Cref{subsec:bandits}. We start with the simplest case of unstructured preferences, where we assume no structure on the utilities.
\begin{restatable}{theorem}{instanceind}\label{thm:instanceind}
For preference class $\UUnstructured$ (see \Cref{subsec:bandits}), {\normalfont\textsc{MatchUCB}} (defined in \Cref{subsec:explicitalgs}) incurs expected  regret \smash{$\mathbb{E}(\Regret_\Rounds) = O\bigl(|\Agents| \sqrt{n \Rounds\log(|\Agents| \Rounds)}\bigr)$}, where $n = \max_\round |\Agents_\round|$.
\end{restatable}
\noindent In \Cref{sec:lowerbound}, we additionally give a matching (up to logarithmic factors) lower bound showing for $n = |\Agents|$ that such scaling in $|\Agents|$ is indeed necessary. This demonstrates that the regret scales with $|\Agents| \sqrt{n}$, which is superlinear in the size of the market. Roughly speaking, this bound means that the platform is required to learn a superconstant amount of information per agent in the marketplace. These results suggest that without preference structure, it is unlikely that a platform can efficiently learn a stable matching in large markets.

The next two bounds demonstrate that, with preference structure, efficient learning of a stable matching becomes possible. First, we consider typed preferences, which are purely specified by a function $f$ mapping finitely many pairs of contexts to utilities.
\begin{restatable}{theorem}{typed}
\label{thm:typed}
For preference class $\UTyped$ (see \Cref{subsec:bandits}), {\normalfont\textsc{MatchTypedUCB}} (defined in \Cref{subsec:explicitalgs}) incurs expected  regret \smash{$\mathbb{E}(\Regret_\Rounds) = O\bigl({|\Contexts|} \sqrt{n \Rounds\log(|\Agents|\Rounds)}\bigr)$}, where $n = \max_t |\Agents_t|$.
\end{restatable}
\noindent For a fixed type space $\Contexts$, the regret bound in \Cref{thm:typed} scales sublinearly with the market size (captured by $|\Agents|$ and $n$). This demonstrates that the platform can efficiently learn a stable matching when preferences are determined by types. In fact, the regret bound only depends on the number of agents who arrive on the platform in any round; notably, it does not depend on the total number of agents on the platform (beyond logarithmic factors). 

Finally, we consider separable linear preferences, where the platform needs to learn hidden information associated with each agent. 
\begin{restatable}{theorem}{linear}
\label{thm:linear}
For preference class $\ULinear$ (see \Cref{subsec:bandits}), {\normalfont\textsc{MatchLinUCB}} (defined in \Cref{subsec:explicitalgs}) incurs expected  regret \smash{$\mathbb{E}(\Regret_\Rounds) = O\bigl(  d  \sqrt{|\Agents|}\sqrt{n \Rounds\log(|\Agents| \Rounds)}\bigr)$}, where $n = \max_t |\Agents_t|$.
\end{restatable}
\noindent When $n$ is comparable to $|\Agents|$, the regret bound in \Cref{thm:linear} scales linearly with the market size (captured by $|\Agents|$) and linearly with the dimension $d$. Roughly speaking, this means that the platform learns (at most) a constant amount of information per agent in the marketplace. We interpret this as indicating that the platform can efficiently learn a stable matching in large markets for separable linear preferences, although learning in this setting is more demanding than for typed preferences.

\subsection{Algorithm}\label{sec:algorithm}

Following the principle of optimism, our algorithm selects at each round a stable market outcome using upper confidence bounds as if they were the true agent utilities. To design and analyze this algorithm, we leverage the fact that, in the full-information setting, stable market outcomes are optimal solutions to a pair of primal-dual linear programs whose coefficients depend on agents' utility functions. This primal-dual perspective lets us compute a market outcome each round. A particular consequence is that any UCB-based algorithm for learning matchings in a semi-bandit setting can be transformed into an algorithm for learning \textit{both} the matching and the prices.

\paragraph{Stable market outcomes via linear programming duality.}
Before proceeding with the details of our algorithm, we review how the primal-dual framework can be used to select a stable market outcome in the full information setting. \citet{SS71} show that stable market outcomes $(\MSet, \transfers)$ correspond to optimal primal-dual solutions to the following pair of primal and dual linear programs (where we omit the round index $\round$ and consider matchings over $\Agents = \Men\cup\Women$):
\begin{center}
{\begin{minipage}[t]{.44\linewidth}
\textbf{Primal} (P) \\[-0.75em]
\rule{\textwidth}{0.4pt} \\[-1.25em]
\begin{alignat*}{5}%
	\max_{\MMatrix\in\R^{|\Men|\times|\Women|}} & \mathrlap{\smashoperator[r]{\sum_{(\man, \woman)\in\Men\times\Women}} \MMatrix_{\man,\woman} (\utility_\man(\woman) + \utility_\woman(\man))} \\
	\text{s.t.}\ \ 
	& \sum_{\woman\in\Women} \MMatrix_{\man,\woman}\le 1\quad&&\forall\man\in\Men \\
	& \sum_{\man\in\Men} \MMatrix_{\man,\woman}\le 1\quad&&\forall\woman\in\Women \\
	& \MMatrix_{\man,\woman}\ge 0 &&\forall(\man, \woman)\in\Men\times\Women
\end{alignat*}
\end{minipage}}
\quad
{\begin{minipage}[t]{.51\linewidth}
\textbf{Dual} (D) \\[-0.75em]
\rule{\textwidth}{0.4pt} \\[-1.25em]
\begin{alignat*}{5}%
    \min_{\prices\in\R^{|\Agents|}} & \sum_{\agent\in\Agents} \prices_\agent \\
    \text{s.t.}\ \ 
    & \prices_\man + \prices_\woman \ge\utility_\man(\woman) + \utility_\woman(\man)\quad &&\forall (\man, \woman)\in\Men\times\Women \\
    & \prices_\agent\ge 0 &&\forall \agent\in\Agents
\end{alignat*}
\end{minipage}}
\end{center}
\medskip

\noindent The primal program (P) is a linear programming formulation of the maximum weight matching problem: the Birkhoff-von Neumann theorem states that its extreme points are exactly the indicator vectors 
for matchings between $\Men$ and $\Women$. Each dual variable $\prices_\agent$ in (D) can be interpreted as a \emph{price} that roughly corresponds to agent $\agent$'s net utility. Specifically, given any optimal primal-dual pair $(\MMatrix, \prices)$, one can recover a matching $\MMap$ from the nonzero entries of $\MMatrix$ and set transfers $\transfers_\agent = \prices_\agent - \utility_\agent(\MMap(\agent))$ to obtain a stable outcome $(\MSet, \transfers)$. Moreover, any stable outcome induces an optimal primal-dual pair $(\MMatrix, \prices)$.

\paragraph{Algorithm overview.}
Leveraging the above primal-dual formulation of stability, we introduce a meta-algorithm \textsc{MetaMatchUCB} for learning stable outcomes (\Cref{alg:meta-bandits}). In each round, we compute a matching with transfers by solving the primal-dual linear programs for our upper confidence bounds: Suppose we have a collection $\mathscr{C}$ of confidence sets \smash{$\Confidence_{\man,\woman}, \Confidence_{\woman,\man}\subseteq\R$} such that $\utility_\man(\woman)\in \Confidence_{\man,\woman}$ and $\utility_\woman(\man)\in\Confidence_{\woman,\man}$ for all $(\man, \woman) \in \Men \times \Women$. Our algorithm uses $\mathscr{C}$ to get an upper confidence bound for each agent's utility function and then computes a stable matching with transfers as if these upper confidence bounds were the true utilities (see \textsc{ComputeMatch} in \Cref{alg:fromconfsets}). This can be implemented efficiently if we use, e.g., the Hungarian algorithm \cite{Kuhn} to solve (P) and (D).

\begin{algorithm}
\caption{\textsc{MetaMatchUCB}: A bandit meta-algorithm for matching with transferable utilities.}\label{alg:meta-bandits}
\begin{algorithmic}[1]
\Procedure{MetaMatchUCB}{$T$}
   \State Initialize confidence intervals $\mathscr{C}$ over utilities.
   \For{$1 \le \round \le\Rounds$}
   \State $(\MSet^\round, \transfers^\round) \gets \Call{ComputeMatch}{\mathscr{C}}$
   \State Update confidence intervals  $\mathscr{C}$.
   \EndFor
\EndProcedure
\end{algorithmic}
\end{algorithm}

\subsection{Main lemma}

The key fact we need to analyze our algorithms is that \instabmeasure{} is upper bounded by the sum of the sizes of the relevant confidence sets, assuming that the confidence sets over the utilities contain the true utilities. (In the following, we again omit the round index $\round$.) 

\begin{restatable}{lemma}{confset}
\label{lemma:confset}
Suppose a collection of confidence sets $\mathscr{C} = \{C_{i,j}, C_{j,i} : (i, j)\in\Men\times\Women\}$ is such that $\utility_{\man}(\woman) \in \Confidence_{\man,\woman}$ and $\utility_{\woman}(\man) \in \Confidence_{\woman,\man}$ for all $(\man, \woman)$. Then the instability of \smash{$(\MSetUCB, \transfers^{\UCB})\coloneqq\textsc{ComputeMatch}(\mathscr{C})$} satisfies
\begin{equation}\label{eq:instabound}
\Instability[^\round]{\utility}{\MSetUCB}{\transfers^\UCB} \le \sum_{\agent \in \Agents} \Bigl(\max\bigl(\Confidence_{\agent, \MMapUCB(\agent)}\bigr) - \min\bigl(\Confidence_{\agent, \MMapUCB(\agent)}\bigr)\Bigr).\end{equation}
\end{restatable}

\begin{proof}
Since $(\MSet^\UCB, \transfers^\UCB)$ is stable with respect to $\uUCB$, we have $\Instability[^\round]{\uUCB}{\MSet^\UCB}{\transfers^\UCB} = 0$. Thus, it is equivalent to bound the difference $\Instability[^\round]{\utility}{\MSet^\UCB}{\transfers^\UCB} - \Instability[^\round]{\uUCB}{\MSet^\UCB}{\transfers^\UCB}$. 

At this stage, it might be tempting to bound this difference using the Lipschitz continuity of \instabmeasure{} (see \Cref{prop:desiderata}). However, this would only allow us to obtain an upper bound of the form $\sum_{\agent \in \Agents} \max_{\agent' \in \Agents} \bigl(\max\bigl(\Confidence_{\agent, \agent'}\bigr) - \min\bigl(\Confidence_{\agent, \agent'}\bigr)\bigr)$. The problem with this bound is that it depends on the sizes of the confidence sets for all pairs of agents, including those that are \emph{not} matched in \smash{$\MSet^\UCB$}, making it too weak to prove regret bounds for UCB-style algorithms.\footnote{For intuition, consider the classical stochastic multi-armed bandits setting and suppose that we could only guarantee that the loss incurred by an arm is bounded by the maximum of the sizes of the confidence sets over \textit{all} arms. Then, we would only be able to obtain a weak bound on regret, since low-reward arms with large confidence sets may never be pulled.} Thus, we proceed with a more fine-grained analysis.

Define the function
\[  f(\Coalition, \MSet, \transfers; \utility) =  \left( \max_{\MSetii \in \mathscr{X}_\Coalition} \sum_{\agent \in \Coalition} \utility_\agent(\MMapii(\agent)) \right)  - \left(\sum_{\agent \in \Coalition} \utility_\agent(\MMap(\agent)) + \transfers_\agent\right). \]
By definition, $\Instability{\utility}{\MSet}{\transfers} = \max_{\Coalition\subseteq\Agents} f(\Coalition, \MSet, \transfers; \utility)$.
It follows that
\begin{multline*} \Instability[^\round]{\utility}{\MSet^\UCB}{\transfers^\UCB} - \Instability[^\round]{\uUCB}{\MSet^\UCB}{\transfers^\UCB} \\ \le \max_{\Coalition \subseteq \Agents} \left( f(\Coalition, \MSet^\UCB, \transfers^\UCB; u) -  f(\Coalition, \MSet^\UCB, \transfers^\UCB; \uUCB) \right).\end{multline*}
To finish, we upper bound $f(\Coalition, \MSet^\UCB, \transfers^\UCB; u) -  f(\Coalition, \MSet^\UCB, \transfers^\UCB; \uUCB)$ for each $\Coalition \subseteq \Agents$. We decompose this expression into two terms: 
\begin{align*}
    f(\Coalition, \MSet^\UCB, \transfers^\UCB; &\utility) - f(\Coalition, \MSet^\UCB, \transfers^\UCB; \utility^\UCB) \\ 
    &= \underbrace{\left( \max_{\MSetii \in \mathscr{X}_\Coalition} \sum_{\agent \in \Coalition} \utility_\agent(\MMapii(\agent)) - \max_{\MSetii \in \mathscr{X}_\Coalition} \sum_{\agent \in \Coalition} \uUCB_\agent(\MMapii(\agent)) \right)}_{(\text{A})}
  \\ &\qquad\qquad+ \underbrace{\left(\sum_{\agent \in \Coalition} \bigl(\uUCB_\agent(\MMapUCB(\agent)) + \transfers^\UCB_\agent\bigr) - \sum_{\agent \in \Coalition} \bigl(\utility_\agent(\MMapUCB(\agent)) + \transfers^\UCB_\agent\bigr) \right)}_{(\text{B})}.
\end{align*}
To see that (A) is nonpositive, observe that the maximum weight matching of $\Coalition$ with respect to $\utility$ is no larger than the maximum weight matching of $\Coalition$ with respect to $\uUCB$, since $\uUCB$ pointwise {upper bounds} $\utility$. To upper bound (B), observe that the transfers cancel out, so the expression is equivalent to 
\[\sum_{\agent \in \Coalition} \bigl(\uUCB_\agent(\MMapUCB(\agent)) - \utility_\agent(\MMapUCB(\agent))\bigr)\le \sum_{\agent \in \Agents} \Bigl(\max\bigl(\Confidence_{\agent, \MMapUCB(\agent)}\bigr) - \min\bigl(\Confidence_{\agent, \MMapUCB(\agent)}\bigr)\Bigr). \qedhere\]
\end{proof}

\begin{algorithm}[tb] \caption{$\textsc{ComputeMatch}$: Compute matching with transfers from confidence sets of utilities}\label{alg:fromconfsets}
\begin{algorithmic}[1]
\Procedure{ComputeMatch}{$\mathscr{C}$}
   \For{$(\man, \woman) \in \Men \times \Women$}\Comment{Instantiate UCB estimates of utilities.}
      \State {$\uUCB_{\man}(\woman)\gets\max\bigl( \Confidence_{\man,\woman}\bigr)$}
      \State \smash{$\uUCB_{\woman}(\man)\gets \max \bigl(\Confidence_{\woman,\man}\bigr)$}
    \EndFor
      \State {$(\MSetopt, \pricesopt)\gets\text{optimal primal-dual pair for (P) and (D) given utilities $\uUCB$}$}
    \For{$\agent \in \Agents$}\Comment{Set transfers based on \smash{$(\MSetopt, \pricesopt)$} and UCB utilities.}
     \State \smash{$\transfers_\agent\gets \pricesopt_\agent - \uUCB_\agent(\MMapopt(\agent))$}
     \EndFor
   \State \textbf{return} ($\MSetopt, \transfers$)
\EndProcedure
\end{algorithmic}
\end{algorithm}

\subsection{Instantiations of the meta-algorithm}\label{subsec:explicitalgs}

As formalized in \textsc{MetaMatchUCB}, the regret bound of \Cref{lemma:confset} suggests a simple approach: {at each round, select the matching with transfers returned by \textsc{ComputeMatch} and update confidence sets accordingly}. To instantiate \textsc{MetaMatchUCB}, it remains to construct confidence intervals that contain the true utilities with high probability. This last step naturally depends on the assumptions made about the utilities and the noise.

\paragraph{Unstructured preferences.} For this setting, we construct confidence intervals following the classical UCB approach: for each utility value involving the pair $(\man, \woman)\in\Men\times\Women$, we take a confidence interval of length \smash{$O\bigl(\sqrt{\log(|\Agents|\Rounds) / \smash{\pulls_{\man\woman}}}\bigr)$} centered at the empirical mean, where $\pulls_{\man\woman}$ is the number of times the pair has been matched thus far. We describe this construction precisely in \Cref{alg:bandits} (\textsc{MatchUCB}).

\begin{algorithm}
\caption{\textsc{MatchUCB}: A bandit algorithm for matching with transferable utilities for unstructured preferences.}\label{alg:bandits}
\begin{algorithmic}[1]
\Procedure{MatchUCB}{$T$}
    \For{$(\man, \woman) \in \Men \times \Women$} \Comment{Initialize confidence intervals.}
   \State $\Confidence_{\man, \woman} \gets [-1, 1]$
   \State $\Confidence_{\woman, \man} \gets [-1, 1]$
   \EndFor
   \For{$1 \le \round \le\Rounds$}
   \State $(\MSet^\round, \transfers^\round) \gets \Call{ComputeMatch}{\mathscr{C}}$
   \For{$(\man, \woman) \in \MSet^\round$} \Comment{Set confidence intervals and update means.}
   \State Update empirical means $\hat\utility_\man(\woman)$ and $\hat\utility_\woman(\man)$ from feedback; increment counter $\pulls_{\man\woman}$.
   \State $\Confidence_{\man, \woman} \gets \bigl[\mean_{\man}(\woman)-8\sqrt{{\log(|\Agents|\Rounds)} / {\pulls_{\man\woman}}},  \mean_{\man}(\woman) + 8\sqrt{{\log(|\Agents|\Rounds)}/{\pulls_{\man\woman}}}\,\bigr]\cap [-1, 1]$
   \State $\Confidence_{\woman, \man} \gets \bigl[\mean_{\woman}(\man)-8\sqrt{{\log(|\Agents|\Rounds)}/{\pulls_{\man\woman}}},  \mean_{\woman}(\man) +8\sqrt{{\log(|\Agents|\Rounds)}/{\pulls_{\man\woman}}} \,\bigr]\cap [-1, 1]$ 
   \EndFor
   \EndFor
\EndProcedure
\end{algorithmic}
\end{algorithm}

To analyze \textsc{MatchUCB}, recall that \Cref{lemma:confset} bounds the regret at each step by the lengths of the confidence intervals of each pair in the selected matching. Bounding the lengths of the confidence intervals parallels the analysis of UCB for classical stochastic multi-armed bandits. We give the full proof of \Cref{thm:instanceind} in \Cref{appendix:proofunstructured}.

\paragraph{Typed Preferences.} For this setting, we construct our confidence intervals as follows: for each pair of types $\context_1$ and $\context_2$, we take a length \smash{$O\bigl(\sqrt{\log(|\Agents|\Rounds) / \smash{\pulls_{\context_1\context_2}}}\bigr)$} confidence interval centered around the empirical mean, where $\pulls_{\context_1\context_2}$ is the number of times that an agent with type $\context_1$ has been matched with an agent with type $\context_2$. We describe this construction precisely in \Cref{alg:typeducb} (\textsc{MatchTypedUCB}). We give the full proof of \Cref{thm:typed} in \Cref{appendix:prooftyped}. 

\begin{algorithm}
\caption{\textsc{MatchTypedUCB}: A bandit algorithm for matching with transferable utilities for typed preferences.}\label{alg:typeducb}
\begin{algorithmic}[1]
\Procedure{MatchTypedUCB}{$T$}
   \For{$(\context, \context') \in \Contexts \times \Contexts$} \Comment{Initialize confidence intervals and empirical means.}
   \State $\Confidence_{\context, \context'} \gets [-1, 1]$
   \EndFor
   \For{$1 \le \round \le\Rounds$}
   \State $(\MSet^\round, \transfers^\round) \gets \Call{ComputeMatch}{\mathscr{C}}$
   \For{$(\man, \woman) \in \MSet^\round$} \Comment{Set confidence intervals and update means.}
   \State Update empirical means $\hat{f}(\context_\man, \context_\woman)$ and $\hat{f}(\context_\man, \context_{\woman})$ from feedback; increment $\pulls_{\context_\man, \context_\woman}$.
   \State $\Confidence_{\context_\man, \context_{\woman}} \gets  \bigl[\hat{f}(\context_\man, \context_{\woman})-8\sqrt{{\log(|\Agents|\Rounds)} / {\smash{\pulls_{\context_\man,\context_\woman}}}}, \hat{f}(\context_\man, \context_{\woman}) + 8\sqrt{{\log(|\Agents|\Rounds)}/{\smash{\pulls_{\context_\man,\context_\woman}}}}\,\bigr]\cap [-1, 1]$
   \State $\Confidence_{\context_\woman, \context_\man} \gets \bigl[\hat{f}(\context_\woman, \context_{\man}) -8\sqrt{{\log(|\Agents|\Rounds)}/{\smash{\pulls_{\context_\man, \context_\woman}}}}, \hat{f}(\context_\man, \context_{\woman})+8\sqrt{{\log(|\Agents|\Rounds)}/{\smash{\pulls_{\context_\man,\context_\woman}}}} \,\bigr]\cap [-1, 1]$ 
   \EndFor
   \EndFor
\EndProcedure
\end{algorithmic}
\end{algorithm}

\paragraph{Separable Linear Preferences.} 

\begin{algorithm}
\caption{\textsc{MatchLinUCB}: A bandit algorithm for matching with transferable utilities for separable linear preferences.}\label{alg:linucb}
\begin{algorithmic}[1]
\Procedure{MatchLinUCB}{$T$}
    \For{$(\man, \woman) \in \Men \times \Women$} \Comment{Initialize confidence intervals.}
   \State $\Confidence_{\man,\woman} \gets [-1, 1]$
   \State $\Confidence_{\woman, \man} \gets [-1, 1]$
   \EndFor
   \For{$1 \le \round \le\Rounds$}
   \State $(\MSet^\round, \transfers^\round) \gets \Call{ComputeMatch}{\mathscr{C}}$
   \For{$\agent \in \Agents^\round$} \Comment{Update confidence intervals.}
   \State Increment the counter $\pulls_\agent$.
     \State $\beta  \gets O\bigg(d \log T + \frac{\pulls_\agent \sqrt{\ln(\pulls_\agent/(T |A|))}}{T^2}\bigg)$. \Comment{Parameter for width of confidence set.}
   \If{$\mu_{\MSet^\round}(\agent) \neq \agent$}
   \State Add $t$ to $\mathcal{T}_{a}$ (the set of rounds in which agent $\agent$ has been matched).
   \State Set $\mathcal{R}_{\agent,\round}$ equal to the observed utility for agent $\agent$ in round $t$. 
   \State $\phi^{\LS}(a)  \gets \text{argmin}_{v \in \mathcal{B}^d} \left(\sum_{t' \in \mathcal{T}_{a}} \bigl(\langle v, \context_{\mu_{X_{t'}}(\agent)} \rangle - \mathcal{R}_{\agent,\round'}\bigr)^2 \right)$ \Comment{Least squares estimate.}
   \State $C_{\phi(\agent)}  \gets \Bigl\{v \bigm\vert  \sum_{t' \in \mathcal{T}_{\agent}} \bigl(\langle v - \phi^{\LS}(\agent), \context_{\mu_{X_{t'}}(a)} \rangle\bigr)^2 \le \beta, \|v\|_2\le 1\Bigr\}$ \Comment{Conf. ellipsoid.}
   \For{$\agent' \in \Agents$}
   \State $C_{\agent, \agent'}  \gets \bigl\{\langle \context_{a'}, v \rangle \mid v \in C_{\phi(\agent)} \bigr\} \cap [-1, 1]$ \Comment{Update confidence sets of agent $\agent$.}
   \EndFor
   \EndIf
   \EndFor
   \EndFor
\EndProcedure
\end{algorithmic}
\end{algorithm}

To build the confidence sets, we use a key idea from the design of LinUCB \cite{RV13, LS2020}. The idea is to compute a confidence set for each hidden vector $\phi(a)$ using the least squares estimate and use that to construct confidence sets $C_{a,a'}$ for the utilities.

More formally, let $\mathcal{T}_{a}$ be the set of rounds where agent $a$ is matched on the platform thus far, and for $t' \in \mathcal{T}_a$, let $\mathcal R_{a,t'}$ be the observed utility at time $t'$ for agent $a$. The center of the confidence set will be given by the least squares estimate
\[   \phi^{\LS}(\agent) = \argmin_{v \in \mathcal{B}^d} \left(\sum_{t' \in \mathcal{T}_{a}} (\langle v, \context_{\mu_{X_{t'}}(a)} \rangle - \mathcal{R}_{a,t'} \right). \]
The confidence set for $\phi(\agent)$ is given by
\[C_{\phi(a)} := \left\{v \Biggm\vert  \sum_{t' \in \mathcal{T}_{a,t}} \Big\langle v - \phi^{\LS}(a), \context_{\mu_{X_{t'}}(a)} \Big\rangle^2 \le \beta\text{ and } \|v\|_2\le 1 \right\},  \]
where $\beta = O\left(D \log T + \frac{\pulls_{\agent} \sqrt{\ln(\pulls_{\agent} /\delta)}}{T^2}\right)$ and $\pulls_{\agent}$ counts the number of times that $\agent$ has appeared in selected matchings. The confidence set for $u(a, a')$ is given by
\[
C_{a, a'} := \left\{\langle \context_{a'}, v \rangle \mid v \in C_{\phi(a)} \right\} \cap [-1, 1]. \]
We describe this construction precisely in \Cref{alg:linucb} (\textsc{MatchLinUCB}).
We give the full proof of \Cref{thm:linear} in \Cref{appendix:prooflinear}.

\subsection{Matching lower bound}\label{sec:lowerbound}

For the case of unstructured preferences, we now show that \textsc{MatchUCB} achieves optimal regret (up to logarithmic factors) by showing a lower bound that (nearly) matches the upper bound in Theorem \ref{thm:instanceind}. 
\begin{restatable}{lemma}{lowerbound}
\label{lemma:lowerbound}
For any algorithm that learns a stable matching with respect to unstructured preferences, there exists an instance on which it has expected regret {$\widetilde{\Omega}(|A|^{3/2} \sqrt{T})$} (where regret is given by \instabmeasure{}). 
\end{restatable}
The idea behind this lemma is to show a lower bound for the easier problem of learning a maximum weight matching using utility difference as regret. By Proposition \ref{prop:desiderata}, this immediately implies a lower bound for learning a stable matching with regret measured by \instabmeasure{}.

This lower bound illustrates the close connection between our setting and that of learning a maximum weight matching. Indeed, by applying \textsc{MatchUCB} and simply disregarding the transfers every round, we recover the classical UCB-based algorithm for learning the maximum weight matching \cite{GKJ12, DBLP-conf/icml/ChenWY13, DBLP-conf/aistats/KvetonWAS15}. From this perspective, the contribution of \textsc{MatchUCB} is an approach to set the dual variables while asymptotically maintaining the same regret as the primal-only problem.

\section{Extensions}

In this section, we discuss several extensions of our results: instance-dependent regret bounds, connections between subset instability and platform revenue, and non-transferable utilities. These extensions illustrate the generality of our framework and also suggest several avenues for future research.

In \Cref{appendix:instancespecific}, we derive instance-dependent regret bounds for Subset Instability, which allow us to improve the \smash{$O(\sqrt T)$} convergence from \Cref{sec:regret} to $O(\log T)$ for any fixed instance. Achieving this logarithmic bound involves choosing ``robust'' dual solutions when setting transfers (rather than choosing an arbitrary optimal primal-dual pair as in \textsc{ComputeMatch}): we want our selected primal-dual pair to lead to stable outcomes even under perturbations of the transfers.

In \Cref{appendix:alternatedef}, we connect the subsidy perspective of Subset Instability to platform revenue. We relate regret to platform revenue and show that, when there are search frictions, the platform can  
achieve substantial long-run profit despite starting with no knowledge of agent preferences. 

In \Cref{appendix:matchntu}, we adapt our framework to matching with non-transferable utilities (where agents do not transfer money to other agents on the platform). We define an analogue of Subset Instability using the subsidy formulation and give an \smash{$\widetilde{O}(\sqrt{T})$} regret algorithm for learning stable matchings.

\subsection{Instance-dependent regret bounds}\label{appendix:instancespecific}

While our analyses in \Cref{sec:algorithm} focused on bounds that hold uniformly for all problem instances, we now explore \textit{instance-dependent} regret bounds. Instance-dependent bounds capture a different facet of bandit algorithms: how does the number of mistakes made by the algorithm scale on each instance with respect to $\Rounds$? Bounds of this nature have been explored in previous works \cite{LMJ20, BSS21, SBS20, CS21, LRMJ20} on learning stable matchings in the non-transferable utilities setting, and we show that they can be obtained within our framework as well.

Our instance-dependent regret bound depends on a gap $\Delta > 0$ determined by the true utility function $\utility$. We focus on the setting where agent utilities are unstructured (i.e., $\utility \in \UUnstructured$) and where the same set of agents $\Agents$ arrives in each round. As is common in analyses of combinatorial bandit problems (e.g., \cite{DBLP-conf/aistats/KvetonWAS15, DBLP-conf/icml/ChenWY13}), the gap $\Delta$ in the bound is global to the matching. Letting $\MSet^{\text{opt}}$ be a maximum weight matching with respect to $\utility$, we define the gap to $\Delta$ be the difference in utility between the optimal and second-best matchings\footnote{Our bound is less fine-grained than the gap in \citep{DBLP-conf/icml/ChenWY13}, and in particular does not allow there to be multiple maximum weight matchings. We defer improving our definition of $\Delta$ to future work.}:
\[ \Gap \coloneqq \inf_{\MSet\neq\MSet^{\text{opt}}} \Biggl\{ \sum_{\agent\in\Agents} \utility_\agent(\mu_{\MSet^{\text{opt}}}(\agent)) - \sum_{\agent\in\Agents} \utility_\agent(\MMap(\agent)) \Biggr\}. \] 
We prove the following regret bound:
\begin{restatable}[Instance-Dependent Regret]{theorem}{instancespecific}
\label{thm:instancespecific}
Suppose that $\Agents_\round = \Agents$ for all $\round$. Let $\utility \in \UUnstructured$ be any utility function, and put
\[ \Gap \coloneqq \inf_{\MSet\neq\MSet^*} \Biggl\{ \sum_{\agent\in\Agents} \utility_\agent(\mu_{\MSet^*}(\agent)) - \sum_{\agent\in\Agents} \utility_\agent(\MMap(\agent)) \Biggr\}. \] 
Then {\normalfont\textsc{MatchUCB$'$}} incurs expected regret $\mathbb{E}(\Regret_\Rounds) = O(|\Agents|^5 \cdot\log(|\Agents|\Rounds)/\Delta^2)$.
\end{restatable}

\begin{remark*}
\textsc{MatchUCB$'$} is \textsc{MatchUCB} with a slight adjustment to \textsc{ComputeMatch} needed to prove \Cref{thm:instancespecific}. \textsc{MatchUCB$'$}, like \textsc{MatchUCB}, does {not} depend on the gap $\Delta$ and achieves the instance-independent regret bound in \Cref{thm:instanceind}.\footnote{The instance-independent regret bound can be shown using the same argument as the proof for \Cref{thm:instanceind}.} That is, \textsc{MatchUCB$'$} achieves both our instance-independent and instance-dependent regret bounds. 
\end{remark*}

Our starting point for proving \Cref{thm:instancespecific} is to upper bound the number of ``mistakes'' that a platform makes while exploring and learning, i.e., the number of rounds where the chosen matching is suboptimal. That is, we bound the number of rounds where the chosen market outcome is not stable with respect to the true utilities $u$. This is similar in spirit to the analysis of the combinatorial bandits problem of learning a maximum weight matching in \citep{DBLP-conf/icml/ChenWY13}. However, a crucial difference is that a mistake can be incurred even when the selected matching is optimal, if the selected transfers do not result in a stable market outcome. Ensuring that the selected transfers result in a stable market outcome when the utility estimates are sufficiently accurate is the main technical hurdle in our analysis. 

To make this argument work, we need to specify more precisely how the primal-dual solution is chosen in line 5 of \textsc{ComputeMatch} (which we previously did not specify). In particular, poor choices of the primal-dual solution can lead to many rounds where the chosen outcome is unstable, because the transfers violate the stability constraints. To see this, consider a market with a single customer $C$ and a single provider $P$ such that $\utility_C(P) = 2$ and $\utility_P(C) = -1$, and suppose we have nearly tight upper bounds $\uUCB_C(P) = 2 + \epsilon$ and $\uUCB_P(C) = -1 + \epsilon$ on the utilities. Then the market outcome with matching $\{(C, P)\}$ with $\transfers_C = -2 - \epsilon$ and $\transfers_P = -\transfers_C$ could be selected by \textsc{ComputeMatch}, since it corresponds to an optimal primal-dual pair for $\uUCB$. However, it is not stable with respect to the true utilities $\utility$ (as individual rationality is violated for $C$), regardless of how small $\epsilon$ is. Thus, without assuming more about how the optimal primal-dual pair is chosen in \textsc{ComputeMatch}, we cannot hope to bound the number of unstable market outcomes selected. 

We show that, by carefully selecting an optimal primal-dual pair each round, we can bound the number of mistakes. In particular, we design an algorithm {\normalfont\textsc{ComputeMatch$'$}} to find primal-dual pairs that satisfy the following property: if the confidence sets are small enough, then the selected matching will be stable with respect to the true utilities.  
\begin{restatable}{lemma}{main}
\label{lemma:main}
Suppose {\normalfont\textsc{ComputeMatch$'$}} is run {on a collection $\mathscr{C}$ of confidence sets $\Confidence_{i,j}$ and $\Confidence_{j,i}$ over the agent utilities that satisfy}
\[
\max\bigl(\Confidence_{\man,\woman}\bigr) - \min\bigl(\Confidence_{\man,\woman}\bigr) \le \todoconstmod{} \frac{\Delta}{|\Agents|} \text{\quad and\quad} \max\bigl(\Confidence_{\woman, \man}\bigr) - \min\bigl(\Confidence_{\woman, \man}\bigr) \le \todoconstmod{} \frac{\Delta}{|\Agents|}
\]
for all $(\man, \woman)$ in the matching returned by {\normalfont\textsc{ComputeMatch$'$}}. Suppose also that the confidence sets $\mathscr{C}$ contain the true utilities for all pairs of agents. Then the market outcome returned by {\normalfont\textsc{ComputeMatch$'$}} is stable with respect to the true utilities $\utility$. 
\end{restatable}
\begin{remark*}
\Cref{lemma:main} does \textit{not} hold for {\normalfont\textsc{ComputeMatch}}; {its proof relies on} the particular specification of the optimal primal-dual pair in {\normalfont\textsc{ComputeMatch$'$}}. 
\end{remark*}
\noindent Using Lemma \ref{lemma:main}, we intuitively can bound the number of mistakes made by the algorithm by the number of samples needed to sufficiently reduce the size of the confidence sets. In \Cref{appendix:proofinstancespecific}, we describe how we choose optimal primal-dual pairs in \textsc{ComputeMatch$'$}, prove Lemma \ref{lemma:main}, and provide a full proof of \Cref{thm:instancespecific}.

\Cref{thm:instancespecific} opens the door to further exploring algorithmic properties of learning stable matchings. First, this result establishes fine-grained regret bounds, demonstrating the typical $O(\log T)$ regret bounds from the combinatorial bandits literature \cite{DBLP-conf/icml/ChenWY13} are achievable in our setting as well. Second, \Cref{thm:instancespecific} provides insight into the number of mistakes made by the platform. In particular,  we show within the proof of \Cref{thm:instancespecific} that the platform fails to choose a matching that is stable with respect to $u$ in at most $O(|\Agents|^4 \cdot\log(|\Agents|\Rounds)/\Delta^2)$ rounds.\footnote{The number of mistakes necessarily depends on the gap $\Delta$ because there exist utility functions $u$ and $\tilde{u}$ where $\norm{u - \tilde{u}}_{\infty}$ is arbitrary small, but where the stable market outcomes with respect to $u$ and $\tilde{u}$ differ. To see this, consider a market where $\Men = \{C\}$ and $\Women = \{P\}$. Suppose that $\utility_C(P) = \tilde\utility_C(P) = 1$, while $\utility_{P}(C) = -1 + \epsilon$ and $\tilde\utility_P(C) = -1 - \epsilon$. Then, the maximum weight matchings under these utility functions differ: $\{(C, P)\}$ is the only maximum weight matching in the former, whereas $\emptyset$ is the only maximum weight matching in the latter.} This means that the platform selects a stable matching in at least $T - O(|\Agents|^5 \cdot\log(|\Agents|\Rounds)/\Delta^2) = T - O(\log T)$ of the rounds. 

As we described, our bounds in \Cref{thm:instancespecific} rely on choosing an appropriate primal-dual solution. An interesting direction for future work would be to provide further insight into how different methods for finding optimal primal-dual pairs affect both regret bounds and the trajectory of the selected market outcomes over time.

\subsection{Search frictions and platform revenue}\label{appendix:alternatedef}
Next, we further ground \instabmeasure{} by explicitly connecting it to the platform's revenue under a stylized economic model of search frictions. A major motivation for this is that it helps explain when 
an online platform can earn a profit in competitive settings, even 
when they start out with no information about agent preferences.

More specifically, we incorporate search frictions where an agent must lose utility $\epsilon$ in order to find an alternative to the given match (e.g., from the time spent finding an alternate partner, or from a cancellation fee). These search frictions weaken the requirements for stability: the platform now only needs matchings to %
be \textit{$\epsilon$-stable}:
\[ \utility_\man(\woman) + \utility_\woman(\man) - 2\epsilon \le \utility_\man(\MMap(\man)) + \transfers_\man  + \utility_\woman(\MMap(\woman)) + \transfers_\woman \] for all $(\man, \woman) \in \Men \times \Women$ and $\utility_\agent(\MMap(\agent)) + \transfers_\agent \ge -\epsilon $ for all $\agent\in\Agents$.\footnote{This definition corresponds to $(\MSet, \transfers)$ belonging to the weak $\epsilon$-core of \citet{SS66}. We note that this definition also relaxes individual rationality. This formulation gives us the cleanest algorithmic results; while it can be extended to an analogue that does not relax individual rationality, it would involve bounds that (necessarily) depend on the specifics of agents' utilities.}

To model revenue, we take the subsidy perspective on \instabmeasure{}. Specifically, recall that \instabmeasure{} is equal to the minimum 
subsidy needed to maintain stability (see \Cref{prop:subsidy}). With search frictions, that subsidy can potentially be \emph{negative}, 
thus allowing the platform to generate revenue. We are interested 
in analyzing the maximum revenue (minimum subsidy) the platform can 
generate while ensuring stability with high probability over all 
rounds. For realism, we also want this subsidy to be computed 
online using only information that the platform has access to, 
but it turns out we can do this with minimal modifications to 
our algorithm.

More formally, in this modified model, the platform must select an $\epsilon$-stable matching in each round with high probability by choosing appropriate subsidies. That is, in round $\round$, the platform selects a matching with transfers $(\MSet^\round, \transfers^\round)$ with the modification that the transfers need not be zero-sum. The transfers thus incorporate the amount that platform is subsidizing or charging agents for participation on the platform. The net profit of the platform is then $-\sum_{\round=1}^{\Rounds} \sum_{\agent \in \Agents} \transfers_\agent^\round$.
We impose the stability requirement that 
\[\mathbb{P}[(\MSet^\round, \transfers^\round) \text{ is $\epsilon$-stable for all } 1 \le \round \le \Rounds] \ge 0.99. \] 
Given this setup, we show the following:
\begin{restatable}{theorem}{searchfricinf}
\label{thm:searchfricinf}
For preference class $\UUnstructured$ (see \Cref{subsec:bandits}), there exists an algorithm giving the platform 
\[\epsilon T \sum_{t=1}^T |\Agents_t| - O\Bigl( |\Agents| \sqrt{n \Rounds} \sqrt{\log(|\Agents| |\Rounds|)}\Bigr)\]
revenue in the presence of search frictions while maintaining stability with high probability.  
\end{restatable}
\begin{remark*}
In particular if $\Agents_t = \Agents$ in every round, the platform will starting making a profit within $O(|\Agents|/\epsilon^2\cdot\log(|\Agents|/\epsilon^2))$ rounds.
\end{remark*}
\noindent We defer the proof of \Cref{thm:searchfricinf} to \Cref{appendix:proofrevenue}.

Qualitatively, Theorem \ref{thm:searchfricinf} captures that if the platform ``pays to learn'' in initial rounds, the information that it obtains will help it achieve a profit in the long run. We note that both the revenue objective and the model for search frictions that we consider in these preliminary results are stylized. An interesting direction for future work would be to integrate more realistic platform objectives and models for search frictions into the framework.

\subsection{Matching with non-transferable utilities}\label{appendix:matchntu}

While we have focused on matching with transferable utilities, utilities are not always transferable in practice, as in the cases of dating markets and college admissions (i.e.,~most people are not willing to date an undesirable partner in exchange for money, and a typical college admission slot is not sold for money). We can extend our findings to this setting following the model of matching with non-transferable utilities (NTU) \cite{GS62}, which has also been studied in previous work \cite{DK05, LMJ20, CS21, SBS20}. The definition of \instabmeasure{} extends naturally and has advantages over the ``utility difference'' metric that is commonly used in prior work. Our algorithmic meta-approach also sheds new light on the convergence properties of the centralized UCB algorithm of \citet{LMJ20}.

The starting point of our instability measure is slightly different than in Section \ref{sec:instab}. Since stable matchings in the NTU model need not maximize total utility, we cannot define instability based on a maximum over all subsets of agents of the utility difference for that subset. On the other hand, the subsidy formulation of \instabmeasure{} (see \eqref{eq:instabilityrecast}) translates well to this setting. Our instability measure will correspond to the minimum amount the platform could subsidize agents so that individual rationality holds and no blocking pairs remain. For matching with NTU, we formalize this notion as follows:  
\begin{definition}[\NTUinstabmeasure{}]
\label{definition:instabilityNTU}
For utilities $\utility$ and agents $\Agents$, the \emph{\NTUinstabmeasure{}}  $\InstabilityNTU{\utility}{\MSet}$ of a matching $\MSet$ is
\begin{alignat*}{5}\label{eq:instabilityNTU}
	\min_{\subsidies \in\R^{|\Agents|}} & \sum_{\agent \in \Agents} \subsidies_{\agent} & \tag{$\dagger$} \\
	\text{s.t.}\ \ & \min \bigl(\utility_\man(\woman) - \utility_\man(\MMap(\man)) - \subsidies_\man, \utility_\woman(\man) - \utility_\woman(\MMap(\woman)) - \subsidies_\woman\bigr)\le 0 \qquad&&\forall (\man,\woman)\in\Men\times\Women \\
	& \utility_\agent(\MMap(\agent)) + \subsidies_\agent\ge 0\quad&&\forall\agent\in\Agents \\
	&\subsidies_\agent\ge 0 \quad&&\forall\agent\in\Agents.
\end{alignat*}
\end{definition}

NTU Subsidy Instability inherits some of the same appealing properties as Subsidy Instability. 
\begin{proposition}[Informal]\label{prop:desiderataNTU}
\NTUinstabmeasure{} satisfies the following properties:
\begin{enumerate}
    \item \NTUinstabmeasure{} is always nonnegative and is zero if and only if $(\MSet, \transfers)$ is stable.
    \item \NTUinstabmeasure{} is Lipschitz continuous with respect to agent utilities.  That is, for any matching $\MSet$ and any pair of utility functions $\utility$ and $\tilde{\utility}$, it holds that:  \[|\InstabilityNTU{\utility}{\MSet} - \InstabilityNTU{\utilityii}{\MSet}| \le 2\sum_{\agent\in\Agents} \norm{\utility_\agent - \utilityii_\agent}_\infty.\]
\end{enumerate}
\end{proposition}
\noindent The proofs of this and subsequent results are deferred to \Cref{appendix:proofmatchNTU}. Together, the preceding properties mean that NTU Subsidy Instability is useful as a regret measure for learning stable matchings.

As in the transferable utilities setting, Property 2 implies the existence of an explore-then-commit algorithm with \smash{$\widetilde{O}(|A|^{4/3} T^{2/3})$} regret.
We show that this can be improved to a $\sqrt{T}$ dependence by adapting our approach from Section \ref{sec:regret}:
\begin{theorem}
\label{thm:matchingNTU}
For matchings with non-transferable utilities, there exists an algorithm that for any utility function $\utility$ incurs regret \smash{$\Regret_\Rounds = O({|\Agents|}^{3/2} \sqrt{\Rounds} \sqrt{\log(|\Agents| \Rounds)})$}.
\end{theorem}

While Theorem \ref{thm:matchingNTU} illustrates that our approach easily generalizes to the NTU setting, we highlight two crucial differences between these settings. First, learning a stable matching is incomparable to learning a maximum weight matching because stable matchings do not
maximize the sum of agents' utilities
in the NTU setting. Next, the instability measure is \textit{not} equivalent to the cumulative unhappiness of agents, unlike in the setting with transferable utilities. Intuitively, these definitions cease to be equivalent because non-transferable utilities render the problem more ``discontinuous'' and thus obstruct the duality results we applied earlier. 

These results provide a preliminary application of our framework to the setting of matching with non-transferable utilities; an interesting direction for future inquiry would be to more thoroughly investigate notions of approximate stability and regret in this setting. 

\subsubsection{Comparison to the utility difference measure}\label{subsubsec:utility}
It turns out that the algorithm underlying \Cref{thm:matchingNTU} 
is equivalent to the centralized UCB algorithm from previous work \citep{LMJ20, CS21}, albeit derived from a different angle. However, an important difference is that 
\Cref{thm:matchingNTU} guarantees low regret relative to the incentive-aware \NTUinstabmeasure{}, as opposed to the 
incentive-unaware ``utility difference'' measure in prior work. In this section, we outline several properties that make our instability measure more suitable especially in the NTU setting. In particular, we show for utility difference that:
\begin{enumerate}[(a)]
    \item There is no canonical formalization of utility difference when multiple stable matchings exist. 
    \item The utility difference of a matching can be positive even if the matching is stable and negative even if the matching is unstable. 
    \item Even when restricting to markets with unique stable matchings, the utility difference of a matching can be discontinuous in the true agent utilities. As a result, it does not allow for instance-independent regret bounds that are sublinear in $\Rounds$.
\end{enumerate} 

For (a), the utility difference requires specifying a stable matching to serve as a benchmark against which to measure relative utility. However, when multiple stable matchings exist, some ambiguity arises as to which one should be chosen as the benchmark. Because of this, previous works \cite{DK05, LMJ20, CS21, SBS20} study two different benchmarks. In particular, they assume providers' preferences are known and benchmark with respect to the customer-optimal and customer-pessimal stable matchings. For (b), notice that the utility difference for the maximum weight matching is negative, even though it is typically not stable in the NTU setting. Moreover, because of the ambiguity in the benchmark from (a), the utility difference may not be zero even when the matching is stable. For (c), to see that utility difference is not continuous as a function of the underlying agent utilities, consider the following example:
\begin{example}
Consider a market where is a single customer $\man$ and two providers $\woman_1$ and $\woman_2$. Suppose their utility functions are given by $\utility_{\man}(\woman_1) = \epsilon$, $\utility_{\man}(\woman_2) = 2\epsilon$, $\utility_{\woman_1}(\man) = 1$, and $\utility_{\woman_2}(\man) = 0.5$. Then the unique stable matching $\{(\man, \woman_2)\}$ has total utility $0.5 + 2\epsilon$. Now, consider the perturbed utility function $\utilityii$ such that $\utilityii_{\man}(\woman_1) = 2\epsilon$, $\utilityii_{\man}(\woman_2) = \epsilon$, $\utilityii_{\woman_1}(\man) = 1$, and $\utilityii_{\woman_2}(\man) = 0.5$. For this perturbed utility function, the unique stable matching is $\{(\man, \woman_1)\}$, which has total utility $1 + 2\epsilon$. The utility difference (either optimal or pessimal) for matching $\{(\man, \woman_2)\}$ is $0$ for $\utility$ and $0.5 + \epsilon$ for $\utilityii$. Since this holds for any $\epsilon > 0$, taking $\epsilon\to 0$ shows that utility difference is not continuous in the utility function.
\end{example}

That utility difference is discontinuous in agent utilities rules out the existence of bandit algorithms that achieve sublinear instance-independent regret when using utility difference as the regret measure. In particular, the analyses in previous work \cite{LMJ20, SBS20, CS21, LRMJ20} focus entirely on \textit{instance-dependent} regret bounds. They show that centralized UCB achieves logarithmic instance-dependent regret with respect to the utility difference relative to the customer-pessimal stable matching (but does not achieve sublinear regret with respect to the customer-optimal stable matching). Our insight here is that a new measure of instability can present a more appealing evaluation metric and paint a clearer picture of an algorithm's convergence to the \emph{set} of stable matchings as a whole.

\section{In what settings are equilibria learnable?}

A core insight of our work is that, in a stochastic environment, ``optimism in the face of uncertainty'' can be effectively leveraged for the problem of learning stable matchings. This motivates us to ask: in what other settings, and with what other algorithmic methods, can equilibria be learned? 

One interesting open direction is to understand when equilibria can be learned in \emph{adversarial} environments where the utility functions can change between rounds. From an economic perspective, adversarial environments could capture evolving market conditions. In the adversarial bandit setting, most work relies on gradient-based algorithms instead of UCB-based algorithms to attain optimal regret bounds (see, e.g., \cite{DBLP:journals/siamcomp/AuerCFS02, AHR08}). Can these gradient-based algorithms 
similarly be adapted to \instabmeasure{}?

Another interesting open direction is to consider more general market settings, even in stochastic environments. For example, within the context of matching markets, each agent might match to more than one agent on the other side of the market; and outside of matching markets, a buyer might purchase multiple units of multiple goods. In markets with transferable utilities, incentive-aligned outcomes can be captured by \textit{Walrasian equilibria} (see, e.g., \cite{bichler21walrasian}). Can \instabmeasure{} and our UCB-based algorithms be adapted to learning Walrasian equilibria in general?

Addressing these questions would provide a richer understanding of when and how large-scale, data-driven marketplaces can efficiently learn market equilibria.

\section*{Acknowledgments}
We would like to thank Itai Ashlagi, Jiantao Jiao, Scott Kominers (along with the Lab for Economic Design), Cassidy Laidlaw, Celestine Mendler-D\"unner, and Banghua Zhu for valuable feedback. Meena Jagadeesan acknowledges support from the Paul and Daisy Soros Fellowship; Alexander Wei acknowledges support from an NSF Graduate Research Fellowship  under Grant No.\ 1752814; and Michael Jordan acknowledges the Vannevar Bush Faculty Fellowship program
under grant number N00014-21-1-2941.

\printbibliography

\appendix

\section{Classical Results for Matching with Transferable Utilities}\label{appendix:recap}

To be self-contained, we briefly state and prove the key results from \citet{SS71} we need.

First, we explicitly relate the primal-dual formulation in  \Cref{sec:algorithm} to stable matchings.
\begin{theorem}[\cite{SS71}]
\label{thm:pd}
If $(\MSet, \transfers)$ is stable, then $(\MMatrix, \prices)$ is an optimal primal-dual pair to (P) and (D), where $\prices_\agent = \transfers_\agent + \utility_\agent(\MSet(\agent))$ and $\MMatrix$ is the indicator matrix in $\R^{\Men\times\Women}$ corresponding to $\MSet$.

Moreover, if $(\MMatrix, \prices)$ is an optimal primal-dual pair to (P) and (D) such that $\MMatrix$ lies at an extreme point of the feasible set, then $(\MSet, \transfers)$ is stable where $\transfers_\agent = \prices_\agent - \utility_\agent(\MSet(\agent))$ and $\MSet$ is the matching corresponding to the nonzero entries of $\MMatrix$.
\end{theorem}

\begin{proof}
Both statements follow from the complementary slackness conditions and the definition of stability in \Cref{def:stability}. The complementary slackness conditions are:
\begin{itemize}
    \item If $\MMatrix_{\man, \woman} > 0$, then $\prices_\man + \prices_\woman = \utility_\man(\woman) + \utility_\woman(\man)$.
    \item If $\prices_\man > 0$, then $\sum_\woman \MMatrix_{\man, \woman} = 1$.
    \item If $\prices_\woman > 0$, then $\sum_\man \MMatrix_{\man, \woman} = 1$.
\end{itemize}

Suppose that $(\MSet, \transfers)$ is stable. Let us first show that $(\MMatrix, \prices)$ is feasible. We see that $\MMatrix$ is primal feasible by definition. For dual feasibility, since there are no blocking pairs, we know that
\[ \bigl(\utility_\man(\MMap(\man)) + \transfers_\man\bigr) + \bigl(\utility_\woman(\MMap(\woman)) + \transfers_\woman\bigr)\ge \utility_\man(\woman) + \utility_\woman(\man),\]
which implies
\[\prices_\man + \prices_\woman \ge \utility_\man(\woman) + \utility_\woman(\man). \] The individual rationality condition $\utility_\agent(\MMap(\agent)) + \transfers_\agent \ge 0$ tells us $\prices_\agent \ge 0$. Hence $\prices$ is dual feasible. Next, we show that $(\MMatrix, \prices)$ is an optimal primal-dual pair by checking the Karush–Kuhn–Tucker conditions. We have already shown primal and dual feasibility, so it suffices to show complementary slackness. The first condition follows from zero-sum transfers. To see the second and third conditions, we show the contrapositive: If $\man\in\Men$ is such that $\sum_\woman \MMatrix_{\man,\woman} < 1$, then $\sum_\woman \MMatrix_{\man,\woman} = 0$ by our assumption on $\MMatrix$. Hence $\man$ is unmatched (i.e.,  $\utility_\man(\MMap(\man)) = 0$ and $\transfers_\man = 0$) which implies $\prices_\man = 0$. The analogous argument applies for $\woman\in\Women$.

We now prove the second part of the theorem. Suppose $(\MMatrix, \prices)$ is an optimal solution to (P) and (D) such that $\MMatrix$ is at a vertex. By the Birkhoff-von Neumann theorem, since $\MMatrix$ is a vertex, it corresponds to a matching. We wish to show that $(\MSet, \transfers)$ has no blocking pairs, is individually rational, and has zero-sum transfers. Dual feasibility tells us that:
\[\prices_\man + \prices_\woman \ge \utility_\man(\woman) + \utility_\woman(\man)\] which means that:
\[\bigl(\utility_\man(\MMap(\man)) + \transfers_\man\bigr) + \bigl(\utility_\woman(\MMap(\woman)) + \transfers_\woman\bigr)\ge \utility_\man(\woman) + \utility_\woman(\man),\] so there are no blocking pairs. Dual feasibility also tells us that $\prices_\agent \ge 0$, which means that $\utility_\agent(\MMap(\agent)) + \transfers_\agent \ge 0$, so individual rationality is satisfied. To show that there are zero-sum transfers, we use complementary slackness. The first complementary slackness condition tells us that if $\MMatrix_{\man, \woman} > 0$, then $\prices_\man + \prices_\woman = \utility_\man(\woman) + \utility_\woman(\man)$. Using the fact that $\MMatrix$ corresponds to a matching, this in particular means that if $(\man, \woman) \in \MSet$, we know $\transfers_\man + \transfers_\woman = 0$. To show that agents who are unnmatched receive $0$ transfers, let's use the second and third complementary slackness conditions. The contrapositive tells us that if $\agent$ is unmatched, then $\prices_\agent = 0$, which implies $\transfers_\agent = 0$.
\end{proof}

Since (P) is exactly the maximum weight matching linear program, \Cref{thm:pd} immediately tells us that if $(\MSet, \prices)$ is stable, then $\MSet$ is a maximum weight matching. This means that stable matchings with transferable utilities maximize social welfare.

\section{Proofs for Section \ref{sec:instab}}\label{appendix:structural}

This section contains further exposition (including proofs) for \Cref{sec:instab}. 

\subsection{Limitations of utility difference as an instability measure}\label{appendix:limitations}

To illustrate why utility difference fails to be a good measure of instability, we describe a matching with transfers that (i) is far from stable, and (ii) has zero utility difference (but large \instabmeasure{}).

\begin{example}
Consider the following market with two agents: $\Men = \{\man\}$ and $\Women = \{\woman\}$. Suppose that $\utility_\man(\woman) = 2$ and $\utility_\woman(\man) = -1$. Consider the matching $\MSet = \{(\man, \woman)\}$ with transfers $\transfers_\man = -\xi$ and $\transfers_\woman = \xi$ for some $\xi > 0$. We will show that this matching with transfers will have the properties stated above when $\xi$ is large.

This matching with transfers has a utility difference equal to zero (for any $\xi$) since it maximizes the sum of utilities. Indeed, it is stable for any $\xi\in [1, 2]$. However, when $\xi > 2$, this matching with transfers is no longer stable, since the individual rationality condition $\utility_\man(\woman) + \transfers_\man\ge 0$ fails. (Intuitively, the larger $\xi$ is, the further we are from stability.) But its utility difference remains at zero.

On the other hand, the \instabmeasure{} of this matching with transfers is $\xi - 2 > 0$ when $\xi > 2$. In particular, \instabmeasure{} increases with $\xi$ in this regime, which is consistent with the intuition that outcomes with larger $\xi$ should be more unstable.
\end{example}

\subsection{Proof of \Cref{prop:subsidy}}\label{app:subsidy}

\subsidy*

\begin{proof}[Proof of \Cref{prop:subsidy}]
We can take the dual of the linear program \eqref{eq:instability} to obtain:
\begin{alignat*}{5}\label{eq:dual-instability}
    \max_{\substack{\CMatrix\in\R^{|\Men|\times|\Women|} \\ \IR \in \mathbb{R}^{|\Agents|}}} & \quad\ \ \mathrlap{\sum_{{(\man, \woman)\in\Men\times\Women}} \CMatrix_{\man,\woman}\Bigl(\bigl(\utility_\man(\woman) - \utility_\man(\MMap(\man)) - \transfers_\man\bigr) + \bigl(\utility_\woman(\man) - \utility_\woman(\MMap(\woman)) - \transfers_\woman\bigr)\Bigr)} & \tag{$\ddagger$} \\[-1em]
    &&& \qquad\qquad\mathrlap{- \sum_{\agent\in\Agents} \IR_{\agent}(\utility_\agent(\MMap(\agent)) + \transfers_\agent)} \\[1em]
    \text{s.t.}\ \ & \IR_{\man} + \sum_{\woman\in\Women} \CMatrix_{\man,\woman}\le 1&\forall\man\in\Men;\qquad\qquad&\IR_{\woman} + \sum_{\man\in\Men} \CMatrix_{\man,\woman}\le 1\qquad&\forall\woman\in\Women; \\
    & \CMatrix_{\man,\woman}\ge 0\qquad&\forall(\man,\woman)\in\Men\times\Women;\qquad\qquad&\IR_{\agent}\ge 0\qquad&\forall \agent \in \Agents.
    &\qquad\qquad.
\end{alignat*}
By strong duality, the optimal values of \eqref{eq:instability} and \eqref{eq:dual-instability} are equal. Thus, it suffices to show that  \instabmeasure{} is equal to \eqref{eq:dual-instability}. By \Cref{prop:unhappiness}, we know that \instabmeasure{} is equal to the maximum unhappiness of any coalition. Thus it suffices to show that \eqref{eq:dual-instability} is equal to the maximum unhappiness of any coalition.

To interpret \eqref{eq:dual-instability}, observe that there exist optimal $\CMatrix^*$ and $\IR^*$ all of whose entries lie in $\{0,1\}$ because this linear program can be embedded into a maximum weight matching linear program. Take such a choice of optimal $\CMatrix^*$ and $\IR^*$. Then, $\CMatrix^*$ is an indicator vector corresponding to a (partial) matching on a subset of the agents such that all pairs in this matching are blocking with respect to $(\MSet, \transfers)$. Similarly, $\IR^*$ is an indicator vector of agents who would rather be unmatched than match according to $(\MSet, \transfers)$. 

We first prove the claim that $\Instability{\utility}{\MSet}{\transfers}$ is at least \eqref{eq:dual-instability}. Based on the above discussion, the optimal objective of \eqref{eq:dual-instability} is obtained through $\CMatrix^*$ and $\IR^*$ that represent a matching and a subset of agents respectively. Let $S$ be the union of agents participating in $\CMatrix^*$ and $\IR^*$. We see that the objective of \eqref{eq:dual-instability} is equal to the utility difference at $S$, i.e.:
\[\left(\max_{\MSetii \in \mathscr{X}_{S}} \sum_{\agent \in S} \utility_\agent(\MMapii(\agent)) \right) - \left(\sum_{a \in S} \utility_\agent(\MMap(\agent)) + \transfers_\agent)\right).\]
This is no larger than \instabmeasure{} by definition.

We next prove the claim that $\Instability{\utility}{\MSet}{\transfers}$ is at most \eqref{eq:dual-instability}. Let's consider  $S^*$ that maximizes:
\[\max_{S \subseteq \Agents} \left(\max_{\MSetii \in \mathscr{X}_{S}} \sum_{\agent \in S} \utility_\agent(\MMapii(\agent)) \right) - \left(\sum_{a \in S} \utility_\agent(\MMap(\agent)) + \transfers_\agent)\right).\]
Let's take the maximum weight matching of $S^*$. Let $\CMatrix$ be given by the matched agents in this matching and let $\IR$ be given by the unmatched agents in this matching (using the interpretation of \eqref{eq:dual-instability} described above). We see that the objective at \eqref{eq:dual-instability} for $(\CMatrix, \IR)$ is equal to \instabmeasure{} which proves the desired statement.

\end{proof}

\subsection{Proof of \Cref{prop:unhappiness}}\label{app:unhappiness}

We first formally define the \textit{unhappiness of a coalition}, as follows. In particular, the unhappiness with respect to $(\MSet, \transfers)$ of a coalition $\Coalition\subseteq\Agents$ is defined to be:
\begin{alignat}{5}\label{eq:unhappiness}
    \sup_{\substack{\MSetii\in \mathscr{X}_{\Coalition} \\ \transfersii\in\R^{|\Coalition|}}}\ & \mathrlap{\sum_{\agent\in\Coalition} \bigl(\utility_\agent(\MMapii(\agent)) + \transfersii_\agent\bigr) - \sum_{\agent\in\Coalition} \bigl(\utility_\agent(\MMap(\agent)) + \transfers_\agent\bigr)} \\
    \text{s.t.}\ \ & \utility_\agent(\MMapii(\agent)) + \transfersii_\agent \ge \utility_\agent(\MMap(\agent)) + \transfers_\agent \qquad&&\forall \agent\in\Coalition \nonumber \\
    & \transfersii_\agent + \transfersii_{\MMapii(\agent)} = 0\qquad&&\forall\agent\in\Coalition, \nonumber
\end{alignat}
with unhappiness being $0$ if there are no feasible $\MSetii$ and $\transfersii$. In the optimization program, $(\MSetii, \transfersii)$ represents a matching with transfers over $\Coalition$, with the constraint $\transfersii_\agent + \transfersii_{\MMapii(\agent)} = 0$ ensuring that it is zero-sum. The objective measures the difference between {$(\MSet, \transfers)$} and $(\MSetii, \transfersii)$ of the total utility of the agents in $\Coalition$. The constraint $\utility_\agent(\MMapii(\agent)) + \transfersii_\agent \ge \utility_\agent(\MMap(\agent)) + \transfers_\agent$ encodes the requirement that all agents be at least as well off under $(\MSetii, \transfersii)$ as they were under $(\MSet, \transfers)$. This optimization program therefore captures the objective of $\Coalition$ to maximize their total payoff while ensuring that no member of the coalition is worse off than they were according to $(\MSet, \transfers)$.

Recall that, in terms of unhappiness, \Cref{prop:unhappiness} is as follows:
\unhappiness*

\begin{proof}[Proof of Proposition \ref{prop:unhappiness}]
By \Cref{prop:subsidy}, we know that \instabmeasure{} is equal to \eqref{eq:instability}. Moreover, by strong duality, we know that \instabmeasure{} is equal to \eqref{eq:dual-instability} (the dual linear program of \eqref{eq:instability}). Thus, it suffices to prove that  the maximum unhappiness of any coalition is equal to \eqref{eq:dual-instability}.

We first prove the claim that \eqref{eq:dual-instability} is at most the maximum unhappiness of any coalition with respect to $(\MSet, \transfers)$. To do this, it suffices to construct a coalition $\Coalition \subseteq \Agents$ such that \eqref{eq:dual-instability} is at most the unhappiness of $\Coalition$. We construct $\Coalition$ as follows: Recall that there exist optimal solutions $\CMatrix^*$ and $\IR^*$ to \eqref{eq:dual-instability} such that $\CMatrix^*$ corresponds to a (partial) matching on $\Men\times\Women$ and $\IR^*$ corresponds to a subset of $\Agents$. We may take $\Coalition$ to be the union of the agents involved in $\CMatrix^*$ and in $\IR^*$. Now, we upper bound the unhappiness of $\Coalition$ by constructing $\MSetii$ and $\transfersii$ that are feasible for \eqref{eq:unhappiness}. We can take  $\MSetii$ to be the matching that corresponds to the indicator vector $\CMatrix^*$. Because $(\CMatrix^*, \IR^*)$ is optimal for \eqref{eq:dual-instability}, 
\[ \utility_\man(\woman) + \utility_\woman(\man)\ge (\utility_\man(\MMap(\man)) + \transfers_\man) + (\utility_\woman(\MMap(\woman)) + \transfers_\woman) \]
for all $(\man, \woman)\in\MSetii$. Thus, we can find a vector $\transfersii$ of transfers that is feasible for \eqref{eq:unhappiness}.
Then, since $\sum_{\agent\in\Coalition} \transfersii_\agent = 0$, the objective of \eqref{eq:unhappiness} at $(\MSetii, \transfersii)$ is
\[ \sum_{\agent\in\Coalition} \bigl( \utility_\agent(\MMapii(\agent)) - \utility_\agent(\MMap(\agent)) - \transfers_\agent \bigr). \]
This equals to the objective of \eqref{eq:dual-instability} at $(\CMatrix^*, \IR^*)$, which equals \eqref{eq:dual-instability}, as desired. 

We now show the inequality in the other direction, that \eqref{eq:dual-instability} is at least the maximum unhappiness of any coalition with respect to $(\MSet, \transfers)$. It suffices to construct a feasible solution $(\CMatrix, \IR)$ to \eqref{eq:dual-instability} that achieves at least the maximum unhappiness of any coalition. Let  $\Coalition$ be a coalition with maximum unhappiness, and let $(\MSetii, \transfersii)$ be an optimal solution for  \eqref{eq:unhappiness}. Moreover, let $\CMatrix$ be the indicator vector corresponding to agents who are matched in $\MSetii$ and $\IR$ be the indicator vector corresponding to agents in $\Coalition$ who are unmatched. The objective of \eqref{eq:unhappiness} at $(\MSetii, \transfersii)$ is
\[ \sum_{\agent\in\Coalition} \bigl( \utility_\agent(\MMapii(\agent)) - \utility_\agent(\MMap(\agent)) - \transfers_\agent \bigr), \]
which equals the objective of \eqref{eq:dual-instability} at the $(\CMatrix, \IR)$ that we constructed.

\end{proof}

\subsection{Proof of \Cref{prop:desiderata}}\label{appendix:properties}

\desiderata*

\begin{proof}[Proof of \Cref{prop:desiderata}]
We first prove the third part of the Proposition statement, then the first part of the Proposition statement, and finally the second part.

\item  \paragraph{Proof of part (c).} Because $\sum_{\agent\in\Agents} \transfers_\agent = 0$, \instabmeasure{} satisfies the following:
\begin{align*}
 \Instability{\utility} {\MSet}{\transfers}
 &\ge \left(\max_{\MSetii \in \mathscr{X}_\Agents} \sum_{\agent \in \Agents} \utility_\agent(\MMapii(\agent)) \right) - \left(\sum_{\agent \in \Agents} \utility_\agent(\MMap(\agent)) + \transfers_\agent\right)\\
 &=    \left(\max_{\MSetii \in \mathscr{X}_\Agents} \sum_{\agent \in \Agents} \utility_\agent(\MMapii(\agent)) \right) - \left(\sum_{\agent \in \Agents} \utility_\agent(\MMap(\agent))\right).
\end{align*}
The second line is exactly the utility difference.

\item \paragraph{Proof of part (a).} From above, we have that \instabmeasure{} is lower bounded by the utility difference, which is always nonnegative. Hence \instabmeasure{} is also always nonnegative.

To see that \instabmeasure{} is $0$ if and only if $(\MSet, \transfers)$ is stable, first suppose $(\MSet, \transfers)$ is unstable. Then, there exists a blocking pair $(\man, \woman)$, in which case 
\[ \Instability{\utility}{\MSet}{\transfers} \ge \utility_\man(\woman) + \utility_\woman(\man) - (\utility_{\man}(\MMap(\man)) + \utility_{\woman}(\MMap(\woman)) + \transfers_\man + \transfers_\woman) > 0 \]
by the definition of blocking. Now, suppose $\Instability{\utility}{\MSet}{\transfers} > 0$. Then, there exists a subset $S\subseteq\Agents$ such that
\[ \left( \max_{\MSetii \in \mathscr{X}_S} \sum_{\agent \in S} \utility_\agent(\MMapii(\agent)) \right)  - \left(\sum_{\agent \in S} \utility_\agent(\MMap(\agent)) + \transfers_\agent\right) > 0. \]
Let $\MSetii$ be a maximum weight matching on $S$. We can rewrite the above as
\[ \sum_{(\man,\woman)\in\MSetii} \bigl( \utility_\man(\woman) + \utility_\woman(\man) - (\utility_\man(\MMap(\man)) + \utility_\woman(\MMap(\woman)) + \transfers_\man + \transfers_\woman \bigr)  > 0. \]
Some term in the sum on the left-hand side must be positive, so there exists a blocking pair $(\man, \woman)\in\MSetii$. In particular, $(\MSet, \transfers)$ is not stable.

\item \paragraph{Proof of part (b).} We prove that 
\[|\Instability{\utility} {\MSet}{\transfers} - \Instability{\tilde{\utility}}{\MSet}{\transfers}| \le 2\sum_{\agent\in\Agents} \norm{\utility_\agent - \utilityii_\agent}_\infty. \]
The supremum of $L$-Lipschitz functions is $L$-Lipschitz, so it suffices to show that 
\[\left(\max_{\MSetii \in \mathscr{X}_S} \sum_{\agent \in S} \utility_\agent(\MMapii(\agent)) \right) - \sum_{\agent \in S} (\utility_\agent(\MMap(\agent)) + \transfers_\agent)\] satisfies the desired Lipschitz condition for any $S \subseteq \Agents$. In particular, it suffices to show that
\begin{equation}
\label{eq:term1}
  \left| \sum_{\agent \in S} (\utility_\agent(\MMap(\agent)) + \transfers_\agent) - \sum_{\agent \in S}(\utilityii_\agent(\MMap(\agent)) + \transfers_\agent) \right| \le \sum_{\agent\in\Agents} \norm{\utility_\agent - \utilityii_\agent}_\infty  
\end{equation}
and
\begin{equation}
\label{eq:term2}
  \left| \left(\max_{\MSetii \in \mathscr{X}_S} \sum_{\agent \in S} \utility_\agent(\MMapii(\agent))\right) - \left(\max_{\MSetii \in \mathscr{X}_S} \sum_{\agent \in S} \utilityii_\agent(\MMapii(\agent))\right) \right| \le \sum_{\agent\in\Agents} \norm{\utility_\agent - \utilityii_\agent}_\infty.
\end{equation}
For \eqref{eq:term1}, we have
\[ \left| \sum_{\agent \in S} (\utility_\agent(\MMap(\agent)) + \transfers_\agent) - \sum_{\agent \in S} (\utilityii_\agent(\MMap(\agent)) + \transfers_\agent) \right| = \left| \sum_{\agent \in S} \bigl(\utility_\agent(\MMap(\agent)) - \utilityii_\agent(\MMap(\agent))\bigr) \right|   \le \sum_{\agent\in\Agents} \norm{\utility_\agent - \utilityii_\agent}_\infty.  \]
For \eqref{eq:term2}, this boils down to showing that the total utility of the maximum weight matching is Lipschitz. Using again the fact that the supremum of Lipschitz functions is Lipschitz, this follows from the total utility of any fixed matching being Lipschitz. 
\end{proof}

\section{Proofs for Section \ref{sec:regret}}\label{appendix:proofsalg}

\subsection{Proof of \Cref{thm:instanceind}}\label{appendix:proofunstructured}

\instanceind*
\begin{proof}[Proof of Theorem \ref{thm:instanceind}]
The starting point for our proof of Theorem \ref{thm:instanceind} is the typical approach in multi-armed bandits and combinatorial bandits \cite{GKJ12,DBLP-conf/icml/ChenWY13, LS2020} of bounding regret in terms of the sizes of the confidence interval of the chosen arms. However, rather than using the sizes of confidence intervals to bound the utility difference (as in the incentive-free maximum weight matching setting), we bound \instabmeasure{} through \Cref{lemma:confset}. From here on, our approach composes cleanly with existing bandits analyses; in particular, we can follow the typical combinatorial bandits approach \cite{GKJ12,DBLP-conf/icml/ChenWY13} to get the desired upper bound. 

For completeness, we present the full proof. We divide into two cases, based on the event $\Event$ that all of the confidence sets contain their respective true utilities at every time step $\round\le\Rounds$. That is, $\utility_{\man}(\woman)\in\Confidence_{\man,\woman}$ and $\utility_\woman(\man)\in\Confidence_{\woman,\man}$ for all $(\man,\woman)\in\Men\times\Women$ at all $\round$. 

\item \paragraph{Case 1: $\Event$ holds.} By \Cref{lemma:confset}, we may bound
\[ \Instability[^\round]{\utility}{\MSet^\round}{\transfers^\round} \le \sum_{\agent \in \Agents^t} \Bigl(\max\bigl(\Confidence_{\agent, \MMapt(\agent)}\bigr) - \min\bigl(\Confidence_{\agent, \MMapt(\agent)}\bigr)\Bigr) = O\paren*{\sum_{(\man,\woman)\in\MSet^\round} \sqrt{\frac{\log(|\Agents|\Rounds)}{\pulls_{\man\woman}^{\round}}}},\]
where $\pulls_{\man\woman}^{\round}$ is the number of times that the pair $(\man,\woman)$ has been matched at the start of round $\round$. Let $w^\round_{\man, \woman} = \frac{1}{\sqrt{\pulls_{\man\woman}^{\round}}}$ be the size of the confidence set (with the log factor scaled out) for $(\man, \woman)$ at the start of round $\round$. 

At each time step $\round$, let's consider the list consisting of $w^\round_{\man_\round, \woman_\round}$ for all $(\man_\round, \woman_\round) \in \MSet^\round$. Let's now consider the overall list consisting of the concatenation of all of these lists over all rounds. Let's order this list in decreasing order to obtain a list $\tilde{w}_1, \ldots, \tilde{w}_L$ where $L = \sum_{\round=1}^{\Rounds} |\MSet^\round| \le n \Rounds$. In this notation, we observe that:
\[\sum_{\round=1}^{\Rounds} \Instability{\utility}{\MSet^\round}{\transfers^\round} \le \sum_{\round=1}^{\Rounds} \sum_{\agent \in \Agents^t} \Bigl(\max\bigl(\Confidence_{\agent, \MMapt(\agent)}\bigr) - \min\bigl(\Confidence_{\agent, \MMapt(\agent)}\bigr)\Bigr) = \log(|\Agents| \Rounds) \sum_{l=1}^L \tilde{w}_l.   \]
We claim that $\tilde{w}_l \le O\paren*{\min(1, \frac{1}{\sqrt{(l / |\Agents|^2) -1}})}$. The number of rounds that a pair of agents can have their confidence set have size at least $\tilde{w}_l$ is upper bounded by $1 + \frac{1}{\tilde{w}_l^2}$. Thus, the total number of times that any confidence set can have size at least $\tilde{w}_l$ is upper bounded by $(|\Agents|^2)(1 + \frac{1}{\tilde{w}_l^2})$. 

Putting this together, we see that:
\begin{align*}
  \log(|\Agents| \Rounds) \sum_{l=1}^L \tilde{w}_l &\le O\paren*{\sum_{l=1}^{L} \min(1, \frac{1}{\sqrt{(l / |\Agents|^2) -1}})} \\
  &\le O\paren*{\log(|\Agents| \Rounds) \sum_{l=1}^{n T} \min(1, \frac{1}{\sqrt{(l / |\Agents|^2) -1}})} \\
  &\le O\paren*{ |\Agents| \sqrt{nT} \log(|\Agents| \Rounds)}. 
\end{align*}

\item \paragraph{Case 2: $\Event$ does not hold.} Since each $\pulls_{\man\woman}(\hat\utility_\man(\woman) -\utility_\man(\woman))$ is mean-zero and $1$-subgaussian, and we have $O(|\Men||\Women|\Rounds)$ such random variables over $\Rounds$ rounds, the probability that any of them exceeds 
\[ 2\sqrt{\log(|\Men||\Women|\Rounds/\delta)}\le 2\sqrt{\log(|\Agents|^2\Rounds/\delta)} \]
is at most $\delta$ by a standard tail bound for the maximum of subgaussian random variables. It follows that $\Event$ fails to hold with probability at most $|\Agents|^{-2}\Rounds^{-2}$. In the case that $\Event$ fails to hold, our regret in any given round would be at most $4|\Agents|$ by the Lipschitz property in \Cref{prop:desiderata}. (Recall that our upper confidence bound for any utility is wrong by at most $2$ due to clipping each confidence interval to lie in $[-1, 1]$.) Thus, the expected regret from this scenario is at most
\[ |\Agents|^{-2}\Rounds^{-2}\cdot 4|\Agents|\Rounds\le 4|\Agents|^{-1}\Rounds^{-1}, \]
which is negligible compared to the regret bound from when $\Event$ does occur.
\end{proof}

\subsection{Proof of \Cref{thm:typed}}\label{appendix:prooftyped}

\typed*

\begin{proof}[Proof of \Cref{thm:typed}]

Like in the proof of \Cref{thm:instanceind}, we divide into two cases, based on the event $\Event$ that all of the confidence sets contain their respective true utilities at every time step $\round\le\Rounds$. That is, $\utility_{\agent}(\agent')\in\Confidence_{\agent, \agent'}$ for all pairs of agents at all $\round$. 

\item \paragraph{Case 1: $\Event$ holds.} By \Cref{lemma:confset}, we may bound
\[ \Instability[^\round]{\utility}{\MSet^\round}{\transfers^\round} \le \sum_{\agent \in \Agents^t } \Bigl(\max\bigl(\Confidence_{\context_a, \context_{\MMapt(\agent))}}\bigr) - \min\bigl(\Confidence_{\context_a, \context_{\MMapt(\agent))}}\bigr)\Bigr) = O\paren*{\sum_{(\man, \woman) \in\MSet^\round} \sqrt{\frac{\log(|\Agents|\Rounds)}{\pulls_{\context_\man \context_\woman}^{\round}}}}, \]
where $\pulls_{\context_1\context_2}^{\round}$ is the number of times that the an agent of type $\context_1$ has been matched with an agent of context $\context_2$ at the start of round $\round$. (We define $\pulls_{\context_1, \context_2}^0 = 0$ by default.) Let $w^\round_{\context_1, \context_2} =  \frac{1}{\sqrt{\pulls^t_{\context_1, \context_2}}}$ be the size of the confidence set (with the log factor scaled out) for $(\context_1, \context_2)$ at the start of round $\round$.

At each time step $\round$, let's consider the list consisting of $w^\round_{\context_{\man_\round}, \context_{\woman_\round}}$ for all $(\man_\round, \woman_\round) \in \MSet^\round$. Let's now consider the overall list consisting of the concatenation of all of these lists over all rounds. Let's order this list in decreasing order to obtain a list $\tilde{w}_1, \ldots, \tilde{w}_L$ where $L = \sum_{\round=1}^{\Rounds} |\MSet^\round| \le n \Rounds$. In this notation, we observe that:
\[\sum_{\round=1}^{\Rounds} \Instability[^\round]{\utility}{\MSet^\round}{\transfers^\round} \le \sum_{\round=1}^{\Rounds} \sum_{\agent \in \Agents^t} \Bigl(\max\bigl(\Confidence_{\context_\agent, \context_{\MMapt(\agent)}}\bigr) - \min\bigl(\Confidence_{\context_\agent, \context_{\MMapt(\agent)}}\bigr)\Bigr) = \log(|\Agents| \Rounds) \sum_{l=1}^L \tilde{w}_l.   \]
We claim that $\tilde{w}_l \le O\paren*{\min(1, \frac{1}{\sqrt{(l / |\Contexts|^2) -1}})}$. The number of instances that a pair of contexts can have their confidence set have size at least $\tilde{w}_l$ is upper bounded by $2n + \frac{1}{\tilde{w}_l^2}$. Thus, the total number of times that any confidence set can have size at least $\tilde{w}_l$ is upper bounded by $(|\Contexts|)(2n + \frac{1}{\tilde{w}_l^2})$.

Putting this together, we see that:
\begin{align*}
  \log(|\Agents| \Rounds) \sum_{l=1}^L \tilde{w}_l &\le O\paren*{\sum_{l=1}^{L} \min(1, \frac{1}{\sqrt{(l / |\Agents|^2) -1}})} \\
  &\le O\paren*{\log(|\Agents| \Rounds) \sum_{l=1}^{n T} \min(1, \frac{1}{\sqrt{(l / |\Contexts|^2) -1}})} \\
  &\le O\paren*{ |\Contexts| \sqrt{nT} \log(|\Contexts|^2 \Rounds)}. 
\end{align*}

\item \paragraph{Case 2: $\Event$ does not hold.} Since each $\pulls_{\man\woman}(\hat\utility_\man(\woman) -\utility_\man(\woman))$ is mean-zero and $1$-subgaussian, and we have $O(|\Men||\Women|\Rounds)$ such random variables over $\Rounds$ rounds, the probability that any of them exceeds 
\[ 2\sqrt{\log(|\Men||\Women|\Rounds/\delta)}\le 2\sqrt{\log(|\Agents|^2\Rounds/\delta)} \]
is at most $\delta$ by a standard tail bound for the maximum of subgaussian random variables. It follows that $\Event$ fails to hold with probability at most $|\Agents|^{-2}\Rounds^{-2}$. In the case that $\Event$ fails to hold, our regret in any given round would be at most $4|\Agents|$ by the Lipschitz property in \Cref{prop:desiderata}. (Recall that our upper confidence bound for any utility is wrong by at most two due to clipping each confidence interval to lie in $[-1, 1]$.) Thus, the expected regret from this scenario is at most
\[ |\Agents|^{-2}\Rounds^{-2}\cdot 4|\Agents|\Rounds\le 4|\Agents|^{-1}\Rounds^{-1}, \]
which is negligible compared to the regret bound from when $\Event$ does occur.
\end{proof}

\subsection{Proof of \Cref{thm:linear}}\label{appendix:prooflinear}

\linear*

To prove \Cref{thm:linear}, it suffices to (a) show that the confidence sets contain the true utilities with high probability, and (b) bound the sum of the sizes of the confidence sets.

Part (a) follows from fact established in  existing analysis of LinUCB in the classical linear contextual bandits setting \cite{RV13}. 
\begin{lemma}[{\cite[Proposition~2]{RV13}}]
\label{lemma:confsetslin}
Let the confidence sets be defined as above (and in \textsc{MatchLinUCB}). For each $\agent \in \Agents$, it holds that:
\[\mathbb{P}[\phi(\agent) \in C_{\phi(a)}  \quad  \forall 1 \le \round \le \Rounds] \ge 1 - 1/(|\Agents|^3 T^2). \]
\end{lemma}

\begin{lemma}
\label{lemma:intermediatelin}
Let the confidence sets be defined as above (and in \textsc{MatchLinUCB}). For each $\agent \in \Agents$ and for any $\epsilon > 0$, it holds that:
\[\sum_{\round \mid \agent \in \Agents^t, \mu_{\MSet^\round}(\agent) \neq \agent}\mathbf{1}\left[\max\bigl(\Confidence_{\agent, \mu_{\MSet^\round} (\agent)}\bigr) - \min\bigl(\Confidence_{\agent, \mu_{\MSet^\round} (\agent)})\bigr) > \epsilon\right] \le O\left( \left(\frac{4 \beta_T}{\epsilon^2} + 1\right) d \log(1/\epsilon) \right).\] 
\end{lemma}
\begin{proof}
We follow the same argument as the proof of Proposition 3 in \cite{RV13}. 

We first recall the definition of $\epsilon$-dependence and $\epsilon$-eluder dimension: We say that an agent $\agent'$ is \emph{$\epsilon$-dependent} on $\agent_1',\ldots,\agent_s'$ if for all $\phi(\agent), \tilde\phi(\agent)\in\Ball^d$ such that
\[ \sum_{k=1}^s \langle\context_{\agent_k'}, \tilde\phi(a) - \phi(a)\rangle^2 \le\epsilon^2, \]
we also have $\langle\context_{\agent'}, \tilde\phi(a) - \phi(a)\rangle^2 \le\epsilon^2$. The $\epsilon$-eluder dimension $\eluderdim$ of $\Ball^d$ is the maximum length of a sequence $\agent_1', \ldots, \agent_s'$ such that no element is $\epsilon$-dependent on a prefix.

Consider the subset $S_\agent$ of $\{\round \mid \agent \in \Agents^t, \mu_{\MSet^\round}(\agent) \neq \agent\}$ such that 
\[ \mathbf{1}\left[\max\bigl(\Confidence_{\agent, \mu_{\MSet^\round} (\agent)}\bigr) - \min\bigl(\Confidence_{\agent, \mu_{\MSet^\round} (\agent)})\bigr) > \epsilon\right]. \]
Suppose for the sake of contradiction that
\[ |S_\agent| > \left(\frac{4 \beta_T}{\epsilon^2} + 1\right) \eluderdim. \]
Then, there exists an element $\round^*$ that is $\epsilon$-dependent on $\frac{4 \beta_T}{\epsilon^2} + 1$ disjoint subsets of $S_\agent$: One can repeatedly remove sequences $\agent_{\mu_{\MSet^{\round_1}}(\agent)}', \ldots, \agent_{\mu_{\MSet^{\round_s}}(\agent)}'$ of maximal length such that no element is $\epsilon$-dependent on a prefix; note that $s\le\eluderdim$ always. Let the subsets be $S_\agent^{(q)}$ for $q = 1,\ldots,\frac{4 \beta_T}{\epsilon^2} + 1$, and let $\phi(\agent), \tilde\phi(\agent)$ be such that $\langle \context_{\mu_{\MSet^{\round^*}}(\agent)}, \tilde\phi(\agent) - \phi(\agent)\rangle > \epsilon$. The above implies that
\[ \sum_{q=1}^{\frac{4 \beta_\Rounds}{\epsilon^2} + 1}\sum_{\round\in S_\agent^{(q)}} \langle\context{\MMapt(\agent)}, \tilde\phi(\agent) - \phi(\agent)\rangle^2 > 4\beta_\Rounds \]
by the definition of $\epsilon$-dependence. But this is impossible, since the left-hand side is upper bounded by
\[ \sum_{\round=1}^\Rounds \langle\context{\MMapt(\agent)}, \tilde\phi(\agent) - \phi(\agent)\rangle^2\le 4\beta_\Rounds \]
by the definition of the confidence sets. Hence it must hold that
\[ |S_\agent| \le \left(\frac{4 \beta_T}{\epsilon^2} + 1\right) \eluderdim. \]
Now, it follows from the bound on the eluder dimension for linear bandits (Proposition 6 in \cite{RV13}) that the bound of $\tilde{O} \left( \left(\frac{4 \beta_T}{\epsilon^2} + 1\right) d \log(1/\epsilon). \right)$ holds.
\end{proof}

\begin{lemma}
\label{lemma:sumsizes}
Let the confidence sets be defined as above (and in \textsc{MatchLinUCB}). For any $\agent \in \Agents$, it holds that: 
\[\sum_{\round \mid \agent \in \Agents^t, \mu_{\MSet^\round}(\agent) \neq \agent}\Bigl(\max\bigl(\Confidence_{\agent, \mu_{\MSet^\round} (\agent)}\bigr) - \min\bigl(\Confidence_{\agent, \mu_{\MSet^\round} (\agent)})\bigr)\Bigr) \le O(d (\log (T |\Agents|)) \sqrt{T_\agent}) ,\]
where $T_\agent$ is the number of times that agents is matched. 
\end{lemma}
\begin{proof}
Let's consider the set of confidence set sizes $\Bigl(\max\bigl(\Confidence_{\agent, \mu_{\MSet^\round} (\agent)}\bigr) - \min\bigl(\Confidence_{\agent, \mu_{\MSet^\round} (\agent)})\bigr)\Bigr)$ for $\round$ such that $\agent \in \Agents^t, \mu_{\MSet^\round}$. Let's sort these confidence set sizes in decreasing order and label them $w_1 \ge \ldots \ge w_{T_a}$. Restating \Cref{lemma:intermediatelin}, we see that
\begin{equation}
    \label{eq:intermediatelin}
    \sum_{t=1}^{T_a} w_t \mathbf{1}[w_t > \epsilon] \le O\left( \left(\frac{4 \beta_T}{\epsilon^2} + 1\right) d \log(1/\epsilon) \right).
\end{equation}
for all $\epsilon >0$. 

We see that:
\begin{align*}
  \sum_{\round \mid \agent \in \Agents^t, \mu_{\MSet^\round}(\agent) \neq \agent}\Bigl(\max\bigl(\Confidence_{\agent, \mu_{\MSet^\round} (\agent)}\bigr) - \min\bigl(\Confidence_{\agent, \mu_{\MSet^\round} (\agent)})\bigr)\Bigr) &= \sum_{t=1}^{T_a} w_t \\
  &\le \sum_{t=1}^{T_a} w_t \mathbf{1}[w_t > 1/T_a^2] + \sum_{t=1}^{T_a} w_t \mathbf{1}[w_t \le 1/T_a^2] \\
  &\le \frac{1}{T_a} + \sum_{t=1}^{T_a} w_t \mathbf{1}[w_t > 1/T_a^2].
\end{align*}

We claim that $w_i \le 2$ if $i \ge d \log(T_a)$ and $w_i \le \min(2, \frac{4 \beta_T (d \log T_a)}{i - d \log T_a})$ if $i > d \log T_a$. The first part follows from the fact that we truncate the confidence sets to be within $[-1, 1]$. It thus suffices to show that $w_i \le \frac{4 \beta_T (d \log T_a)}{i - d \log T_a}$ for $t \le d \log T$. If $w_i \ge \epsilon > 1/T_a^2$, then we see that $\sum_{t=1}^{T_a} \mathbf{1}[w_t > \epsilon] \ge i$, which means by \eqref{eq:intermediatelin} that $i \le O\left( \left(\frac{4 \beta_T}{\epsilon^2} + 1\right) d \log(1/\epsilon) \right) \le O\left( \left(\frac{4 \beta_T}{\epsilon^2} + 1\right) d \log(T_a) \right)$ which means that $\epsilon \le \frac{4 \beta_T (d \log T_a)}{i - d \log T_a}$. This proves the desired statement. 

Now, we can plug this into the above expression to obtain:
\begin{align*}
    \lefteqn{\sum_{\round \mid \agent \in \Agents^t, \mu_{\MSet^\round}(\agent) \neq \agent}\Bigl(\max\bigl(\Confidence_{\agent, \mu_{\MSet^\round} (\agent)}\bigr) - \min\bigl(\Confidence_{\agent, \mu_{\MSet^\round} (\agent)})\bigr)\Bigr)}\\
    &\le \frac{1}{T_a} + \sum_{t=1}^{T_a} w_t \mathbf{1}[w_t > 1/T_a^2] \\
    &\le \frac{1}{T_a} + 2 d \log(T_a) +  \sum_{i > d \log T_a}^{T_a} \min\left(2, \frac{4 \beta_T (d \log T_a)}{i - d \log T_a}\right) \\
    &\le \frac{1}{T_a} + 2 d \log(T_a) + 2 \sqrt{d \log T_a \beta_T} \int_{t=0}^{T_a} t^{-1/2} dt \\
    &= \frac{1}{T_a} + 2 d \log(T_a) + 4 \sqrt{d T_a \log T_a \beta_T}.
\end{align*}
We now use the fact that:
\[\beta_T = O(d \log T + \frac{1}{T} \sqrt{\log(T^2 |A|)}).   \] Plugging this into the above expression, we obtain the desired result.  
\end{proof}

We are now ready to prove Theorem \ref{thm:linear}.
\begin{proof}[Proof of Theorem \ref{thm:linear}]
Like in the proof of \Cref{thm:instanceind}, we divide into two cases, based on the event $\Event$ that all of the confidence sets contain their respective true utilities at every time step $\round\le\Rounds$. That is, $\utility_{\context_1}(\context_2)\in\Confidence_{\context_1, \context_2}$ for all $\context_1, \context_2 \in \Contexts$ at all $\round$. 

\item \paragraph{Case 1: $\Event$ holds.} 
By \Cref{lemma:confset}, we know that the cumulative regret is upper bounded by 
\begin{align*}
R_T &\le \sum_{\round =1}^{\Rounds} \sum_{\agent \in \Agents^t} \Bigl(\max\bigl(\Confidence_{\agent, \mu_{\MSet^\round}(\agent)}\bigr) - \min\bigl(\Confidence_{\agent, \mu_{\MSet^\round}(\agent)}\bigr)\Bigr) \\
&= \sum_{\agent \in \Agents} \sum_{\round \mid \agent \in \Agents^t, \mu_{\MSet^\round}(\agent) \neq \agent}\Bigl(\max\bigl(\Confidence_{\agent, \mu_{\MSet^\round} (\agent)}\bigr) - \min\bigl(\Confidence_{\agent, \mu_{\MSet^\round} (\agent)})\bigr)\Bigr) \\
&\le \sum_{\agent \in \Agents} O(d \log (T |\Agents|) \sqrt{T_\agent}),
\end{align*}
where the last inequality applies \Cref{lemma:sumsizes} to the inner summand. We see that $\sum_{\agent \in \Agents} T_\agent = \sum_t |\Agents_t| \le n T$ by definition, since at most $n$ agents show up at every round. Let's now observe that:
\[\sum_{\agent \in \Agents} \sqrt{T_\agent} \le \sqrt{|\Agents|} \sqrt{\sum_{\agent \in \Agents} T_\agent} \le \sqrt{|\Agents| n T},\]
as desired.

\item \paragraph{Case 2: $\Event$ does not hold.} 
From \Cref{lemma:confsetslin}, it follows that:
\[\mathbb{P}[\phi(\agent) \in C_{\phi(a)}  \quad  \forall 1 \le \round \le \Rounds] \ge 1 - 1/(|\Agents|^3 T^2).\]
Union bounding, we see that
\[\mathbb{P}[\phi(\agent) \in C_{\phi(a)}  \quad  \forall 1 \le \round \le \Rounds \forall \agent \in \Agents] \ge 1 - 1/(|\Agents|^2 T^2).\]
By the definition of the confidence sets for the utilities, we see that: 
\begin{equation}
\mathbb{P}[\utility(\agent, \agent') \in C_{a, a'} \quad \forall 1 \le \round \le \Rounds, \forall \agent, \agent' \in \Agents]   \ge 1/(|\Agents|^2 T^2).
\end{equation}
Thus, the probability that event $\Event$ does not hold is at most $|\Agents|^{-2} T^{-2}$. In the case that $\Event$ fails to hold, our regret in any given round would be at most $4 |\Agents|$ by the Lipschitz property in \Cref{prop:desiderata}. Thus, the expected regret is at most $4 |\Agents|^{-1} T^{-1}$ which is negligible compared to the regret bound from when $\Event$ does occur.

\end{proof}

\subsection{Proof of \Cref{lemma:lowerbound}}

\lowerbound*

\begin{proof}[Proof of \Cref{lemma:lowerbound}]
Recall that, by \Cref{prop:desiderata}, the problem of learning a maximum weight matching with respect to utility difference is no harder than that of learning a stable matching with respect to \instabmeasure{}. In the remainder of our proof, we reduce a standard ``hard instance'' for stochastic multi-armed bandits to our setting of learning a maximum weight matching.

\item \paragraph{Step 1: Constructing the hard instance for stochastic MAB.}
Consider the following family of stochastic multi-armed bandits instances: for a fixed $K$, let $\Hard_\arm$ for $\arm\in\{1,\ldots,K\}$ denote the stochastic multi-armed bandits problem where all arms have 0-1 rewards, and the $k$-th arm has mean reward $\frac 12 + \bonus$ if $k =\arm$ and $\frac 12$ otherwise, where $\bonus > 0$ will be set later. A classical lower bound for stochastic multi-armed bandits is the following:

\item 
\begin{lemma}[\cite{DBLP:journals/siamcomp/AuerCFS02}]
\label{lemma:classicallb}
The expected regret of any stochastic multi-armed bandit algorithm on an instance $\Hard_\arm$ for $\arm$ selected uniformly at random from $\{1,\ldots,K\}$ is $\Omega(\sqrt{K\Rounds})$.
\end{lemma}

\item \paragraph{Step 2: Constructing a (random) instance for the maximum weight matching problem.} We will reduce solving the above distribution over stochastic multi-armed bandits problems to a distribution over instances of learning a maximum weight matching. Let us now construct this random instance of the maximum weight matching problem. Let $|\Men| = K$ and $|\Women| = 10 K\log(K\Rounds)$. Specifically, we sample inputs for learning a maximum weight matching as follows: For each man $\man\in\Men$, select $\arm_\man\in\{1,\ldots,K\}$ uniformly at random, and define $\utility_\man(\woman)$ to be $\frac 12 + \bonus$ if $\lfloor (\woman - 1) / \log K\rfloor = \arm_\man$ and \smash{$\frac 12$} otherwise. Furthermore, let $\utility_\woman(\man) = 0$ for all $(\man,\woman)\in\Men\times\Women$. Finally, suppose observations are always in $\{0,1\}$ (but are unbiased).

The key property of the above setup that we will exploit for our reduction is the fact that, due to the imbalance in the market, the maximum weight matching for these utilities has with high probability each $\man$ matched with some $\woman$ whom they value at $\frac 12 + \bonus$. Indeed, by a union bound, the probability that more than $10\log(K\Rounds)$ different $\man$ have the same $\alpha_\man$ is at most
\[ K\cdot\binom{K}{10\log(K\Rounds)} K^{-10\log(K\Rounds)} = O\paren*{K^{-4}\Rounds^{-4}}. \]
Thus, with probability $1 - O(K^{-4}\Rounds^{-4})$, this event holds. The case where this event does not hold contributes negligibly to regret, so we do not consider it further.

\item \paragraph{Step 3: Establishing the reduction.}
Now, suppose for the sake of contradiction that some algorithm could solve our random instance of learning a maximum weight matching problem with expected regret \smash{$o(K^{3/2}\sqrt{\Rounds})$}. We can obtain a stochastic multi-armed bandits that solves the instances in \Cref{lemma:classicallb} as follows: Choose a random $\man^*\in\Men$ and set $\alpha_{\man^*} = \alpha$. Simulate the remaining $\man$ by choosing $\alpha_{\man}$ for all $\man\neq\man^*$ uniformly at random. Run the algorithm on this instance of learning a maximum weight matching, ``forwarding'' arm pulls to the true instance when matching $\man^*$.

To analyze the regret of this algorithm when faced with the distribution from \Cref{lemma:classicallb}, we first note that with high probability, all the agents $\man\in\Men$ can simultaneously be matched to a set of $\woman\in\Women$ such that each $\man$ is matched to some $\woman$ whom they value at $\frac 12 + \bonus$. Then, the regret of any matching is $\bonus$ times the number of $\man\in\Men$ who are not matched to a $\woman$ whom they value at \smash{$\frac 12 + \bonus$}. Thus, we can define the cumulative regret for an agent $\man\in\Men$ as $\bonus$ times the number of rounds they were not matched to someone whom they value at $\frac 12 + \bonus$. For $\man^*$, this regret is just the regret for the distribution from \Cref{lemma:classicallb}. Since $\man^*$ was chosen uniformly at random, their expected cumulative regret is at most 
\[ \frac 1K\cdot o(K^{3/2}\sqrt\Rounds) = o(\sqrt{K\Rounds}), \]
in violation of \Cref{lemma:classicallb}.

\item \paragraph{Step 4: Concluding the lower bound.} 
This contradiction implies that no algorithm can hope to obtain $o(K^{3/2}\sqrt{\Rounds})$ expected regret on this distribution over  instances of learning a maximum weight matching. Since there are \smash{$O(K\log(K\Rounds)) = \widetilde O(K)$} agents in the market total,  the desired lower bound follows.
\end{proof}

\section{Proof of Theorem \ref{thm:instancespecific}}\label{appendix:proofinstancespecific}

\instancespecific*

\subsection{\textsc{MatchUCB$'$}}

\textsc{MatchUCB$'$} is the same as \textsc{MatchUCB}, except we call \textsc{ComputeMatch$'$} instead of \textsc{ComputeMatch}. The idea behind \textsc{ComputeMatch$'$} is that we compute an optimal primal-dual solution for both the original confidence sets $\Confidence$ as well as expanded confidence sets $\Confidence'$, which we define to be twice the width of the original confidence sets. More formally, we define
\[\Confidence'_{\agent, \agent'} \coloneqq \biggl[\min(\Confidence_{\agent, \agent'}) - \frac{\max(\Confidence_{\agent, \agent'}) - \min(\Confidence_{\agent, \agent'})}{2}, \max(\Confidence_{\agent, \agent'}) + \frac{\max(\Confidence_{\agent, \agent'}) - \min(\Confidence_{\agent, \agent'})}{2}\biggr]. \]
We will adaptively explore (following UCB) according to both $\Confidence$ and $\Confidence'$. Doing extra exploration according to the more pessimistic confidence sets $\Confidence'$ is necessary for us to be able to find ``robust'' dual solutions for setting transfers.

We define $(\MSet^{*}, \prices^{*})$, which will be an optimal primal-dual solution for the upper confidence bounds of ${C}$ as follows. Let $\MSet^*$ be a maximum weight matching with respect to $\uUCB$. We next compute the gap
\[ \Gap^{\text{UCB}} = \min_{\MSet \neq \MSet^{*}} \Biggl\{ \sum_{\agent\in\Agents} \uUCB_\agent(\mu_{\MSet^{*}}(\agent)) - \sum_{\agent\in\Agents} \uUCB_\agent(\MMap(\agent)) \Biggr\} \] 
with respect to $\uUCB$. We can compute this gap by computing the maximum weight matching and the second-best matching with respect to $\uUCB$.\footnote{See \citet{CH87} for efficient algorithms for to compute the second-best matching.} Next, define utility functions $\utility'_\agent$ such that
\[ 
\utility'_\agent(\agent') = 
\begin{cases}
\uUCB_\agent(\agent') - \frac{\Gap^{\text{UCB}}}{|\Agents|} & \text{if $\mu_{\MSet^{*}}(\agent) = \agent'$ and $\agent\neq\agent'$} \\
\uUCB_\agent(\agent') & \text{otherwise}
\end{cases}
\]
for all $\agent\in\Agents$. (We show in \Cref{lemma:maximummatch} that $\MSetopt$ is still a maximum weight matching for $\utility'$.) Now, compute an optimal dual solution $\prices'$ for utility function $\utility'$. To get $\pricesopt$, we add $\Delta^\text{UCB}/|\Agents|$ to $\prices'_\agent$ for each matched agent $\agent$ in $\MSetopt$. (See \Cref{lemma:optimality} for a proof that $(\MSet^*, \prices^*)$ is an optimal primal-dual pair with respect to $\uUCB$.) 

Finally, let $(\MSet^{*,2}, \prices^{*,2})$ be any optimal primal-dual pair for the utility function $\utility^{\UCB,2}$ given by the upper confidence bounds \smash{$\max(\Confidence'_{\agent,\agent'})$} of ${C}'$.

With this setup, we define \textsc{ComputeMatch$'$} as follows: If $\MSet^{*} \neq \MSet^{*,2}$, return $(\MSet^{*,2}, \transfers^{*,2})$, where $\transfers^{*,2}$ is given by $\transfers_\agent^{*,2} = \prices^{*,2}_\agent - \utility^{\UCB,2}_\agent(\mu_{\MSet^{*,2}}(\agent))$ if $\agent$ is matched and $\transfers^{*,2}_\agent = 0$ if $\agent$ is unmatched. Otherwise, return $(\MSet^{*}, \transfers^*)$, where $\transfers^*$ is given by $\transfers^*_\agent = \prices^{*}_\agent - \uUCB_\agent(\MMap(\agent))$ if $\agent$ is matched and $\transfers^*_\agent = 0$ if $\agent$ is unmatched.

\subsection{Proof of \Cref{thm:instancespecific}}

We first verify (as claimed above) that $\MSetopt$ is a maximum weight matching with respect to $\utility'$. 
\begin{lemma}
\label{lemma:maximummatch}
Matching $\MSetopt$ is a maximum weight matching with respect to $\utility'$.
\end{lemma}
\begin{proof}
Consider any matching $\MSet\neq\MSetopt$. Since \[ \sum_{\agent\in\Agents}\uUCB_\agent(\MMap(\agent))\le -\Delta^{\UCB} + \sum_{\agent\in\Agents} \uUCB_\agent(\MMapopt(\agent)) \]
by the definition of $\Delta^\UCB$, we have
\[ \sum_{\agent\in\Agents} \utility'_\agent( \MMap(\agent))\le\sum_{\agent\in\Agents} \uUCB_\agent(\MMap(\agent))\le \sum_{\agent\in\Agents} \biggl(\uUCB_\agent(\MMapopt(\agent)) - \frac{\Delta^\UCB}{|\Agents|}\biggr)\le\sum_{\agent\in\Agents} \utility'_\agent(\MMapopt(\agent)). \qedhere \]
\end{proof}

We now prove the main lemma for this analysis, restated below. Lemma \ref{lemma:main} shows that if the confidence sets are small enough, then the selected matching will be stable with respect to the true utilities. 
\main*
\begin{proof}[Proof of \Cref{lemma:main}]
The proof proceeds in five steps, which we now outline. We first show the matching returned by \textsc{ComputeMatch$'$} is the maximum weight matching $\MSet^{\text{opt}}$ with respect to $\utility$. We next show that $\MSet^*$ as defined in \textsc{ComputeMatch$'$} also equals $\MSet^{\text{opt}}$. These facts let us conclude that \textsc{ComputeMatch$'$} returns $(\MSetopt, \transfers^*)$. We then show $\Delta^\UCB$ is at least $0.1\Delta$. We then show that $(\MSetopt, \transfers^*)$ is stable with respect to $\utility'$. We finish by showing that this implies $(\MSetopt, \transfers^*)$ is a stable with respect to $\utility$. 

Throughout the proof, we will use the following observation about the expanded confidence sets: 
\begin{equation}\label{eq:d2hypo}
\max\bigl(\Confidence'_{\man,\woman}\bigr) - \min\bigl(\Confidence'_{\man,\woman}\bigr) \le \todoconst{} \frac{\Delta}{|\Agents|} \text{\quad and\quad} \max\bigl(\Confidence'_{\woman, \man}\bigr) - \min\bigl(\Confidence'_{\woman, \man}\bigr) \le \todoconst{} \frac{\Delta}{|\Agents|}
\end{equation}
for all $(\man, \woman)$ in the matching returned by \textsc{ComputeMatching$'$}. This follows from the assumptions in the lemma statement.

\item \paragraph{Proving {\normalfont\textsc{ComputeMatch$'$}} returns $\MSet^{\mathrm{opt}}$ as the matching.}
\normalfont\textsc{ComputeMatch$'$} by definition returns $\MSet^{*,2}$ always, so it suffices to show that $\MSet^{*,2} = \MSet^{\text{opt}}$. Note that $\MSet^{*,2}$ is a maximum weight matching with respect to $u^{\text{UCB}, 2}$. This means that
\begin{align*}
    \sum_{\agent \in \Agents} \utility_\agent(\mu_{\MSet^{*, 2}}(\agent)) 
    &\ge -\sum_{\agent \in \Agents} \bigl(\max\bigl(\Confidence'_{\agent, \mu_{\MSet^{*,2}}(\agent) }\bigr) - \min\bigl(\Confidence'_{\agent, \mu_{\MSet^{*,2}}(\agent) }\bigr)\bigr) + \sum_{\agent \in \Agents} u^{\text{UCB}, 2}_\agent(\mu_{\MSet^{*, 2}}(\agent)) \\
    &\ge - \todoconst{}\Delta + \sum_{\agent \in \Agents} u^{\text{UCB}, 2}_\agent(\mu_{\MSet^{*, 2}}(\agent)) \\
    &\ge -\todoconst{}\Delta + \sum_{\agent \in \Agents} u^{\text{UCB}, 2}_\agent(\mu_{\MSet^{\text{opt}}}(\agent)) \\
    &\ge -\todoconst{}\Delta + \sum_{\agent \in \Agents} \utility_\agent(\mu_{\MSet^{\text{opt}}}(\agent)).
\end{align*}
By the definition of the gap $\Delta$, we conclude that $\MSet^{*,2}= \MSet^{\text{opt}}$. 

\item \paragraph{Proving $\MSetopt = \MSet^{\mathrm{opt}}$.} Suppose for sake of contradiction that $\MSet^* \neq \MSet^{\text{opt}}$. Then
\[ \sum_{\agent \in \Agents} \uUCB_\agent(\mu_{\MSet^{*}}(\agent)) \ge  \sum_{\agent \in \Agents} \uUCB_\agent(\mu_{\MSet^{\text{opt}}}(\agent)) \ge \sum_{\agent \in \Agents} \utility_\agent(\mu_{\MSet^{\text{opt}}}(\agent)), \]
since $\MSet^*$ is a maximum weight matching with respect to $\uUCB$. Moreover, by the definition of the gap, we know that $\sum_{\agent \in \Agents} \utility_\agent(\mu_{\MSet^{*}}(\agent)) \le \sum_{\agent \in \Agents} \utility_\agent(\mu_{\MSet^{\text{opt}}}(\agent)) - \Delta$. Putting this all together, we see that 
\begin{align*}
  \sum_{\agent \in \Agents} \left(\max\bigl(\Confidence_{\agent, \mu_{\MSet^{*}}(\agent) }\bigr) - \min\bigl(\Confidence_{\agent, \mu_{\MSet^{*}}(\agent) }\bigr)\right) &\ge \sum_{\agent \in \Agents} \uUCB_\agent(\mu_{\MSet^{*}}(\agent)) - \sum_{\agent \in \Agents} \utility_\agent(\mu_{\MSet^{*}}(\agent))\\
  &\ge \Delta. 
\end{align*}
We now use this to lower bound the utility of $\MSet^*$ on $u^{\text{UCB}, 2}$. By the definition of the confidence sets, we see that
\begin{align*}
  \sum_{\agent \in \Agents} \utility^{\text{UCB}, 2}_\agent(\mu_{\MSet^{*}}(\agent)) &\ge  \sum_{\agent \in \Agents} \utility^{\text{UCB}}_\agent(\mu_{\MSet^{*}}(\agent)) + \frac 12 \sum_{\agent \in \Agents} \left(\max\bigl(\Confidence_{\agent, \mu_{\MSet^{*}}(\agent) }\bigr) - \min\bigl(\Confidence_{\agent, \mu_{\MSet^{*}}(\agent) }\bigr)\right) \\
  &\ge \sum_{\agent \in \Agents} \utility^{\text{UCB}}_\agent(\mu_{\MSet^{*}}(\agent))  + 0.5 \Delta.
\end{align*}
However, $\MSet^{\text{opt}}$ only achieves a utility of
\begin{align*}
  \sum_{\agent \in \Agents} \utility^{\text{UCB}, 2}_\agent(\mu_{\MSet^{\text{opt}}}(\agent))
    &\le \sum_{\agent \in \Agents} \utility_\agent(\mu_{\MSet^{\text{opt}}}(\agent)) +  \sum_{\agent \in \Agents} \left(\max\bigl(\Confidence'_{\agent, \mu_{\MSet^{\text{opt}}}(\agent) }\bigr) - \min\bigl(\Confidence'_{\agent, \mu_{\MSet^{\text{opt}}}(\agent) }\bigr)\right)  \\
    &\le  \sum_{\agent \in \Agents} \utility_\agent(\mu_{\MSet^{\text{opt}}}(\agent)) + \todoconst{}\Delta.
\end{align*}
But this contradicts the fact (from above) that $\MSet^{\text{opt}} = \MSet^{*,2}$ is a maximum weight matching with respect to $u^{\UCB, 2}$. Therefore, it must be that $\MSetopt = \MSet^{\text{opt}}$.

Putting the above two arguments together, we conclude \textsc{ComputeMatch$'$} returns $(\MSet^*, \transfers^*)$ in this case. 

\item \paragraph{Bounding the gap $\Gap^{\mathrm{UCB}}$.} 
We next show that $\Gap^{\text{UCB}} \ge \todoconst{} \Gap$.  We proceed by assuming
\begin{equation}\label{eq:contradiction}
\sum_{\agent\in\Agents} \uUCB_\agent(\MMap(\agent))\ge -\todoconst{}\Delta + \sum_{\agent\in\Agents} \uUCB_\agent(\MMapopt(\agent))
\end{equation}
for some $\MSet\neq\MSetopt$ and deriving a contradiction.

We first show that \eqref{eq:contradiction} implies a lower bound on
\[ S = \sum_{\agent\in\Agents} \Bigl(\max\bigl(\Confidence_{\agent, \MMap(\agent)}\bigr) - \min\bigl(\Confidence_{\agent, \MMap(\agent)}\bigr)\Bigr) \]
in terms of $\Delta$.
Because the confidence sets contain the true utilities and $\uUCB_\agent$ upper bounds $\utility_\agent$ pointwise, \eqref{eq:contradiction} implies
\[ S + \sum_{\agent\in\Agents}\utility_\agent(\MMap(\agent)) \ge \sum_{\agent\in\Agents} \uUCB_\agent(\MMap(\agent))\ge -\todoconst{}\Delta + \sum_{\agent\in\Agents} \utility_\agent(\MMapopt(\agent)). \]
Applying the definition of $\Delta$, we obtain the lower bound
\[ S\ge - \todoconst{}\Delta + \sum_{\agent\in\Agents} \utility_\agent(\MMapopt(\agent))  - \sum_{\agent\in\Agents} \utility_\agent(\MMap(\agent))\ge (1 - \todoconst{})\Delta. \]

Now, we apply the fact that $\MSet^{*} = \MSet^{*,2} = \MSet^{\text{opt}}$. We establish the following contradiction:
\begingroup
\allowdisplaybreaks
\begin{align*}
\todoconst{}\Delta + \sum_{\agent\in\Agents} \uUCB_\agent(\MMapopt(\agent))
&\ge\todoconst{}\Delta + \sum_{\agent\in\Agents} \utility_\agent(\MMapopt(\agent)) \\
&=\sum_{\agent\in\Agents} (\utility_\agent(\MMapopt(\agent)) + \todoconst{}\Delta/|\Agents|)  \\
&\stackrel{\text{(i)}}{\ge} \sum_{\agent\in\Agents} \utility^{\UCB,2}_\agent(\MMapopt(\agent)) \\
&\stackrel{\text{(ii)}}{\ge} \sum_{\agent\in\Agents} \utility^{\UCB,2}_\agent(\MMap(\agent)) \\
&\stackrel{\text{(iii)}}{\ge} \frac S2 + \sum_{\agent\in\Agents} \utility^{\UCB}_\agent(\MMap(\agent)) \\
&\stackrel{\text{(iv)}}{\ge} \left(\frac 12(1 - \todoconst{})\right)\Delta + \sum_{\agent\in\Agents} \uUCB_\agent(\MMap(\agent)) \\
&\stackrel{\text{(v)}}{\ge} \left(\frac 12(1 - \todoconst{}) - \todoconst{}\right)\Delta + \sum_{\agent\in\Agents} \uUCB_\agent(\MMapopt(\agent)).
\end{align*}
\endgroup
Here, (i) comes from \eqref{eq:d2hypo} in the lemma statement; (ii) holds because $\MSetopt = \MSet^{*,2}$ is a maximum weight matching with respect to $\utility^{\UCB,2}$; (iii) is by the definition of $\utility^{\UCB,2}$; (iv) follows from our lower bound on $S$; and (v) follows from \eqref{eq:contradiction}.

\item \paragraph{Proving that $(\MSet^*, \transfers^*)$ is stable with respect to $\utility'$.}
By \Cref{lemma:maximummatch}, $(\MSet^*, \prices')$ is an optimal primal-dual pair with respect to $\utility'$. Now, it suffices to show that the primal-dual solution corresponds to the market outcome $(\MSet^*, \transfers^*)$ for $\utility'$.  To see this, notice that $\prices'_\agent = 0$ for unmatched agents and
\[ \prices'_\agent = \prices^*_\agent - \frac{\Delta^{\text{UCB}}}{2 |\Agents|} = \transfers^*_\agent + \utility'_\agent(\mu_{\MSet^*}(\agent)) \]
for matched agents. 

\item \paragraph{Proving that $(\MSet^*, \transfers^*)$ is stable with respect to $\utility$.} We show the stability $(\MSetopt, \transfers^*)$ with respect to $\utility$ by checking that individual rationality holds and that there are no blocking pairs. 

The main fact that we will use is that
\[\utility_\agent(\MMapopt(\agent)) \ge \utility'_\agent(\MMapopt(\agent)).\]
To prove this, we split into two cases: (i) agent $\agent$ is matched in $\MSetopt$ (i.e., $\MMapopt(\agent)\neq\agent$), and (ii) agent $\agent$ is not matched by $\MSetopt$. For (i), if $\agent$ is matched by $\MSetopt$, then
\[ \utility_\agent(\MMapopt(\agent)) \ge  \uUCB_\agent(\MMapopt(\agent)) - \todoconst{} \frac{\Delta}{|\Agents|} \ge \uUCB_\agent(\MMapopt(\agent)) - \frac{\Delta^{\UCB}}{|\Agents|} =  \utility'_\agent(\MMapopt(\agent)).\]
For (ii), if $\agent$ is not matched by $\MSetopt$, then $\utility_\agent(\MMapopt(\agent))\ge\utility'_\agent(\MMapopt(\agent))$ because both sides are $0$. 

For individual rationality, we thus have
\begin{align*}
 \utility_\agent(\mu_{\MSet^*}(\agent)) + \transfers^*_\agent &\ge \utility'_\agent(\mu_{\MSet^*}(\agent)) + \transfers^*_\agent \ge 0,
\end{align*}
where the second inequality comes from the individual rationality of $(\MSetopt, \transfers^*)$ with respect to $\utility'$. 

Let's next show that there are no blocking pairs. If $(\man, \woman)\in\MSetopt$, then we see that:
\[\utility_\man(\mu_{\MSet^*}(\man)) +  \transfers^*_\man +  \utility_{\woman}(\mu_{\MSet^*}(\woman)) +  \transfers^*_{\woman} = \utility_\man(\mu_{\MSet^*}(\man)) +\utility_{\woman}(\mu_{\MSet^*}(\woman)), \] as desired.  
Next, consider any pair $(\man, \woman)\not\in\MSetopt$. Then,
\[ \utility_\man(\woman) + \utility_\woman(\man) \le \uUCB_\man(\woman) + \uUCB_\woman(\man) = \utility'_\man(\woman) + \utility'_\woman(\man). \]
It follows that
\begin{align*}
\utility_\man(\mu_{\MSet^*}(\man)) +  \transfers^*_\man +  \utility_{\woman}(\mu_{\MSet^*}(\woman)) +  \transfers^*_{\woman}
&\ge\utility'_\man(\mu_{\MSet^*}(\man)) +  \transfers^*_\man +  \utility_{\woman}(\mu_{\MSet^*}(\woman)) +  \transfers^*_{\woman} \\
&\ge\utility'_\man(\woman) + \utility'_\woman(\man) \\
&\ge \utility_\man(\woman) + \utility_\woman(\man),
\end{align*}
where the second inequality comes from the fact that $(\MSetopt, \transfers^*)$ has no blocking pairs with respect to $\utility'$.

This completes our proof that $(\MSet^*, \transfers^*)$ is stable with respect to $\utility$. 
\end{proof}

Now, we are ready to prove Theorem \ref{thm:instancespecific}. 
\begin{proof}[Proof of Theorem \ref{thm:instancespecific}]
As in the proof of Theorem \ref{thm:instanceind}, the starting point for our proof is the typical approach in multi-armed bandits and combinatorial bandits \cite{GKJ12,DBLP-conf/icml/ChenWY13, LS2020} of bounding regret in terms of the sizes of the confidence interval of the chosen arms. Our approach does not quite compose cleanly with these proofs, since we need to handle the transfers in addition to the matching. 

We divide in two cases, based on the event $\Event$ that all of the confidence sets contain their respective true utilities at every time step $\round\le\Rounds$. That is, $\utility_{\man}(\woman)\in\Confidence_{\man,\woman}$ and $\utility_\woman(\man)\in\Confidence_{\woman,\man}$ for all $(\man,\woman)\in\Men\times\Women$ at all $\round$. 

\paragraph{Case 1: $\Event$ holds.} Let $\pulls_{\man\woman}^{\round}$ be the number of times that the pair $(\man,\woman)$ has been matched by round $\round$. For each pair $(\man, \woman)$, we maintain a ``blame'' counter $\blame^\round_{\man\woman}$. We will ultimately bound the total number of time steps where the algorithm chooses a matching that is not stable by \smash{$\sum_{(i,j)} b_{i,j}^T$}.

We increment the blame counters as follows. First, suppose that  \[\max\bigl(\Confidence_{\agent, \MMapt(\agent)}\bigr) - \min\bigl(\Confidence_{\agent, \MMapt(\agent)}\bigr) \le \todoconst{}\frac{\Delta}{|\Agents|} \]  for every matched agent $\agent \in \Agents$. By \Cref{lemma:main} and since the event $E$ holds, we know the chosen matching is stable and thus incurs $0$ regret. We do not increment any of the blame counters in this case. Now, suppose that \[\max\bigl(\Confidence_{\agent, \MMapt(\agent)}\bigr) - \min\bigl(\Confidence_{\agent, \MMapt(\agent)}\bigr) > \todoconst{}\frac{\Delta}{ |\Agents|} \] for some matched agent $\agent$. We increment the counter of the least-blamed pair $(\man, \woman) \in \MSet^t$. 

We now bound the blame counter $\blame^\Rounds_{\man\woman}$. We use the fact that the blame counter is only incremented when the corresponding confidence set  is sufficiently large, and that a new sample of the utilities is received whenever the blame counter is incremented. This means that:  
\[\blame^\Rounds_{\man\woman} = O\left(\frac{|\Agents|^2 \log (|\Agents| \Rounds))}{\Delta^2} \right).\]
The maximum regret incurred by any matching is at most $12 |\Agents|$ which means that the regret incurred by this case is at most:
\[ 12 |\Agents| \sum_{(\man, \woman)} \blame^\Rounds_{\man\woman} \le 12 |\Agents| \sum_{(\man, \woman)} O\left(\frac{|\Agents|^2 \log (|\Agents| \Rounds))}{\Delta^2}\right) = O \left(\frac{|\Agents|^5 \log (|\Agents| \Rounds))}{\Delta^2} \right).  \]

\paragraph{Case 2: $\Event$ does not hold.}  Since each $\pulls_{\man\woman}(\hat\utility_\man(\woman) -\utility_\man(\woman))$ is mean-zero and $1$-subgaussian and we have $O(|\Men||\Women|\Rounds)$ such random variables over $\Rounds$ rounds, the probability that any of them exceeds 
\[ 2\sqrt{\log(|\Men||\Women|\Rounds/\delta)}\le 2\sqrt{\log(|\Agents|^2\Rounds/\delta)} \]
is at most $\delta$ by a standard tail bound for the maximum of subgaussian random variables. It follows that $\Event$ fails to hold with probability at most $|\Agents|^{-2}\Rounds^{-2}$. In the case that $\Event$ fails to hold, our regret in any given round would be at most $12|\Agents|$ by the Lipschitz property in \Cref{prop:desiderata}. (Recall that our upper confidence bound is off by at most $6$ due to clipping the confidence interval to lie in $[-1, 1]$, so that the expanded confidence sets also necessarily lie in $[-3, 3]$.) Thus, the expected regret from this scenario is at most
\[ |\Agents|^{-2}\Rounds^{-2}\cdot 12|\Agents|\Rounds\le 12|\Agents|^{-1}\Rounds^{-1}, \]
which is negligible compared to the regret bound from when $\Event$ does occur.
\end{proof}

\subsection{Instance-independent regret bounds for \textsc{MatchUCB$'$}}

To establish \textit{instance-independent} regret bounds for \textsc{MatchUCB$'$}, we show that $(\MSet^{*}, \prices^{*})$ is indeed optimal with respect to $\uUCB$; the remainder then follows the same argument as \Cref{thm:instanceind}.
\begin{lemma}
\label{lemma:optimality}
The pair $(\MSet^{*}, \prices^{*})$ is an optimal primal-dual pair with respect to $\uUCB$.
\end{lemma}
\begin{proof}
It suffices to verify feasibility and, by weak duality, check that $\MSetopt$ and $\pricesopt$ achieve the same objective value. It is clear that $\MSetopt$ is primal feasible. For dual feasibility, if $(\man, \woman)\not\in\MSetopt$, then
\[ \pricesopt_\man + \pricesopt_\woman\ge\prices'_\man + \prices'_\woman\ge\utility'_\man(\woman) + \utility'_\woman(\man) = \uUCB_\man(\woman) + \uUCB_\woman(\man); \]
and if $(\man, \woman)\in\MSetopt$, then
\[ \pricesopt_\man + \pricesopt_\woman = \prices'_\man + \prices'_\woman + 2\frac{\Delta^{\text{UCB}}}{|\Agents|}\ge\utility'_\man(\woman) + \utility'_\woman(\man) + 2\frac{\Delta^{\text{UCB}}}{|\Agents|} = \uUCB_\man(\woman) + \uUCB_\woman(\man). \]
Finally, we check that they achieve the same objective value with respect to $\uUCB$. By \Cref{lemma:maximummatch} and strong duality, $\MSetopt$ achieves the same objective value as $\prices'$ with respect to $\utility'$. Hence
\[ \sum_{\agent\in\Agents} \uUCB_\agent(\MMapopt(\agent)) = 2|\MSetopt|\frac{\Delta^\UCB}{|\Agents|} + \sum_{{\agent\in\Agents}} \utility'_\agent(\MMapopt(\agent))  = 2|\MSetopt|\frac{\Delta^\UCB}{|\Agents|} +\sum_{\agent\in\Agents} \prices'_\agent = \sum_{\agent\in\Agents} \pricesopt_\agent. \qedhere \]
\end{proof}

\section{Proofs for Section \ref{appendix:alternatedef}}\label{appendix:proofrevenue}

\searchfricinf*

\begin{proof}[Proof of \Cref{thm:searchfricinf}]
The algorithm is defined as follows. We set confidence sets according to \textsc{MatchUCB} and run essentially that algorithm, but with a modified \textsc{ComputeMatch}. Instead of \textsc{ComputeMatch}, we use the following algorithm. The platform first computes a matching with transfers $(\MSet^*, \transfers^*)$ according to the UCB estimates $\uUCB$, like before. Then, the platform chooses $\MSet^*$ to be the selected matching, and sets the transfers according to:
\[ \transfers_\agent = \transfers^*_\agent - \epsilon +  \max\bigl(\Confidence_{\agent, \MMap(\agent))}\bigr) - \min\bigl(\Confidence_{\agent, \MMap(\agent)}\bigr).\]
This choice of transfers has a clean economic intuition: agents should be compensated based on the platform's uncertainty about their utilities with $\epsilon$ of their transfer shaved off as revenue for the platform.

First, we show that if the confidence sets contain the true utilities, then $(\MSet^*, \transfers)$ is $\epsilon$-stable. It suffices to show that $(\MSet^*, \transfers')$ where:
\[\transfers'_\agent = \transfers^*_\agent +  \max\bigl(\Confidence_{\agent, \MMap(\agent))}\bigr) - \min\bigl(\Confidence_{\agent, \MMap(\agent)}\bigr)  \]
is stable. First, we see that 
\[\utility_\agent(\MMapUCB(\agent)) + \transfers'_\agent = \uUCB_\agent(\MMapUCB(\agent)) + \transfers^*_\agent \ge 0,\] since $(\MSet, \transfers^*)$ is stable with respect to $\uUCB$. Furthermore, we see that:
\begin{align*}
 \bigl(\utility_\man(\MMap(\man)) + \transfers'_\man\bigr) + \bigl(\utility_\woman(\MMap(\woman)) + \transfers'_\woman\bigr) &\ge   \bigl(\uUCB_\man(\MMap(\man)) + \transfers^*_\man\bigr) + \bigl(\uUCB_\woman(\MMap(\woman)) + \transfers^*_\woman\bigr) \\
 &\ge \uUCB_\man(\woman) + \uUCB_\woman(\man) \\
 &\ge \utility_\man(\woman) + \utility_\woman(\man),
\end{align*}
where the second to last line follows from the fact that $(\MSet, \transfers^*)$ is stable with respect to $\uUCB$. 

We first show that $\subsidies$ is a feasible solution to \eqref{eq:instabilityNTU}:
\begin{align*}
  \lefteqn{\min \bigl(\utility_\man(\woman)  - \utility_\man(\MMapUCB(\man))- \subsidies_\man, \utilityii_\woman(\man) - \utility_\woman(\MMapUCB(\woman)) - \subsidies_\woman\bigr)} \\
  &=  \min \bigl(\utility_\man(\woman)  - \uUCB_\man(\MMapUCB(\man)), \utilityii_\woman(\man) - \uUCB_\woman(\MMapUCB(\woman))\bigr) \\
  &\le \min \bigl(\uUCB_\man(\woman)  - \uUCB_\man(\MMapUCB(\man)), \uUCB_\woman(\man) - \uUCB_\woman(\MMapUCB(\woman))\bigr) \\
  &\le 0,
\end{align*}
where the last step uses the fact that $\MMapUCB$ is stable with respect to $\uUCB$ by definition. Moreover, we see that 
\[\utility_\agent(\MMapUCB(\agent)) + \subsidies_\agent = \uUCB_\agent(\MMapUCB(\agent)) \ge 0,\] where the last inequality uses that $\MMapUCB$ is stable with respect to $\uUCB$ by definition. This implies that $\subsidies$ is feasible.

We see that the platform's revenue is equal to:
\begin{align*}
  -\sum_{t=1}^T \sum_{\agent \in \Agents_t} \transfers_\agent &= -\sum_{t=1}^T \sum_{\agent \in \Agents_t} \transfers^*_\agent + \sum_{t=1}^T \sum_{\agent \in \Agents_t} \epsilon + \sum_{t=1}^T \sum_{\agent \in \Agents_t} \left(\max\bigl(\Confidence_{\agent, \MMap(\agent))}\bigr) - \min\bigl(\Confidence_{\agent, \MMap(\agent)}\bigr)\right) \\
  &= \epsilon \sum_{t=1}^T |\Agents_t| - \sum_{t=1}^T \sum_{\agent \in \Agents_t} \left(\max\bigl(\Confidence_{\agent, \MMap(\agent))}\bigr) - \min\bigl(\Confidence_{\agent, \MMap(\agent)}\bigr)\right).
\end{align*}
Using the proof of \Cref{thm:instanceind}, we see that  
\[\sum_{t=1}^T \sum_{\agent \in \Agents_t} \left(\max\bigl(\Confidence_{\agent, \MMap(\agent))}\bigr) - \min\bigl(\Confidence_{\agent, \MMap(\agent)}\bigr)\right) \le O(|\Agents| \sqrt{n \Rounds} \log(|\Agents| T)),\] as desired.

\end{proof}

\section{Proofs for Section \ref{appendix:matchntu}}\label{appendix:proofmatchNTU}

\subsection{Proof of Proposition \ref{prop:desiderataNTU}}

\begin{proof}[Proof of Proposition \ref{prop:desiderataNTU}]

We first prove the first part of the statement, and then the second part of the statement.

\item \paragraph{Proof of part (a).}  We note that it follows immediately from Definition \ref{definition:instabilityNTU} that \NTUinstabmeasure{} is nonnegative. Let's now show that $\InstabilityNTU{\utility}{\MSet}$ is zero if and only if $(\MSet, \transfers)$ is stable. It is not difficult to see that the infimum of \eqref{eq:instabilityNTU} is attained at some $\subsidies^*$. 

If $\InstabilityNTU{\utility}{\MSet} = 0$, then we know that $\subsidies^*_\agent = 0$ for all $\agent \in \Agents$. The constraints in the optimization problem imply that $\MSet$ has no blocking pairs and individually rationality is satisfied, as desired.

If $\MSet$ is stable, then we see that $\subsidies = \vec{0}$ is a feasible solution to \eqref{eq:instabilityNTU}, which means that the optimum of \eqref{eq:instabilityNTU}  is at most zero. This coupled with the fact that $\InstabilityNTU{\utility}{\MSet}$ is always nonnegative means that $\InstabilityNTU{\utility}{\MSet} = 0$ as desired. 

\item \paragraph{Proof of part (b).} Consider two utility functions $\utility$ and $\utilityii$. To show Lipchitz continuity, it suffices to show that for any matching $\MSet$:
\[|\InstabilityNTU{\utility}{\MSet} - \InstabilityNTU{\utilityii}{\MSet}| \le 2\sum_{\agent\in\Agents} \norm{\utility_\agent - \utilityii_\agent}_\infty. \] We show that:
\[\InstabilityNTU{\utilityii}{\MSet} \le  \InstabilityNTU{\utility}{\MSet} + 2\sum_{\agent\in\Agents} \norm{\utility_\agent - \utilityii_\agent}_\infty, \] noting that the other direction follows from an analogous argument. Let $\subsidies^*$ be an optimal solution to \eqref{eq:instabilityNTU} for the utilities $\utility$. Consider the solution $\subsidies_\agent = \subsidies^*_\agent + 2 \norm{\utility_\agent - \utility_\agent}_{\infty}$. We first verify that $\subsidies$ is a feasible solution to \eqref{eq:instabilityNTU} for $\utilityii$. We see that:
\begin{align*}
  & \min \bigl(\utilityii_\man(\woman) - \utilityii_\man(\MMap(\man)) - \subsidies_\man, \utilityii_\woman(\man) - \utilityii_\woman(\MMap(\woman)) - \subsidies_\woman\bigr) \\
  &=    \min \bigl(\utilityii_\man(\woman) - \utilityii_\man(\MMap(\man)) - \subsidies^*_\man - 2 \norm{\utility_\man - \utilityii_\man}_{\infty}, \utilityii_\woman(\man) - \utilityii_\woman(\MMap(\woman)) - \subsidies^*_\woman - 2 \norm{\utility_\woman - \utilityii_\woman}_{\infty}\bigr) \\
  &\le  \min \bigl(\utility_\man(\woman) - \utility_\man(\MMap(\man)) - \subsidies^*_\man, \utility_\woman(\man) - \utility_\woman(\MMap(\woman)) - \subsidies^*_\woman\bigr) \\
  &\le 0,
\end{align*}
as desired. Moreover, we see that 
\[\utilityii_\agent(\MMap(\agent)) + \subsidies_\agent = \utilityii_\agent(\MMap(\agent)) + \subsidies^*_\agent + 2 \norm{\utility_\agent - \utility_\agent}_{\infty} \le \utility_\agent(\MMap(\agent)) + \subsidies^*_\agent \ge 0.\]
Thus we have demonstrated that $\subsidies$ is feasible. This means that:
\[\InstabilityNTU{\utilityii}{\MSet} \le \sum_{\agent \in \Agents} \subsidies_\agent = \sum_{\agent \in \Agents} \subsidies^*_\agent + 2\sum_{\agent\in\Agents} \norm{\utility_\agent - \utilityii_\agent}_\infty = [\InstabilityNTU{\utility}{\MSet} + 2\sum_{\agent\in\Agents} \norm{\utility_\agent - \utilityii_\agent}_\infty, \]
as desired. 

\end{proof}

\subsection{Proof of Theorem \ref{thm:matchingNTU}}

We show that the algorithmic approach from \Cref{sec:regret} can be adapted to the setting of matching with non-transferable utilities. 

Drawing intuition from \Cref{sec:regret}, at each round, we compute a stable matching for utilities given by the upper confidence bounds. More precisely, suppose we have a collection $\mathscr{C}$ of confidence sets \smash{$\Confidence_{\man, \woman}, \Confidence_{\woman, \man}\subseteq\R$} such that $\utility_\man(\woman)\in \Confidence_{\man, \woman}$ and $\utility_\woman(\man)\in\Confidence_{\woman, \man}$ for all $(\man, \woman) \in \Men \times \Women$. Our algorithm uses $\mathscr{C}$ to get an upper confidence bound for each agent's utility function and then computes a stable matching with transfers as if these upper confidence bounds were the true utilities (see \textsc{ComputeMatchNTU}). This can be implemented efficiently if we use, e.g., the Gale-Shapley algorithm (either the customer-proposing algorithm or the provider-proposing algorithm will work). 

\begin{algorithm}[tb] \caption{$\textsc{ComputeMatchNTU}$: Compute matching with transfers from confidence sets}\label{alg:fromconfsetsNTU}
\begin{algorithmic}[1]
\Procedure{ComputeMatchNTU}{$\mathscr{C}$}
  \For{$(\man, \woman) \in \Men \times \Women$}
      \State $\uUCB_{\man}(\woman)\gets\max\bigl( \Confidence_{\man, \woman }\bigr)$;\quad
      $\uUCB_{\woman}(\man)\gets \max \bigl(\Confidence_{\woman, \man}\bigr)$ \Comment{UCB estimates of utilities.}
    \EndFor
      \State Run any version of the Gale-Shapley algorithm \cite{GS62} on $\uUCB$ to obtain a matching $\MSetopt$. 
  \State \textbf{return} $\MSetopt$
\EndProcedure
\end{algorithmic}
\end{algorithm}

The core property of \textsc{ComputeMatchNTU} is that we can upper bound \NTUinstabmeasure{} by the sum of the sizes of the relevant confidence sets, assuming that the confidence sets contain the true utilities. 
\begin{proposition}
\label{lemma:confsetNTU}
Consider a collection confidence sets $\mathscr{C}$ such that $\utility_{\man}(\woman) \in \Confidence_{\man, \woman}$ and $\utility_{\woman}(\man) \in \Confidence_{\woman,\man}$ for all $(\man, \woman) \in \Men \times \Women$. The instability of the output \smash{$\MSetUCB$} of \textsc{ComputeMatch} satisfies
\begin{equation}\label{eq:instaboundNTU}
\InstabilityNTU{\utility}{\MSetUCB} \le \sum_{\agent \in \Agents^t} \Bigl(\max\bigl(\Confidence_{\agent, \MMapUCB(\agent)}\bigr) - \min\bigl(\Confidence_{\agent, \MMapUCB(\agent)}\bigr)\Bigr).\end{equation}
\end{proposition}
\begin{proof}
We construct subsidies for this setting to be:
\[ \subsidies_\agent = \max\bigl(\Confidence_{\agent, \MMap(\agent)}\bigr) - \utility_\agent(\MMap(\agent))\le \max\bigl(\Confidence_{\agent, \MMap(\agent)}\bigr) - \min\bigl(\Confidence_{\agent, \MMap(\agent)}\bigr).\]

\paragraph{Step 1: Verifying feasibility.} We first show that $\subsidies$ is a feasible solution to \eqref{eq:instabilityNTU}. 
\begin{align*}
  & \min \bigl(\utility_\man(\woman)  - \utility_\man(\MMapUCB(\man))- \subsidies_\man, \utilityii_\woman(\man) - \utility_\woman(\MMapUCB(\woman)) - \subsidies_\woman\bigr) \\
  &=  \min \bigl(\utility_\man(\woman)  - \uUCB_\man(\MMapUCB(\man)), \utilityii_\woman(\man) - \uUCB_\woman(\MMapUCB(\woman))\bigr) \\
  &\le \min \bigl(\uUCB_\man(\woman)  - \uUCB_\man(\MMapUCB(\man)), \uUCB_\woman(\man) - \uUCB_\woman(\MMapUCB(\woman))\bigr) \\
  &\le 0,
\end{align*}
where the last step uses the fact that $\MMapUCB$ is stable with respect to $\uUCB$ by definition. Moreover, we see that 
\[\utility_\agent(\MMapUCB(\agent)) + \subsidies_\agent = \uUCB_\agent(\MMapUCB(\agent)) \ge 0,\] where the last inequality uses that $\MMapUCB$ is stable with respect to $\uUCB$ by definition. This implies that $\subsidies$ is feasible.

\paragraph{Step 2: Computing the objective.} We next compute the objective of \eqref{eq:instabilityNTU} at $\subsidies$ and use this to bound $\InstabilityNTU{\utility}{\MSet^*}$. A simple calculation shows that:
\[\InstabilityNTU{\utility}{\MSet^*} \le \sum_\agent \subsidies_\agent =  \sum_{\agent \in \Agents} \Bigl(\max\bigl(\Confidence_{\agent, \MMapUCB(\agent)}\bigr) - \min\bigl(\Confidence_{\agent, \MMapUCB(\agent)}\bigr)\Bigr),\]
as desired.

\end{proof}

\subsubsection{Explicit algorithm and regret bounds}\label{subsec:basicNTU}

Using the same intuition as Section \ref{sec:regret}, the regret bound of \Cref{lemma:confsetNTU} hints at an algorithm: {each round, select the matching with transfers returned by \textsc{ComputeMatchNTU} and update confidence sets accordingly}. To instantiate this approach, it remains to construct confidence intervals that contain the true utilities with high probability. 

We showcase this algorithm in the simple setting of unstructured preferences. For this setting, we can construct our confidence intervals following the classical UCB approach. That is, for each utility value involving the pair $(\man, \woman)$, we take a length \smash{$O(\sqrt{\log(|\Agents|\Rounds)} / \pulls_{\man\woman})$} confidence interval centered around the empirical mean, where $\pulls_{\man\woman}$ is the number of times the pair has been matched before. We describe this construction precisely in \Cref{alg:bandits} (\textsc{MatchNTUUCB}).

\begin{algorithm}
\caption{\textsc{MatchNTUUCB}: A bandit algorithm for matching with non-transferable utilities.}\label{alg:banditsNTU}
\begin{algorithmic}[1]
\Procedure{MatchNTUUCB}{$T$}
  \For{$(\man, \woman) \in \Men \times \Women$} \Comment{Initialize confidence intervals and empirical mean.}
  \State $\Confidence_{\man, \woman} \gets [-1, 1]$;\quad $\Confidence_{\woman, \man} \gets [-1, 1]$; \quad $\hat{\utility}_\man(\woman)\gets 0$; \quad $\hat{\utility}_\woman(\man)\gets 0$
  \EndFor
  \For{$1 \le \round \le\Rounds$}
  \State $\MSet^\round \gets \Call{ComputeMatchNTU}{\Confidence}$
  \For{$(\man, \woman) \in \MSet^\round$} \Comment{Set confidence intervals and update means.}
  \State Update $\hat\utility_\man(\woman)$ and $\hat\utility_\woman(\man)$ from feedback; increment counter $\pulls_{\man\woman}$
  \State $\Confidence_{\man, \woman} \gets \bigl[\mean_{\man}(\woman)-8\sqrt{{\log(|\Agents|\Rounds)} / {\pulls_{\man\woman}}},  \mean_{\man}(\woman) + 8\sqrt{{\log(|\Agents|\Rounds)}/{\pulls_{\man, \woman}}}\,\bigr]\cap [-1, 1]$
  \State $\Confidence_{\woman, \man} \gets \bigl[\mean_{\woman}(\man)-8\sqrt{{\log(|\Agents|\Rounds)}/{\pulls_{\man\woman}}},  \mean_{\woman}(\man) +8\sqrt{{\log(|\Agents|\Rounds)}/{\pulls_{\man, \woman}}} \,\bigr]\cap [-1, 1]$ 
  \EndFor
  \EndFor
\EndProcedure
\end{algorithmic}
\end{algorithm}

To analyze \textsc{MatchNTUUCB}, recall that \Cref{lemma:confset} bounds the regret at each step by the lengths of the confidence intervals of each pair in the selected matching. Like in Section \ref{sec:regret}, this yields the following instance-independent regret bound:
\begin{theorem}
\label{thm:instanceindNTU}
\textsc{MatchNTUUCB} incurs expected  regret \smash{$\mathbb{E}(\Regret_\Rounds) \le O\bigl({|\Agents|}^{3/2} \sqrt{\Rounds} \sqrt{\log(|\Agents| \Rounds)}\bigr)$}.
\end{theorem}
\begin{proof}
This proof proceeds very similarly to the proof of Theorem \ref{thm:instanceind}. We consider the event $\Event$ that all of the confidence sets contain their respective true utilities at every time step $\round\le\Rounds$. That is, $\utility_{\man}(\woman)\in\Confidence_{\man, \woman}$ and $\utility_\woman(\man)\in\Confidence_{\woman, \man}$ for all $(\man,\woman)\in\Men\times\Women$ at all $\round$.

\item \paragraph{Case 1: $\Event$ holds.} 
By \Cref{lemma:confset}, we may bound
\[ \InstabilityNTU[^\round]{\utility}{\MSet^\round} \le \sum_{\agent \in \Agents^t} \Bigl(\max\bigl(\Confidence_{\agent, \MMapt(\agent)}\bigr) - \min\bigl(\Confidence_{\agent, \MMapt(\agent)}\bigr)\Bigr) = O\paren*{\sum_{(\man,\woman)\in\MSet^\round} \sqrt{\frac{\log(|\Agents|\Rounds)}{\pulls_{\man\woman}^{\round}}}},\]
where $\pulls_{\man\woman}^{\round}$ is the number of times that the pair $(\man,\woman)$ has been matched at the start of round $\round$. Let $w^\round_{\man, \woman} = \frac{1}{\sqrt{\pulls_{\man\woman}^{\round}}}$ be the size of the confidence set (with the log factor scaled out) for $(\man, \woman)$ at the start of round $\round$. 

At each time step $\round$, let's consider the list consisting of $w^\round_{\man_\round, \woman_\round}$ for all $(\man_\round, \woman_\round) \in \MSet^\round$. Let's now consider the overall list consisting of the concatenation of all of these lists over all rounds. Let's order this list in decreasing order to obtain a list $\tilde{w}_1, \ldots, \tilde{w}_L$ where $L = \sum_{\round=1}^{\Rounds} |\MSet^\round| \le n \Rounds$. In this notation, we observe that:
\[\sum_{\round=1}^{\Rounds} \InstabilityNTU[^\round]{\utility}{\MSet^\round} \le \sum_{\round=1}^{\Rounds} \sum_{\agent \in \Agents^t} \Bigl(\max\bigl(\Confidence_{\agent, \MMapt(\agent)}\bigr) - \min\bigl(\Confidence_{\agent, \MMapt(\agent)}\bigr)\Bigr) = \log(|\Agents| \Rounds) \sum_{l=1}^L \tilde{w}_l.   \]
We claim that $\tilde{w}_l \le O\paren*{\min(1, \frac{1}{\sqrt{(l / |\Agents|^2) -1}})}$. The number of rounds that a pair of agents can have their confidence set have size at least $\tilde{w}_l$ is upper bounded by $1 + \frac{1}{\tilde{w}_l^2}$. Thus, the total number of times that any confidence set can have size at least $\tilde{w}_l$ is upper bounded by $(|\Agents|^2)(1 + \frac{1}{\tilde{w}_l^2})$. 

Putting this together, we see that:
\begin{align*}
  \log(|\Agents| \Rounds) \sum_{l=1}^L \tilde{w}_l &\le O\paren*{\sum_{l=1}^{L} \min(1, \frac{1}{\sqrt{(l / |\Agents|^2) -1}})} \\
  &\le O\paren*{\log(|\Agents| \Rounds) \sum_{l=1}^{n T} \min(1, \frac{1}{\sqrt{(l / |\Agents|^2) -1}})} \\
  &\le O\paren*{ |\Agents| \sqrt{nT} \log(|\Agents| \Rounds)}. 
\end{align*}

\item \paragraph{Case 2: $\Event$ does not hold.} Since each $\pulls_{\man\woman}(\hat\utility_\man(\woman) -\utility_\man(\woman))$ is mean-zero and $1$-subgaussian, and we have $O(|\Men||\Women|\Rounds)$ such random variables over $\Rounds$ rounds, the probability that any of them exceeds 
\[ 2\sqrt{\log(|\Men||\Women|\Rounds/\delta)}\le 2\sqrt{\log(|\Agents|^2\Rounds/\delta)} \]
is at most $\delta$ by a standard tail bound for the maximum of subgaussian random variables. It follows that $\Event$ fails to hold with probability at most $|\Agents|^{-2}\Rounds^{-2}$. In the case that $\Event$ fails to hold, our regret in any given round would be at most $4|\Agents|$ by the Lipschitz property in \Cref{prop:desiderataNTU}. (Recall that our upper confidence bound for any utility is wrong by at most two due to clipping each confidence interval to lie in $[-1, 1]$.) Thus, the expected regret from this scenario is at most
\[ |\Agents|^{-2}\Rounds^{-2}\cdot 4|\Agents|\Rounds\le 4|\Agents|^{-1}\Rounds^{-1}, \]
which is negligible compared to the regret bound from when $\Event$ does occur.
\end{proof}

\end{document}